\documentclass[10pt,journal,compsoc]{IEEEtran}
\usepackage{amsmath,amsfonts}
\usepackage{algorithmic}
\usepackage{algorithm}
\usepackage{array}
\usepackage{textcomp}
\usepackage{stfloats}
\usepackage{url}
\usepackage{verbatim}
\usepackage{cite}
\usepackage{soul}
\usepackage{balance}
\usepackage{colortbl}
\usepackage{amsthm}
\newtheorem{proposition}{Proposition}
\newtheorem{theorem}{Theorem}

\usepackage{courier}
\usepackage{graphicx, wrapfig, setspace, booktabs}
\usepackage{subcaption}
\usepackage{multirow}
\usepackage{pifont}
\usepackage{comment}
\usepackage{float}
\usepackage[table]{xcolor}
\usepackage{threeparttable}
\usepackage{hyperref}
\usepackage{booktabs}
\usepackage{adjustbox} 
\usepackage{longtable}

\hypersetup{
    colorlinks=true,
    citecolor=black,
    linkcolor=blue,
    filecolor=magenta,      
    urlcolor=blue,
    pdfpagemode=FullScreen,
    }

\begin{document}

\title{A Unified Perspective for Learning Graph Representations Across Multi-Level Abstractions}

\author{Mohamed~Mahmoud~Amar,~Nairouz~Mrabah,~
        Mohamed~Bouguessa,~
        Abdoulaye~Baniré~Diallo
\IEEEcompsocitemizethanks{\IEEEcompsocthanksitem M. M. Amar, N. Mrabah, M. Bouguessa, A. B. Diallo are with the Department
of Computer Science, University of Quebec at Montreal, Montreal, QC, Canada.
E-mails: amar.mohamed\_mahmoud@courrier.uqam.ca,~mrabah.nairouz@gmail.com, ~bouguessa.mohamed@uqam.ca,~diallo.abdoulaye@uqam.ca}
}

\markboth{IEEE Transactions on Knowledge and Data Engineering,~Vol.~x, No.~x, Month~Year}
{Amar \MakeLowercase{\textit{et al.}}: A Unified Perspective for Learning Graph Representations Across Multi-Level Abstractions}

\IEEEtitleabstractindextext{
\begin{abstract}
Graph Self-Supervised Learning (GSSL) has emerged as a powerful paradigm for generating high-quality representations for graph-structured data.
While multi-scale graph contrastive learning has received increasing attention, many existing methods still predominantly focus on a single graph abstraction level. To address this limitation, we propose a unified contrastive framework that can target node-level, proximity-level, cluster-level, and graph-level information and integrate them through a linear combination of similarity scores on positive pairs and dissimilarity scores (i.e., similarity scores on negative pairs). Furthermore, current approaches typically assign uniform penalty strengths to all examples, which reduces optimization flexibility and leads to ambiguous convergence status. To overcome this, we introduce a novel parameter-free fine-grained self-weighting mechanism that adaptively assigns weights to individual similarity and dissimilarity scores. The proposed mechanism emphasizes the scores that deviate significantly from their target values. Our approach not only enhances optimization flexibility but also eliminates the computational overhead of hyperparameter tuning in conventional multi-task GSSL methods. Comprehensive experiments on real-world datasets show that our methods consistently outperform state-of-the-art approaches across downstream tasks, including classification, clustering, and link prediction, in both single-level and multi-level scenarios.
\end{abstract}

}

\maketitle

\IEEEdisplaynontitleabstractindextext

\IEEEpeerreviewmaketitle

\section{Introduction} 

\IEEEPARstart{S}{elf-supervised} learning (SSL) \cite{shaheen2025rethinking} has emerged as a highly effective approach for learning data representations, particularly in scenarios where supervisory signals are unavailable. Unlike supervised learning, which requires large labeled datasets, or reinforcement learning, which depends on repeated trials and feedback, SSL uses inherent signals within the data itself. By designing pretext tasks that challenge models to predict or reconstruct parts of the input data, SSL enables the learning of intrinsic data properties and relationships.
As a result, representations learned through SSL are often more robust and versatile. When finetuned for downstream tasks, they show strong results across various applications \cite{chen2020big, mrabah2023toward, mrabah2023exploring}.

In recent years, the contrastive learning paradigm has been widely adopted in graph self-supervised learning (GSSL) \cite{liu2022graph}. This paradigm aims to bring similar entities closer together and push dissimilar ones farther apart in the representation space. 
Augmented views of a graph are generated through transformations like node feature masking, edge perturbation, and subgraph sampling. These augmentations create diverse yet semantically consistent views. Thus, the model can learn invariant representations by maximizing similarity between positive pairs while distinguishing them from negative pairs. 
By emphasizing intrinsic relationships within the graph, contrastive GSSL methods learn generalizable representations, which can be effectively used in node classification, link prediction, and clustering without requiring extensive retraining.

%Graph Abstraction levels 
GSSL methods operate across multiple abstraction levels to capture different aspects of a graph's structure and semantics. \textit{Node-Level} methods \cite{zhu2020deep, thakoor2021large} focus on local structural and feature information, making them particularly effective for tasks like node classification and link prediction. However, they may struggle to capture global properties, risk overfitting to local structures, and are sensitive to noise. \textit{Proximity-Level} methods \cite{grover2016node2vec, kipf2016variational} emphasize structural relationships within local neighborhoods, excelling in tasks like link prediction and community detection. However, they often fail to capture long-range dependencies and global structures. \textit{Cluster-Level} methods \cite{mrabah2024contrastive, mrabah2023beyond} identify and utilize community structures to capture relationships between node groups. Their performance depends on achieving the right balance in cluster granularity, as overly coarse or fine clusters can lead to overfitting or loss of detail. Moreover, the clustering process is prone to noisy clustering assignments, which can affect the quality of the learned representations. \textit{Graph-Level} methods \cite{velickovic2019deep, hassani2020contrastive} focus on the global structure of the graph, making them ideal for graph classification tasks. However, these methods may overlook important local details.

%Multi-task methods and their limitations
The previous multi-task GSSL methods combine several pretext tasks \textit{simultaneously} to improve task generalization. AutoSSL \cite{jin2021automated} employs a pseudo-homophily mechanism to assess the quality of representations across pretext tasks. Then, evolutionary algorithms or meta-gradient descent are used to identify the best linear combination of GSSL tasks. ParetoGNN \cite{ju2022multi} leverages diverse pretext tasks, including generative reconstruction, whitening decorrelation, and mutual information maximization. This model introduces a multi-gradient descent mechanism that promotes Pareto optimality across pretext tasks and mitigates potential conflicts. However, both approaches rely on an inner optimization process to search the hyperparameters associated with self-supervised losses, resulting in significant computational overhead. Moreover, both methods select the set of candidate pretext tasks based on heuristics and only focus on identifying the optimal combination of these tasks. The synergy across multiple graph abstraction levels
—node-level, proximity-level, cluster-level, and graph-level—remains unexplored in multi-task GSSL. We refer to this paradigm as multi-level GSSL.

\begin{figure}[t]
    \centering
    \begin{subfigure}[b]{0.15\textwidth}
        \centering
        \includegraphics[width=\linewidth]{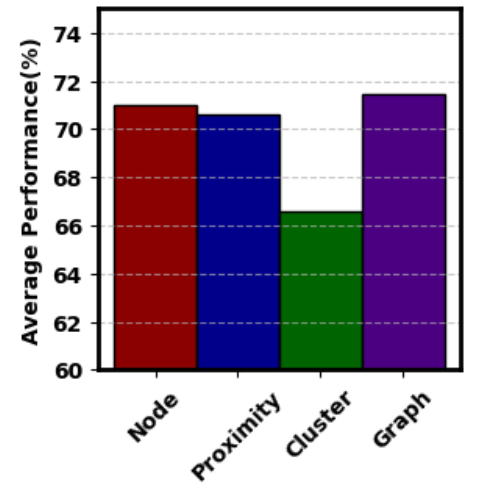}
    \end{subfigure}
    \hfill
    \begin{subfigure}[b]{0.15\textwidth}
        \centering
        \includegraphics[width=\linewidth]{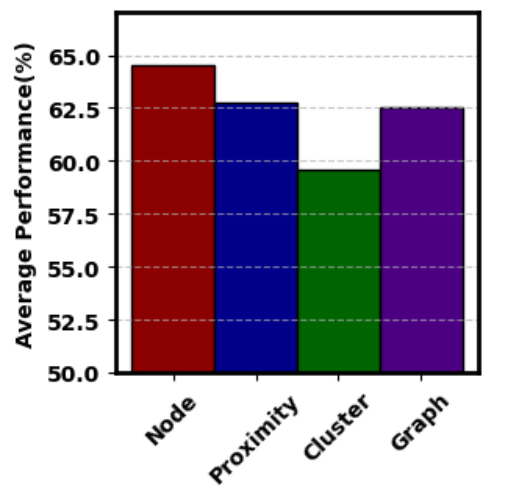}
    \end{subfigure}
    \hfill
    \begin{subfigure}[b]{0.15\textwidth}
        \centering
        \includegraphics[width=\linewidth]{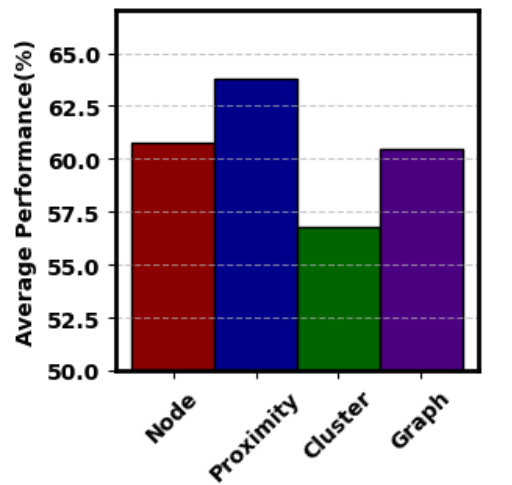}
    \end{subfigure}
    \hfill
    \begin{subfigure}[b]{0.15\textwidth}
        \centering
        \includegraphics[width=\linewidth]{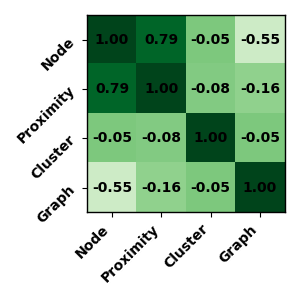}
        \caption{Cora.}
        \label{fig:cos_sim_cora}
    \end{subfigure}
    \hfill
    \begin{subfigure}[b]{0.15\textwidth}
        \centering
        \includegraphics[width=\linewidth]{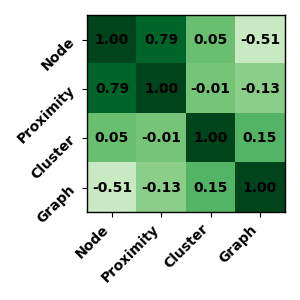}
        \caption{CiteSeer.}
        \label{fig:cos_sim_citeseer}
    \end{subfigure}
    \hfill
    \begin{subfigure}[b]{0.15\textwidth}
        \centering
        \includegraphics[width=\linewidth]{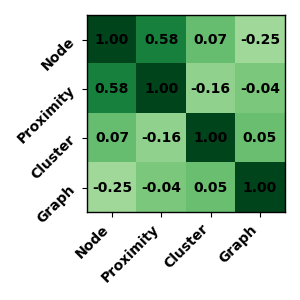}
        \caption{Pubmed.}
        \label{fig:cos_sim_pubmed}
    \end{subfigure}
    \caption{
   Top row: Average performance of our single-level GSSL model (SL-GSSL) across three downstream tasks (node classification, node clustering, and link prediction) at different abstraction levels on Cora, CiteSeer, and Pubmed datasets. Bottom row: Average cosine similarity between the gradients of two abstraction-level GSSL losses when trained using our linear multi-level GSSL model (L-ML-GSSL).}
    \label{fig:two_images}
    \vspace{-5mm}
\end{figure}

We introduce a unified framework that can operate \textit{seamlessly} at different abstraction levels through a linear combination of similarity and dissimilarity scores. The proposed framework yields competitive results compared with its corresponding state-of-the-art methods from each level. Using the unified framework, we analyze the correlation between the graph abstraction levels and their relevance to the downstream tasks. As illustrated in Fig. \ref{fig:two_images}, the GSSL abstraction levels do not perform equally well across datasets, and there is a potential conflict between them. To this end, we extend the linear combination of similarity and dissimilarity scores across all abstraction levels. As an advantage, multi-level learning can capture local and global information.

From another perspective, previous single-task and multi-task GSSL approaches \textit{implicitly} apply uniform penalty strengths to all examples, which limits their optimization flexibility and can lead to ambiguous convergence status \cite{sun2020circle}. We address this by introducing a dynamic weighting mechanism that adaptively assigns weights to individual similarity and dissimilarity scores across examples and abstraction levels, prioritizing those that significantly deviate from the optimum. We implement the weighting coefficients as linear functions w.r.t. their similarity/dissimilarity scores. This results in a hyperspherical decision boundary w.r.t. similarity/dissimilarity scores. This mechanism enhances optimization flexibility and ensures definite convergence status while eliminating the computational cost associated with inner optimization for hyperparameter tuning, as seen in conventional multi-task GSSL.
% \textbf{Contributions.} \begin{enumerate}
%     \item We propose the \textit{multi-level GSSL paradigm}, which exploits the synergy of multiple abstraction levels—node-level, proximity-level, cluster-level, and graph-level. 
%     \item We introduce a unified framework that can operate seamlessly at different abstraction levels and facilitate their integration by combining their similarity and dissimilarity scores. 
%     \item We then extend the unified framework to the multi-task scenario by leveraging multi-level graph information. Our multi-task formulation aligns with the new GSSL paradigm. 
%     \item We propose a dynamic self-weighting mechanism that adaptively assigns weights to individual similarity and dissimilarity scores across examples and abstraction levels, prioritizing the scores that significantly deviate from the optimum. This mechanism enhances optimization flexibility and ensures definite convergence status while eliminating the computational cost associated with training multiple expert networks and inner optimization for hyperparameter tuning. 
%     \item We conduct extensive experiments on several real-world datasets. Our results showcase the effectiveness of our approaches in single-task and multi-task learning. In particular, the proposed self-weighted approach achieves better task generalization by targeting multi-level graph information.
% \end{enumerate}

% Nairouz suggestion
\textbf{Contributions.} This work introduces a principled multi-level graph self-supervised learning (GSSL) framework. It unifies learning objectives across 4 abstraction levels and proposes a score-level self-weighted multi-task formulation with theoretical and empirical support. The contributions are summarized below:
\begin{itemize}
\item \textbf{Multi-level GSSL paradigm with explicit cross-level complementarity and conflict.} We formalize GSSL as learning from four abstraction levels: node, proximity, cluster, and graph.
We empirically show that no single level is consistently optimal across datasets and downstream tasks. Fig.~\ref{fig:two_images} (top row) summarizes these task-dependent behaviors. Fig.~\ref{fig:two_images} (bottom row) further shows that level-wise objectives can have weak gradient alignment under joint training, which motivates principled multi-level integration.

\item \textbf{Unified single-level contrastive objective across four abstraction levels (\textsc{SL-GSSL}).
We develop a single contrastive formulation that is instantiated at each level by changing only the level-specific positive/negative sample generators (Eqs.~(\ref{eq.pos_samples})--(\ref{eq.loss_unified})). The training pipeline is summarized in Algorithm~\ref{alg:slgssl}. We adopt a fixed augmentation and training configuration shared across datasets and downstream tasks (Table~\ref{Table:hyperparameters_indep}), and tune only a small set of data-dependent hyperparameters (Table~\ref{Table:hyperparameters_dep}). Under this unified setup, \textsc{SL-GSSL} achieves competitive or superior within-level performance compared with state-of-the-art single-task baselines (Table~\ref{table_1}). Ablations validate the design of the unified loss and show that margin removal, hinge replacement, and InfoNCE~\cite{chen2020asimple}  substitution degrade performance (Table~\ref{sl_ablation}).}

\item \textbf{Score-level multi-level integration and a unified multi-task formulation (\textsc{L-ML-GSSL} / \textsc{LW-ML-GSSL}).} We extend the unified formulation to multi-level learning by integrating positive and negative similarity scores across levels inside a single objective (Eqs.~(\ref{eq.loss_l_ml_gssl}) and (\ref{eq.loss_lw_ml_gssl})). This yields a multi-task GSSL setup where abstraction levels act as complementary tasks under a shared encoder (Fig.~\ref{fig:Model_Architecture}). Table~\ref{Table:table_3} shows that different levels contribute distinct information and that multi-level variants improve over single-level ones. At the same time, Table~\ref{linear_vs_self_weighted} shows that naive linear integration is not sufficient to consistently surpass the best single-level configuration, which motivates adaptive balancing.

\item \textbf{Dynamic score-level self-weighting with hyperspherical convergence geometry (\textsc{LSW-ML-GSSL}).} We propose a linear self-weighting mechanism that assigns weights to individual positive and negative similarity scores across levels (Eqs.~(\ref{positive_weight}) and (\ref{negative_weight})), resulting in the self-weighted multi-level objective (Eq.~(\ref{eq.loss_lsw_ml_gssl})). The induced decision boundary becomes hyperspherical in similarity-score space (Eqs.~(\ref{eq.decision_boundary}) and (\ref{eq.circle})). Gradient visualizations show improved optimization flexibility and reduced convergence ambiguity compared with uniform weighting (Figs.~\ref{fig:fig1}, \ref{fig:gr_lsw_sn}, and \ref{fig:gr_lsw_sp}). We provide a dedicated convergence analysis in similarity-score space (Theorem~\ref{thm:error_contraction} and Proposition~\ref{prop:no_radial_linear}). The method avoids auxiliary expert networks and inner-loop optimization, and remains computationally efficient (Fig.~\ref{fig:execution_time}).

\item \textbf{Extensive empirical validation, ablations, and robustness studies.} We evaluate on six real-world datasets and three downstream tasks. The proposed \textsc{LSW-ML-GSSL} achieves the best overall performance against state-of-the-art multi-task and multi-scale baselines (Table~\ref{table_2}). It consistently improves over linear multi-level variants and over the strongest single-level configuration (Table~\ref{linear_vs_self_weighted}). It also outperforms representative multi-loss weighting strategies (Table~\ref{lsw_vs_other_weighting_methods}). Sensitivity analyses indicate stable performance across wide hyperparameter ranges (Figs.~\ref{fig:sensitivity} and \ref{fig:sensitivity_revision}).
\end{itemize}

\section{Related Work}
\label{sec_related_work}
\textbf{Single-Task Methods.} 
Conventional GSSL methods construct self-supervision signals by extracting information from different graph abstraction levels, including individual nodes, local proximity, clusters, and the overall graph structure. 
At the \textit{node level}, methods such as GRACE \cite{zhu2020deep} and GraphCL \cite{you2020graph} learn node representations by maximizing agreement between augmented graph views. This agreement is achieved using the InfoNCE loss \cite{oord2018representation}, which serves as a lower-bound estimator of Mutual Information (MI). Building upon this, GCA \cite{zhu2021graph} enhances GRACE by introducing an adaptive augmentation strategy that selectively modifies edges and node features based on their importance, rather than applying random perturbations.
\textit{Proximity-level} GSSL methods differ from node-level approaches by eliminating the need for explicit data augmentation. Instead, these methods leverage neighborhood relationships to encode semantic invariance. Methods such as DeepWalk \cite{perozzi2014deepwalk} and Node2Vec \cite{grover2016node2vec} employ a random walk to capture proximity-based information. Another category of methods focuses on reconstructing the graph to capture proximity-level information. For example, MGAE \cite{wang2017mgae} and MaskGAE \cite{li2023s} randomly mask edges 
and train the model to reconstruct the missing connections. When the observed topology is noisy or imperfect, structure-learning methods aim to refine or infer a task-adaptive graph. For example, SLAPS~\cite{fatemi2021slaps} shows that self-supervised signals can improve graph structure learning for GNNs, leading to better downstream performance, while latent graph inference methods aim to recover an underlying graph structure from limited supervision \cite{lu2023latent}. From another perspective, methods such as NCLA \cite{shen2023neighbor} apply neighbor contrastive learning on learnable graph augmentations, enabling the joint learning of both augmentations and embeddings. \textit{Cluster-level} GSSL methods emphasize capturing intra-cluster similarities and inter-cluster distinctions by integrating clustering with embedding learning. These methods minimize a clustering-oriented loss. 
For instance, DAEGC \cite{wang2019attributed} uses an attention mechanism to generate expressive latent representations while simultaneously performing embedding clustering. GMM-VGAE \cite{hui2020collaborative} leverages a variational graph autoencoder with a latent Gaussian mixture model. \textit{Graph-level} GSSL methods generate representations that encode the global structure of graphs. For example, DGI \cite{velickovic2019deep} maximizes mutual information (MI) between localized patch representations and their corresponding high-level graph summaries. This objective ensures that the encoder captures features relevant across the graph. MVGRL \cite{hassani2020contrastive} builds on this by introducing a contrastive framework that maximizes MI between representations derived from local (i.e., node-level) and global (i.e., graph-level) structural views of graphs. Beyond contrastive and reconstruction-based GSSL, recent work has explored leveraging large language models (LLMs) for graph representation learning by encoding graphs into token or text-like sequences compatible with LLM processing \cite{fatemi2024talk,ICLR2025_ce3e23a2}.

%Multi-scale methods 
\textbf{Multi-Scale Methods.} Recent research has increasingly combined insights from multiple hierarchical levels of graphs \cite{jiao2020sub, jin2021anemone, ijcai2021p204, ijcai2023p246, li2023contrastive}. SUBG-CON \cite{jiao2020sub} is a self-supervised representation learning method that captures regional structural information by leveraging the strong correlation between central nodes and their corresponding sampled subgraphs. ANEMONE \cite{jin2021anemone} is a graph anomaly detection framework that identifies anomalies across multiple graph scales by employing a GNN-based encoder and multi-scale contrastive learning to capture pattern distributions through simultaneous agreement learning at the patch and context levels. MERIT \cite{ijcai2021p204} is another multi-scale graph contrastive learning approach that first generates two augmented views of the input graph, one focusing on local and the other on global perspectives. It then employs two objectives (cross-view and cross-network contrastiveness) to maximize the alignment of node representations across these different views and networks. Based on the idea that nodes can be observed at different abstraction levels, MNCSCL \cite{li2023contrastive} samples multiple node-centered subgraphs to reflect differences across various granularity levels. Then it applies contrastive learning to maximize mutual information between the graph views generated at different abstraction levels. In another graph self-supervised learning method, MSSGCL \cite{ijcai2023p246} generates multi-scale global and local views via subgraph sampling and builds multiple contrastive relations between these abstraction levels. LMGTA \cite{li2025anomaly} similarly adopts a multi-order contrastive strategy that integrates subgraph-level augmentations with a topology-aware global module, enabling the model to capture both local and structural irregularities. However, multi-scale GSSL methods often instantiate "scale" using only local–global or patch–context contrasts, providing limited coverage of intermediate abstractions. They often couple each scale with scale-specific objectives and augmentation heuristics, confounding the effect of abstraction level with objective engineering. Finally, they usually aggregate scale losses via fixed linear weighting, leaving inter-scale interference and overall optimization dynamics insufficiently characterized.

%Multi-level methods
\textbf{Multi-Task Methods.} 
Recently, extensive research has been devoted to learning from multiple tasks \cite{sener2018multi, mahapatra2020multi, liu2021profiling, navon2021learning, yu2020gradient, doersch2017multi, georgescu2021anomaly, yu2021learning}.
The motivation arises from single-pretext tasks favoring task-specific features, which leads to suboptimal performance across diverse downstream objectives. Several multi-task approaches have been developed for graph-structured data. 
For example, AutoSSL~\cite{jin2021automated} introduces pseudo-homophily as a metric to assess the quality of GSSL tasks and searches for an optimal combination of these tasks using evolution algorithms. ParetoGNN \cite{ju2022multi} leverages a multi-gradient descent algorithm to assign task weights that promote Pareto optimality. DyFSS \cite{zhu2024every} introduces a Mixture of Experts (MoE) framework that integrates features derived from various GSSL tasks, with a focus on node clustering. This model trains a gating network to learn node-specific weights for each task. From another perspective, GraphTCM \cite{fang2024exploring} models the correlations between GSSL tasks and exploits these correlations to derive representations that maximize performance. 
WAS \cite{fan2024decoupling} formulates the multi-task learning problem as multi-teacher knowledge distillation. Moreover, the authors of \cite{fan2024decoupling} highlight the significance of selecting a set of tasks based on their compatibility before assigning importance weights to them. 
The previous multi-task learning methods typically require a trainable gating network (e.g., DyFSS), trainable task-specific expert networks (e.g., GraphTCM, DyFSS, and WAS), and inner-optimization algorithms (e.g., the evolution algorithm for AutoSSL and the multi-gradient descent algorithm for ParetoGNN) to learn task-specific importance weights. 
Unlike previous methods, our approach employs a single multi-task network and does not require inner-optimization algorithms. The proposed method introduces a \textit{self-weighting} mechanism based on similarity and dissimilarity scores across different graph abstraction levels. Furthermore, our approach inherits the circle loss \cite{sun2020circle} advantages. In particular, it improves optimization flexibility and reduces convergence uncertainty.

\textbf{Multi-Loss Weighting.} Beyond multi-task learning, a broad body of literature has emerged on strategies for balancing multiple objectives during optimization. Uncertainty weighting \cite{kendall2018multi} learns scalar loss weights from homoscedastic task uncertainty, down-weighting objectives that appear noisier. Instead of modeling uncertainty, GradNorm \cite{chen2018gradnorm} dynamically rescales task losses to equalize gradient norms and encourage comparable learning speeds across tasks. Complementarily, gradient-conflict methods modify the update direction itself: PCGrad \cite{yu2020gradient} projects gradients to reduce pairwise conflicts, while CAGrad \cite{liu2021conflict} computes a conflict-averse direction that explicitly limits interference. More recent approaches address dominance effects over longer training horizons. For instance, AdaTask \cite{yang2023adatask} maintains task-wise accumulators under adaptive optimizers, and IGBv1 \cite{dai2023improvable} prioritizes tasks with larger improvable gaps between current and desired progress. As a recent example, AUTAUT \cite{zhongautomatic} uses LLM-based retrieval to identify candidate auxiliary tasks and adaptively reweights them via gradient alignment. However, these strategies typically operate at the task-loss granularity by assigning a single scalar weight to an entire objective. In contrast, we adopt a more fine-grained weighting strategy that reweights individual positive and negative similarity scores within each abstraction level, thereby controlling the relative influence of each score on the optimization process. 

\section{A Unified GSSL Perspective} 
\label{sec.unified_gssl}
Let \( \mathcal{G} = (\mathcal{V}, \, \mathcal{E}, \, \mathcal{X}) \) be an undirected attributed graph, where \( \mathcal{V} = \{ v_1, \ldots, v_n \} \) represents the set of \( n \) nodes, \( \mathcal{E} \subseteq \mathcal{V} \times \mathcal{V} \) is the set of edges, and \( \mathcal{X} \in \mathbb{R}^{n \times d'} \) is the node feature matrix. The $i^{\text{th}}$ row \( \mathbf{x}_i \) of $\mathcal{X}$ is the feature vector of node $v_{i}$. We define the adjacency matrix of \(\ \mathcal{G} \) as \( \mathbf{A} = (a_{ij}) \in \mathbb{R}^{n \times n} \), such that \( a_{ij} = 1 \) if \( (v_i, v_j) \in \mathcal{E} \) and \( a_{ij} = 0 \) otherwise.   
A graph neural network (GNN) encodes the input graph $\mathcal{G}$ into a $d$-dimensional latent space and generates the node embedding matrix \( \mathbf{H} \in \mathbb{R}^{n \times d} \). The $i$-th row \( \mathbf{h}_i \) of \( \mathbf{H} \) represents the embedding vector of node $v_{i}$. In this section, we introduce a unified contrastive learning framework that operates seamlessly at different graph abstraction levels and facilitates their integration.

We generate three augmented views of $\mathcal{G}$, which are then encoded by the GNN encoder to derive node representations for contrastive learning. At each epoch, the input graph undergoes three augmentations, denoted as 
\( \widetilde{\mathcal{G}}^{(1)} = \{ \widetilde{\mathcal{V}}^{(1)}, \, \widetilde{\mathcal{E}}^{(1)}, \, \widetilde{\mathcal{X}}^{(1)} \} \), 
\( \widetilde{\mathcal{G}}^{(2)} = \{ \widetilde{\mathcal{V}}^{(2)}, \, \widetilde{\mathcal{E}}^{(2)}, \, \widetilde{\mathcal{X}}^{(2)} \} \), and
\( \widetilde{\mathcal{G}}^{(3)} = \{ \widetilde{\mathcal{V}}^{(3)}, \, \widetilde{\mathcal{E}}^{(3)}, \, \widetilde{\mathcal{X}}^{(3)} \} \). These augmentations are categorized into positive and negative transformations. Positive augmentations preserve key structural and semantic properties of the graph. We employ stochastic edge dropping and node feature masking as positive augmentations \cite{zhu2020deep, zhu2021graph}. In contrast, negative augmentations introduce substantial modifications that alter the graph’s structure or feature distribution. We use random node shuffling as negative augmentation. This corruption breaks node-attribute alignment and is commonly used to generate negative views in GSSL methods such as DGI \cite{velickovic2019deep} and MVGRL \cite{hassani2020contrastive}. The first two graphs $\widetilde{\mathcal{G}}^{(1)}$ and $\widetilde{\mathcal{G}}^{(2)}$ are positive views, while the third $\widetilde{\mathcal{G}}^{(3)}$ is negative. 

For each augmented view \( \widetilde{\mathcal{G}}^{(j)} \), we define \( \mathbf{H}^{(j)} \in \mathbb{R}^{n \times d} \) as its node embedding matrix and \( \mathbf{h}_i^{(j)} \) as the embedding vector of node $v_{i}$. The core principle of graph contrastive learning involves two steps: (i) generating positive and negative samples for each anchor, and (ii) bringing the anchor closer to positive samples while distancing it from the negative ones. Formally, the problem can be framed as maximizing the similarity between the anchor and the positive samples, while minimizing the similarity between the anchor and the negative samples. Our unified contrastive learning framework aligns with the multi-level GSSL paradigm by capturing similarities and dissimilarities at four graph abstraction levels: node-level, proximity-level, cluster-level, and graph-level. The index \( l \) selects the granularity level at which the GSSL is performed. Formally, we have \(l \in \{ \text{node, proximity, cluster, graph} \} \). 

The anchors are defined as the node representations in the first augmented graph view, \( \widetilde{\mathcal{G}}^{(1)} \). For each anchor, \( n_1 \) positive samples are drawn from the second graph view, \( \widetilde{\mathcal{G}}^{(2)} \), while \( n_2 \) negative samples are selected from the second and third views, \( \widetilde{\mathcal{G}}^{(2)} \) and \( \widetilde{\mathcal{G}}^{(3)} \). Sample generation is governed by the functions \( G_l^+ \) and \( G_l^- \), which operate at a granularity level \( l \). Here, \( G_{k, \, l}^+ \) produces the \( k \)-th positive sample, and \( G_{k,\,l}^- \) generates the \( k \)-th negative sample.

We employ a function \( \theta \) to measure the similarity between anchors and their positive and negative samples. The similarity between the \( i \)-th anchor and its \( k \)-th positive sample at granularity level \( l \) is denoted as \( s^+_{ik,\,l} \), where \( (i, \, k) \in \{1, \dots, n\} \times \{1, \dots, n_1\} \). Similarly, the similarity between the \( i \)-th anchor and its \( k \)-th negative sample is given by \( s^-_{ik,\,l} \), where \( (i, \, k) \in \{1, \dots, n\} \times \{1, \dots, n_2\} \). The positive and negative similarities are formally expressed as:  
\begin{equation}
\label{eq.pos_samples}
    s^+_{ik,\,l} = \theta\big(\mathbf{h}^{(1)}_{i}, G^+_{k,\,l}(v_{i})\big),
\end{equation}

\begin{equation}
\label{eq.neg_samples}
    s^-_{ik,\,l} = \theta\big(\mathbf{h}^{(1)}_{i}, G^-_{k,\,l}(v_{i})\big).
\end{equation}

At the node level, the positive samples for an anchor consist of a single element, which is the embedding of the same node in \( \widetilde{\mathcal{G}}^{(2)} \). The negative samples contain the embeddings of \( n_2 \) nodes, different from the anchor node, selected from the embeddings of \( \widetilde{\mathcal{G}}^{(2)} \) and \( \widetilde{\mathcal{G}}^{(3)} \). At the proximity level, positive samples are generated by considering the embeddings of \( \widetilde{\mathcal{G}}^{(2)} \) for nodes along a \( \kappa \)-length path in \(\mathcal{G}\) originating from the anchor node. Negative samples are generated by considering the embeddings of \( \widetilde{\mathcal{G}}^{(2)} \) and \( \widetilde{\mathcal{G}}^{(3)} \) for nodes along a \( \kappa \)-length path in \(\mathcal{G}\) that does not intersect the anchor's \( \kappa \)-hop neighborhood in \(\mathcal{G}\). At the cluster level, the graph is first partitioned by applying k-means to the node features $\mathcal{X}$. The positive samples are selected from the embeddings of \( \widetilde{\mathcal{G}}^{(2)} \) for nodes within the same cluster as the anchor node, while the negative samples are drawn from the embeddings of \( \widetilde{\mathcal{G}}^{(2)} \) and \( \widetilde{\mathcal{G}}^{(3)} \) for nodes outside the anchor's cluster. At the graph level, the positive samples are drawn from the embeddings of  \( \widetilde{\mathcal{G}}^{(2)} \), whereas the negative samples come from the embeddings of \( \widetilde{\mathcal{G}}^{(3)} \).

%For notational simplicity, we omit the sample indices and denote the similarities of positive and negative pairs by \(s^{+}\) and \(s^{-}\), respectively. In this work, we instantiate \(\theta\) as the shifted cosine similarity:
In the general case, the similarities are indexed by the anchor and sample identifiers. For notational simplicity, we omit these indices when no ambiguity arises, and denote the similarities of positive and negative pairs by \(s^{+}\) and \(s^{-}\), respectively. We instantiate \(\theta\) as the shifted cosine similarity:
\begin{equation}
\label{eq:shifted_cosine}
\begin{aligned}
\theta(\mathbf{u},\mathbf{v}) &= \frac{1+\mathrm{cos}(\mathbf{u},\mathbf{v})}{2} \\
&= \frac{1}{2}\left(1+\frac{\mathbf{u}^\top\mathbf{v}}{\|\mathbf{u}\|_2\,\|\mathbf{v}\|_2}\right)\in[0,1].
\end{aligned}
\end{equation}
The target values are $s^+ \to 1$ (cosine $\to 1$) and $s^- \to 0$ (cosine $\to -1$). For terminological convenience, we use the term \emph{dissimilarity score} to refer to the same similarity function \(\theta(\cdot,\cdot)\) evaluated on negative pairs (i.e., \(s^-\)). That is, we do not introduce a separate distance/dissimilarity function. Under the shifted cosine similarity \(\theta\in[0,1]\), minimizing the negative-pair similarity \(s^-\) is equivalent to maximizing dissimilarity in the representation space (ideally \(s^- \to 0\)). 

We propose a unified contrastive loss that can operate at different granularity levels. For each anchor node, this loss function maximizes the similarity between the anchor node representation and its positive samples, while minimizing the similarity with its negative samples. 
The unified loss function \( \mathcal{L}_{\text{SL-GSSL}} \) for single-level graph self-supervised learning (SL-GSSL) is defined as follows:

\begin{equation}
\label{eq.loss_unified}
    \mathcal{L}_{\text{SL-GSSL}}^{(l)} = \frac{1}{n} \sum_{i=1}^{n} \log \left[ 1 + \sum_{j=1}^{n_1} \sum_{k=1}^{n_2} \exp \left( \gamma (s^-_{ik,\,l} - s^+_{ij,\,l} + m) \right) \right],
\end{equation}

\noindent where \( \gamma \) denotes a scaling factor, and \( m \) is a margin hyperparameter that prevents excessive updates by disregarding dissimilar pairs once they surpass a predefined separation threshold. This loss function iterates through each similarity pair to minimize \(( s^-_{ik,\,l} - s^+_{ij,\,l} \)). Algorithm  \ref{alg:slgssl} outlines the training procedure of SL-GSSL. 

\begin{algorithm}
\caption{
The unified perspective SL-GSSL}
\label{alg:slgssl}
\begin{algorithmic}[1]

\STATE \textbf{Input:} The input graph $\mathcal{G}$, \# of epochs T, scaling factor $\gamma$, margin $m$, \# of positive samples $n_1$, \# of negative samples $n_2$, abstraction level $l \in \{\text{node, proximity, cluster, graph}\}$, similarity function $\theta$, positive augmentations $O^+_1$ and $O^+_2$, negative augmentation $O^-$, functions that generates positive $G^+_l$, and function that generates negative samples $G^-_l$
\STATE \textbf{Output:} Trained model

\FOR{$\text{epoch} = 1$ to $T$}
    \STATE Generate \( \widetilde{\mathcal{G}}^{(1)}, \, \widetilde{\mathcal{G}}^{(2)}, \text{ and } \widetilde{\mathcal{G}}^{(3)} \) by stochastically corrupting \( \mathcal{G} \) using \( O^+_1 \), \( O^+_2 \), and \( O^- \), respectively. 
    \STATE Compute \(\mathbf{H}^{(1)}, \, \mathbf{H}^{(2)}, \text{ and } \mathbf{H}^{(3)} \) by encoding the graphs \( \widetilde{\mathcal{G}}^{(1)}, \, \widetilde{\mathcal{G}}^{(2)}, \text{ and } \widetilde{\mathcal{G}}^{(3)} \), respectively, using the GNN. 
    \STATE Generate $n_1$ positive samples and $n_2$ negative samples for each node using \(G^+_l \text{ and } G^-_l\), respectively.
    \STATE Compute the positive and negative similarity scores for each node using Eq. (\ref{eq.pos_samples}) and Eq. (\ref{eq.neg_samples}), respectively.
    \STATE Compute the loss function \(\mathcal{L}_{\text{SL-GSSL}}^{(l)}\) using Eq. (\ref{eq.loss_unified}).
    \STATE Update the model's parameters using Adam optimizer. 
\ENDFOR

\end{algorithmic}
\end{algorithm}

\textbf{Gradient Analysis.} To simplify visualization, we assume the toy scenario of a single anchor and one corresponding positive and negative sample (\(n = n_1 = n_2\)). The unified loss function becomes:
\begin{equation}
\mathcal{L} = \log \left[ 1 +  \exp \left( \gamma (s^- - s^+ + m) \right) \right].
\end{equation}
For notational simplicity, we omit the granularity index \( l \). We analyze the gradient of \(\mathcal{L}\) w.r.t. the positive and negative similarities: 
\begin{equation}
    \left| \frac{\partial \mathcal{L}}{\partial s^+} \right| =  \left| \frac{\partial \mathcal{L}}{\partial s^-} \right| = \frac{\gamma \, \exp \left( \gamma (s^- - s^+ + m) \right)}{1 + \exp \left( \gamma (s^- - s^+ + m) \right)}.
\end{equation}

Since the two expressions \( \left| \frac{\partial \mathcal{L}}{\partial s^+} \right|\) and \( \left| \frac{\partial \mathcal{L}}{\partial s^-} \right|\) are mathematically identical, we plot only one. 
Fig. \ref{fig:fig1} illustrates the partial derivative of \( \mathcal{L} \) w.r.t. \( s^+ \). The figure shows a sharp transition from large to very small values, reflecting the saturation of the logistic term once the margin constraint is satisfied. This transition induces an effective (soft) decision boundary in the \( (s^+, s^-) \) plane. In particular, the line \( s^- - s^+ + m = 0 \) corresponds to the midpoint of the transition (where \( \left| \frac{\partial \mathcal{L}}{\partial s^+} \right| = \gamma/2 \)) and separates a high-gradient regime (\( s^- - s^+ + m > 0 \)) from a low-gradient regime (\( s^- - s^+ + m < 0 \)). Note that the gradient remains strictly positive for any finite value of \( s^- - s^+ + m \), but it decays exponentially as \( s^- - s^+ + m \) becomes more negative.
%Fig. \ref{fig:fig1} illustrates the partial derivative of \( \mathcal{L} \) w.r.t. \( s^+ \). The figure shows a sharp and abrupt transition to zero, which indicates a sudden vanishing gradient. This phenomenon defines a decision boundary, a region in the parameter space where the gradient diminishes abruptly. In this scenario, the decision boundary is represented by the line defined by the equation \( s^- - s^+ + m = 0 \). The gradient is strictly positive when  \( s^- - s^+ + m \geq 0 \) and vanishes rapidly when (\( s^- - s^+ + m \)) becomes negative.

\begin{figure*}[t]
    \centering
    \includegraphics[width=0.95\linewidth]{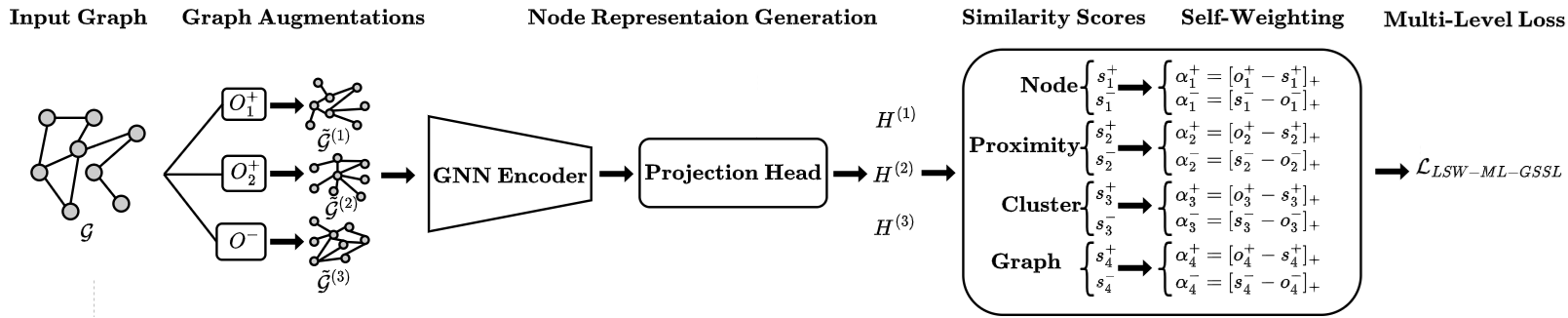}
    \caption{Illustration of our multi-task approach, LSW-ML-GSSL. The positive and negative similarity scores are denoted by \( s^+_l \) and \( s^-_l \), respectively, with their corresponding reference endpoints denoted as \( o^+_l \) and \( o^-_l \). The corresponding weights in the multi-level loss \( \mathcal{L}_{\text{LSW-ML-GSSL}} \) are denoted by \( \alpha^+_l \) and \( \alpha^-_l \). $H^{(i)}$ denotes the embedding of the augmented graph $\tilde{\mathcal{G}}^{(i)}$.}
    \label{fig:Model_Architecture}
\end{figure*}

\begin{figure}[t]
    \centering
    \includegraphics[width=0.95\linewidth]{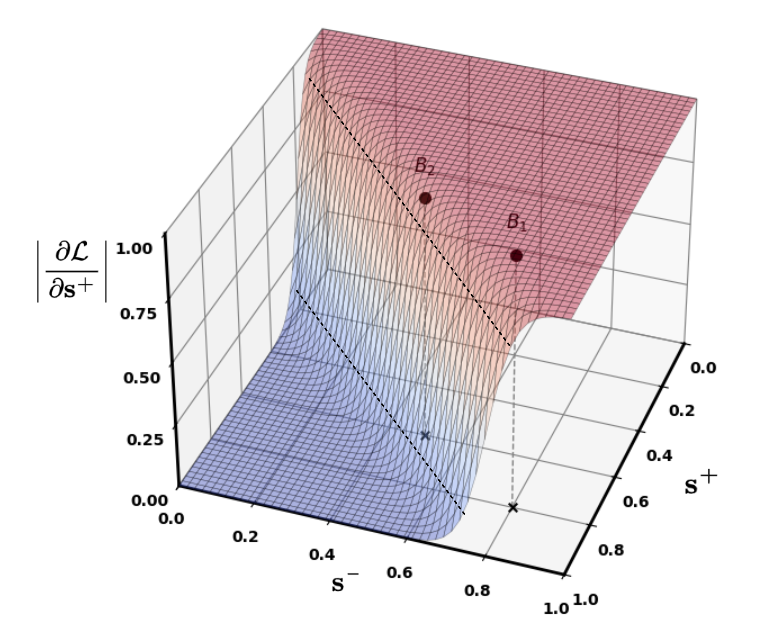} 
    \caption{The gradient magnitude of \( \mathcal{L} \) w.r.t. \( s^+ \) \Big(\( \left| \frac{\partial \mathcal{L}}{\partial s^+} \right| \)\Big).
} 
    \label{fig:fig1} 
\end{figure}

We select two points, \( B_{1}(0.8, 0.8) \) and \( B_{2}(0.6, 0.5) \), on the gradient function graph. Point \( B_{1} \) exhibits high positive and negative similarities, yet still maintains a large derivative magnitude w.r.t. \( s^+ \). The derivative magnitude remains identical w.r.t. $s^-$ and $s^+$, regardless of their relative balance. This indicates that the optimization process is limited in its flexibility to adapt according to the balance between \( s^+ \) and \( s^- \). Point \( B_{2} \) lies closer to the decision boundary \( s^- - s^+ + m = 0 \) compared to point \( B_{1} \). However, the derivatives at point \( B_{2} \) w.r.t. \( s^+ \) and \( s^- \) remain nearly identical to those at point \( B_{1} \). Consequently, the loss function penalizes both points equally, regardless of their relative proximity to the boundary. Moreover, for any pair \( (s^+, s^-) \) on the convergence boundary (i.e., \( s^+ - s^- = m \)), the model exhibits no preference between these points. Thus, the optimization process is susceptible to ambiguity in the convergence outcome.

\section{Multi-Level Approach}
We extend the unified formulation to the multi-task scenario by leveraging multi-level granularity information.  
Initially, we aggregate the positive and negative similarity scores linearly across all abstraction levels. Then, we present a self-weighting mechanism to enhance optimization flexibility and reduce convergence ambiguity. Our multi-level approach, LSW-ML-GSSL, is illustrated in Fig. \ref{fig:Model_Architecture}.

\textbf{Linear Combination.} The effectiveness of single-level GSSL on each downstream task varies depending on the selected graph granularity level. To ensure task generalization, we integrate positive and negative similarity scores across four graph granularity levels into the same loss function through a linear combination.

We define \( \Delta_{ijk, \, l} \) as the contrastive similarity difference between the \( i \)-th anchor, its \( j \)-th positive sample, and its \( k \)-th negative sample at granularity level \( l \). The expression of \( \Delta_{ijk, \, l} \) is given by:

\begin{equation}
\label{eq.residual_similarity_score}
\Delta_{ijk, \, l} = s^-_{ik, \, l} - s^+_{ij, \, l} + m.
\end{equation}

We perform a linear combination of the contrastive similarity differences across the four graph granularity levels. The linear multi-level GSSL loss, denoted as \( \mathcal{L}_{\text{L-ML-GSSL}}^{(l)} \), is expressed as follows:

\begin{equation}
\label{eq.loss_l_ml_gssl}
\mathcal{L}_{\text{L-ML-GSSL}} = \frac{1}{n} \sum_{i=1}^{n} \log \left[ 1 + \sum_{j=1}^{n_1} \sum_{k=1}^{n_2} \exp \left( \gamma \sum_{l=1}^4 \beta_l \; \Delta_{ijk, \, l} \right) \right],
\end{equation}
\noindent where \( \beta_l \) is a balancing hyperparameter that controls the contribution of the \( l \)-th graph granularity level to the loss function. We leverage enumeration for the index \( l \) to simplify notation, where \( l = 1 \) corresponds to the node level, \( l = 2 \) to the proximity level, \( l = 3 \) to the cluster level, and \( l = 4 \) to the graph level.

\textbf{Linear Weighted Combination.} Similar to the single-level GSSL, discussed in Sec. \ref{sec.unified_gssl}, the training process w.r.t. the loss function \( \mathcal{L}_{\text{L-ML-GSSL}} \) lacks optimization flexibility and is prone to ambiguity in its convergence outcome. To address this issue, we introduce a dynamic weighting mechanism that adjusts the contribution of each similarity score to the final loss. The goal is to optimize each score at its own pace by prioritizing the less optimized ones. 

We define \( \Delta'_{ijk, \, l} \) as the weighted contrastive similarity difference between the \( i \)-th anchor, its \( j \)-th positive sample, and its \( k \)-th negative sample at granularity level \( l \). The expression of \( \Delta'_{ijk, \, l} \) is:

\begin{equation}
\label{eq.residual_similarity_score_prime}
\Delta'_{ijk, \, l} = \alpha^-_{ik, \, l} \: (s^-_{ik, \, l} - \delta^-) - \alpha^+_{ij, \, l} \: (s^+_{ij, \, l} - \delta^+),
\end{equation}

\noindent where \( \alpha^+_{ij, \, l} \) and \( \alpha^-_{ik, \, l} \) are the weighting coefficients assigned to the positive and negative similarity scores \( s^+_{ij, \, l} \) and \( s^-_{ik, \, l} \), respectively;  \(\delta^+\) and \(\delta^-\) are the positive and negative margins. In SL-GSSL and L-ML-GSSL, the positive and negative similarity scores are assigned equal weights. This allows the use of a single margin hyperparameter $m$. After integrating the weighting coefficients, \( s^+_{ij, \, l} \) and \( s^-_{ik, \, l} \) are no longer in symmetric positions. Thus, we need two margins.

We sum the weighted contrastive similarity differences across the four graph granularity levels. The linear weighted multi-level GSSL loss, denoted as \( \mathcal{L}_{\text{LW-ML-GSSL}}^{(l)} \), is formally expressed as follows:

\begin{equation}
\label{eq.loss_lw_ml_gssl}
\mathcal{L}_{\text{LW-ML-GSSL}} = \frac{1}{n} \sum_{i=1}^{n} \log \left[ 1 + \sum_{j=1}^{n_1} \sum_{k=1}^{n_2} \exp \left( \gamma \sum_{l=1}^4 \, \Delta'_{ijk, \, l} \right) \right].
\end{equation}

The hyperparameters \( (\beta_l)_{l \in \{1, \cdots, 4\}} \) that control the contribution of the graph granularity levels in \(\mathcal{L}_{\text{L-ML-GSSL}}\), are absorbed by the weighting coefficients \( \alpha^+_{ij, \, l} \) and \( \alpha^-_{ij, \, l} \) in \(\mathcal{L}_{\text{LW-ML-GSSL}}\).

%\textbf{Linear Self-Weighted Combination.} We denote by \(o_{ij,\, l}^+\) and \(o_{ik,l}^-\) as the optimal values of \(s_{ij,\, l}^+\) and \( s_{ik,l}^- \). 
\textbf{Linear Self-Weighted Combination.} We introduce $o_l^+$ and $o_l^-$ as \emph{reference endpoints} (level-wise constants) used in the self-weighting
coefficients for $s^+_{ij,l}$ and $s^-_{ik,l}$. Under shifted cosine similarity, \(s_{ij,\, l}^+, s_{ik,l}^- \in [0,1]\), and the feasible targets remain \(s^+ \to 1\) and \(s^- \to 0\). The weighting terms \( \alpha^+_{ij, \, l} \) and \( \alpha^-_{ik, \, l} \) control the gradient contribution of each similarity score during optimization. Ideally, the positive similarity scores should be high, while the negative scores should be low. When a similarity score deviates significantly from its target value, its corresponding weight should increase to enforce a stronger correction. Thus, we can define the self-weighting coefficients  \( \alpha^+_{ij, \, l} \) and \( \alpha^-_{ik, \, l} \) as follows:

% \begin{equation}
%     \label{positive_weight}
%     \alpha^+_{ij, \, l} = [o_{ij,\, l}^+ - s_{ij,\, l}^+]_+,
% \end{equation} 
% \begin{equation}
%     \label{negative_weight}
%     \alpha^-_{ik, \, l} = [s_{ik,l}^- - o_{ik,l}^-]_+,
% \end{equation}
\begin{equation}
\label{positive_weight}
\alpha^+_{ij, l} = [\,o_{l}^+ - s_{ij, l}^+\,]_+,
\end{equation}
\begin{equation}
\label{negative_weight}
\alpha^-_{ik, l} = [\,s_{ik,l}^- - o_{l}^-\,]_+,
\end{equation}

\noindent where \( [\cdot]_+ \) represents the ``cut-off at zero'' operation, ensuring the weights remain non-negative. 
%If the negative similarity score \( s_{ik,l}^- \) exceeds its optimal value \( o_{ik,l}^- \) significantly, then the weighting coefficient \( \alpha^-_{ik, \, l} \) increases. A higher \( \alpha^-_{ik, \, l} \) value triggers a stronger update w.r.t. \( s_{ik,l}^- \). Likewise, if the positive similarity score \( s_{ij,\, l}^+ \) falls significantly below its target value \( o_{ij,\, l}^+ \), then the coefficient \( \alpha^+_{ij, \, l} \) increases. A higher \( \alpha^+_{ij, \, l} \) reinforces the update w.r.t. \( s_{ij,\, l}^+ \).
Under the shifted cosine similarity, the desired (feasible) targets are $s^+_{ij,l}\to 1$ for positive pairs and $s^-_{ik,l}\to 0$ for negative pairs.
Accordingly, when a negative similarity score $s^-_{ik,l}$ is large (i.e., far above its target $0$), its weight $\alpha^-_{ik,l}$ increases, strengthening the update that reduces $s^-_{ik,l}$. Likewise, when a positive similarity score $s^+_{ij,l}$ is small (i.e., far below its target $1$), its weight $\alpha^+_{ij,l}$ increases, strengthening the update that increases $s^+_{ij,l}$. The quantities $o^+_l$ and $o^-_l$ are \emph{reference endpoints} used only in the weighting functions (Eqs.~(\ref{positive_weight})--(\ref{negative_weight})); they are not similarity targets and may lie outside $[0,1]$.

We define \( \Delta''_{ijk, \, l} \) as the self-weighted contrastive similarity difference between the \( i \)-th anchor, its \( j \)-th positive sample, and its \( k \)-th negative sample at \( l \)-th level. The expression of \( \Delta''_{ijk, \, l} \) is:

\begin{equation}
\label{eq.residual_similarity_score_second}
\Delta''_{ijk, \, l} = [s_{ik,l}^- - o_{l}^-]_+   \; \Big(s^-_{ik, \, l} - \delta^-\Big) - [o_{l}^+ - s_{ij,\, l}^+]_+ \; \Big(s^+_{ij, \, l} - \delta^+\Big),
\end{equation}

We sum the self-weighted contrastive similarity differences across the four graph granularity levels. The linear self-weighted multi-level GSSL loss, denoted as \( \mathcal{L}_{\text{LSW-ML-GSSL}}^{(l)} \), is expressed as follows:

\begin{equation}
\label{eq.loss_lsw_ml_gssl}
\mathcal{L}_{\text{LSW-ML-GSSL}} = \frac{1}{n} \sum_{i=1}^{n} \log \left[ 1 + \sum_{j=1}^{n_1} \sum_{k=1}^{n_2} \exp \left( \gamma \sum_{l=1}^4 \, \Delta''_{ijk, \, l} \right) \right].
\end{equation}

After omitting the indices $i$, $j$, and $k$ for the sake of simplicity and substituting the self-weighting coefficients \( \alpha^+_{ij, \, l} \) and \( \alpha^-_{ik, \, l} \) with their respective values, the decision boundary associated with $\mathcal{L}_{\text{LSW-ML-GSSL}}$ can be formally expressed as follows:

% \begin{equation}
% \label{eq.decision_boundary}
% \begin{split}
%     \sum_{l=1}^4 (s^-_{l} - \frac{o^-_{l} + \delta^-}{2})^2 + (s^+_{l} - \frac{o^+_{l} + \delta^+}{2})^2 \\
%     = \frac{1}{4} \sum_{l=1}^4 (\delta^- - o^-_{l})^2 + (o^+_{l} - \delta^+)^2. 
% \end{split}
% \end{equation}
\begin{equation}
\label{eq.decision_boundary}
\begin{split}
\sum_{l=1}^4 \left[
\left(s^-_{l} - \frac{o^-_{l} + \delta^-}{2}\right)^2
+
\left(s^+_{l} - \frac{o^+_{l} + \delta^+}{2}\right)^2
\right]  \\
=
\frac{1}{4} \sum_{l=1}^4 \left[
(\delta^- - o^-_{l})^2
+
(o^+_{l} - \delta^+)^2
\right].
\end{split}
\end{equation}
Eq. (\ref{eq.decision_boundary}) represents a 7-dimensional hypersphere, which is a 7-dimensional manifold embedded in 8-dimensional space, with a radius 
$r=\frac{1}{2}\sqrt{\sum_{l=1}^4\left[(\delta^- - o^-_l)^2 + (o^+_l-\delta^+)^2\right]}$. %The loss function expects that \( s^+ > \delta^+ \) and \( s^- < \delta^- \). Following these constraints, we reparametrize $\delta^+$, $\delta^-$, $o^+_l$, $o^-_l$ with a single margin hyperparameter $m$ by setting $\delta^+ = 1-m$, $\delta^- = m$,  $ o^+_l = 1+m$ and $o^-_l = -m$  for all $l\in \{1, \cdots, 4\}$. Then, the decision boundary's equation becomes as follows: 
The loss function expects that \( s^+ > \delta^+ \) and \( s^- < \delta^- \). We use a single margin hyperparameter \(m\) by setting \(\delta^+=1-m\) and \(\delta^-=m\). We then set the reference endpoints to \(o_l^+=1+m\) and \(o_l^-=-m\) for all \(l\in\{1,\dots,4\}\). These endpoint values are used only in the weighting functions in Eqs.~(\ref{positive_weight})--(\ref{negative_weight}). They are not similarity scores and they are not required to lie in \([0,1]\). The stationary targets induced by Eq.~(\ref{eq.decision_boundary}) are the midpoints \(\frac{o_l^+ + \delta^+}{2}=1\) and \(\frac{o_l^- + \delta^-}{2}=0\), which lie in \([0,1]\) under the shifted cosine similarity. Then, the decision boundary's equation becomes as follows:
% \begin{equation}
% \label{eq.circle}
%     \sum_{l=1}^4 (s^-_{l} - 0)^2 + (s^+_{l} - 1)^2 = 8m^2.
% \end{equation}
\begin{equation}
\label{eq.circle}
\sum_{l=1}^4 \left[
(s^-_{l})^2 + (s^+_{l} - 1)^2
\right] = 8m^2.
\end{equation}

The only parameter in Eq. (\ref{eq.circle}) is $m$, which determines the radius of the decision boundary. It can be interpreted as a relaxation factor that adjusts the level of flexibility during the optimization process.
    
\textbf{Gradient Analysis.} To simplify visualization, we will focus on the single abstraction level case with one anchor and one positive and negative sample. The linear self-weighted loss becomes: 

\begin{equation}
\mathcal{L} = \log \left[ 1 +  \exp \left( \gamma (\alpha^-(s^- -\delta^-) - \alpha^+(s^+-\delta^+)) \right) \right].
\end{equation}

We study the gradient of $\mathcal{L}$ w.r.t. the positive and negative similarities. The corresponding partial derivatives are given by:  

\begin{equation}
\begin{aligned}
    \left|\frac{\partial \mathcal{L}}{\partial s^-}\right|
&=  \frac{2 \, \gamma \, s^- \, \exp(\gamma \, \Delta'')}{1 + \exp(\gamma \Delta'')},
\end{aligned}
\end{equation}
\begin{equation}
   \left|\frac{\partial \mathcal{L}}{\partial s^+}\right| = \frac{2 \, \gamma \, \left| \delta^+ - s^+  \right|\exp(\gamma \Delta'')}{1 + \exp(\gamma \, \Delta'')},
\end{equation}

\noindent where, $\Delta''=\alpha^-(s^- -\delta^-) - \alpha^+(s^+-\delta^+)$, in both equations. 

\begin{figure}[ht]
    \centering
    \includegraphics[width=0.95\linewidth]{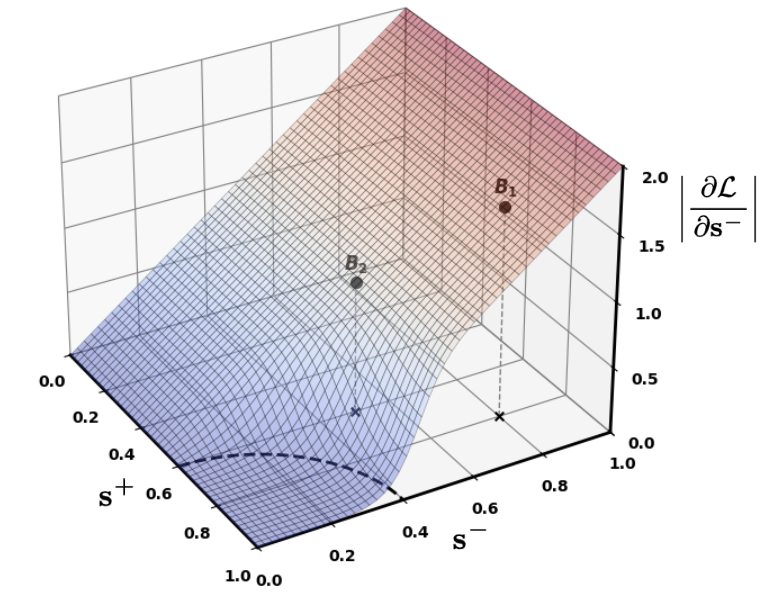} 
    \caption{The gradient magnitude of \( \mathcal{L} \) w.r.t. \( s^- \) \Big(\( \left| \frac{\partial \mathcal{L}}{\partial s^-} \right| \)\Big).
} 
    \label{fig:gr_lsw_sn} 
\end{figure}

\begin{figure}[ht]
    \centering
    \includegraphics[width=0.95\linewidth]{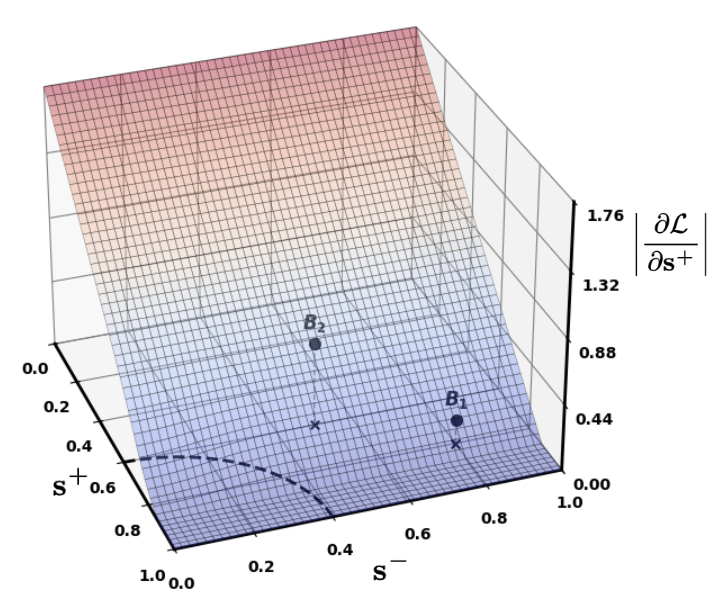} 
    \caption{The gradient magnitude of \( \mathcal{L} \) w.r.t. \( s^+ \) \Big(\( \left| \frac{\partial \mathcal{L}}{\partial s^+} \right| \)\Big).
} 
    \label{fig:gr_lsw_sp} 
\end{figure}

%Figs. \ref{fig:gr_lsw_sn} and  \ref{fig:gr_lsw_sp} depict the gradient magnitude of $\mathcal{L}$ w.r.t. $s^-$ and $s^+$, respectively. Unlike Fig. \ref{fig:fig1}, where the partial derivative w.r.t. $s^+$ and $s^-$ remain large even when the similarity scores are close to their reference endpoints. The self-weighting mechanism effectively addresses this problem. Specifically, the weighting coefficients dynamically adjust the gradient based on the deviation of similarity scores from their reference endpoints. For $B_1(0.8,0.8)$ (both $s^+$ and $s^-$ are large, but $s^-$ deviates significantly from its optimum), the loss function prioritizes optimizing $s^-$ by assigning a large partial derivative w.r.t. $s^-$ and a small partial derivative w.r.t. $s^+$. Conversely, at $B_2(0.6, 0.5)$, $s^-$ is significantly smaller, resulting in a notably smaller partial derivative w.r.t. $s^-$. This demonstrates that the proposed self-weighting mechanism improves optimization flexibility. Moreover, the new convergence boundary (a circle in our toy scenario) reduces the ambiguity of convergence that arises when positive and negative similarities have equal weights.
Figs.~\ref{fig:gr_lsw_sn} and \ref{fig:gr_lsw_sp} depict the gradient magnitude of $\mathcal{L}$ w.r.t.\ $s^-$ and $s^+$, respectively. In contrast to Fig.~\ref{fig:fig1}, where the partial derivatives remain large even when the similarity scores are close to their target values (i.e., $s^+ \approx 1$ and $s^- \approx 0$), the proposed self-weighting mechanism effectively addresses this issue. Specifically, the weighting coefficients dynamically scale the gradient according to the deviation of the similarity scores from their targets. At $B_1(0.8,0.8)$, $s^-$ is far from its target while $s^+$ is relatively close to its target, so the loss prioritizes reducing $s^-$ by yielding a large $\left|\partial \mathcal{L}/\partial s^-\right|$ and a smaller $\left|\partial \mathcal{L}/\partial s^+\right|$. Conversely, at $B_2(0.6,0.5)$, $s^-$ is smaller, which results in a noticeably smaller $\left|\partial \mathcal{L}/\partial s^-\right|$ (and relatively more emphasis on improving $s^+$). This demonstrates that the proposed self-weighting mechanism improves optimization flexibility. Moreover, the new convergence boundary (a circle in our toy scenario) reduces the convergence ambiguity that arises under uniform weighting of positive and negative similarities.

\textbf{Convergence Analysis} We study the optimization dynamics directly in the similarity-score space $\{s_l^+,s_l^-\}_{l=1}^4$ for the toy setting with one anchor and one positive/negative sample per level. This allows us to rigorously characterize the effect of the proposed self-weighting mechanism on the geometry and dynamics of the loss.

\begin{proposition}[Quadratic form of the exponent argument]
\label{prop:quadratic_form}
By adopting a shifted cosine similarity ($(1+cos)/2$) and applying the reparameterization $\delta^+=1-m$, $\delta^-=m$, $o_l^+=1+m$, and $o_l^-=-m$ for all levels $l$, we can rewrite $\Delta''_l$ as follows:
\[
\Delta''_l
=
(s_l^-)^2 + (1-s_l^+)^2 - 2m^2.
\]
Consequently, the exponent argument of the toy version of Eq.~(\ref{eq.loss_lsw_ml_gssl}) becomes:
\[
\sum_{l=1}^{4} \Delta''_l
=
D - 8m^2,
\qquad
D = \sum_{l=1}^{4}\big[(s_l^-)^2 + (1-s_l^+)^2\big].
\]
\end{proposition}

It is important to reiterate that \(o_l^+\) and \(o_l^-\) are reference endpoints, while the feasible midpoint targets are \(\frac{o_l^+ + \delta^+}{2}=1\) and \(\frac{o_l^- + \delta^-}{2}=0\).

\begin{theorem}[Error contraction in the toy similarity space]
\label{thm:error_contraction}
In the same toy setting, we define the error vector: 
\[
\mathbf e =
[s_1^-,\dots,s_4^-,\,1-s_1^+,\dots,1-s_4^+]^\top.
\]
Let
\[
Z = D-8m^2,
\qquad
\mathcal{L} = \log\big(1+\exp(\gamma Z)\big).
\]
Then
\[
\nabla_{\mathbf e}\mathcal{L} = 2\,\gamma\,\sigma(\gamma \, Z)\,\mathbf e,
\quad\text{where}\quad
\sigma(t)=\frac{e^t}{1+e^t}\in(0,1).
\]
Under gradient descent with step size $0<\eta<\frac{1}{2\gamma}$, the error vector contracts multiplicatively:
\[
\mathbf e^{(t+1)} = \Big(1-2 \, \eta \, \gamma \, \sigma(\gamma \, Z^{(t)})\Big)\mathbf e^{(t)},
\]
and therefore:
\[
D^{(t+1)}=
\Big(1-2 \,\eta\,\gamma\,\sigma(\gamma \, Z^{(t)})\Big)^2\,D^{(t)}.
\]
Hence, the dynamics perform radial descent toward the ideal point
$s_l^- \to 0$ and $s_l^+ \to 1$ for all $l$.
\end{theorem}

We now show that the multiplicative contraction property in
Theorem~\ref{thm:error_contraction} is specific to the self-weighted construction
and does not hold for the linear multi-level combination in Eq.~(\ref{eq.loss_l_ml_gssl}). Consider the toy version of Eq.~(\ref{eq.loss_l_ml_gssl}), with fixed coefficients $\beta_l$: 
\begin{equation}
\mathcal{L}_{\mathrm{lin}}
=
\log\Big(1+\exp\big(\gamma S\big)\Big),
\label{eq:L_lin_toy_1_main}
\end{equation}
\begin{equation}
S = \sum_{l=1}^{4}\beta_l\big(s_l^- - s_l^+ + m\big).
\label{eq:L_lin_toy_2_main}
\end{equation}

\begin{proposition}[No radial descent / no multiplicative contraction for the linear combination]
\label{prop:no_radial_linear}
The gradient of $\mathcal{L}_{\mathrm{lin}}$ in $\mathbf{e}$-coordinates can be expressed as follows:
\[
\nabla_{\mathbf{e}} \mathcal{L}_{\mathrm{lin}}
=
\gamma\,\sigma(\gamma S)\,
[\beta_1,\dots,\beta_4,\beta_1,\dots,\beta_4]^\top,
\]
which is a fixed direction independent of $\mathbf{e}$ (up to the scalar factor $\sigma(\gamma S)$). Therefore, in general $\nabla_{\mathbf{e}} \mathcal{L}_{\mathrm{lin}}$
is not parallel to $\mathbf{e}$ and the update is not radial. In particular, there does not exist a scalar sequence $\{c_t\}_t$, such that $\mathbf{e}^{(t+1)}=c_t \, \mathbf{e}^{(t)}$ holds for all initial $\mathbf{e}^{(0)}$ under gradient descent on $\mathcal{L}_{\mathrm{lin}}$.
\end{proposition}

The proofs of Proposition 1, Theorem 1, and Proposition 2 are provided in Appendix \ref{appendix_a}.

\section{Experiments}
%%%%Single-Level experiments
We carry out extensive experiments to evaluate the effectiveness of our methods. At the single level, we compare our unified SL-GSSL approach with state-of-the-art methods that operate within the same abstraction level. Subsequently, we evaluate our multi-level approaches (i.e., L-ML-GSSL, LW-ML-GSSL, and LSW-ML-GSSL) against state-of-the-art multi-task GSSL methods. This systematic evaluation ensures a comprehensive understanding of how our method performs relative to the best available techniques. The code of the proposed approaches is available at this \textcolor{blue}{\href{https://github.com/M-M-Amar/LSW_ML_GSSL}{GitHub repository}}.% and will be made publicly accessible upon acceptance.

\textbf{Baselines.} We evaluate the performance of our single-level approaches by comparing them with thirteen state-of-the-art GSSL methods. These methods are selected to ensure comprehensive coverage of all abstraction levels. At the node level, we evaluate our approach (with $l = \text{node}$) against three prominent node-level GSSL methods: GRACE \cite{zhu2020deep},  BGRL \cite{thakoor2021large}, and GCMAE \cite{wang2024generative}. At the proximity level, we compare our method (with $l = \text{proximity}$) with Node2Vec \cite{grover2016node2vec}, GAE \cite{kipf2016variational}, VGAE \cite{kipf2016variational}, and NCLA \cite{shen2023neighbor}. At the cluster level, we assess our approach (with $l = \text{cluster}$) against GMM-VGAE \cite{hui2020collaborative}, DGAE \cite{mrabah2022rethinking}, and DGCLUSTER \cite{bhowmick2024dgcluster}. It is crucial to highlight that, at the cluster level, most deep graph clustering methods involve a pretraining self-supervised phase. Following this phase, pseudo-labels are generated and refined during a clustering phase. For instance, GMM-VGAE and DGAE employ proximity-level adjacency reconstruction as a pretext task during the pretraining phase. To ensure comparability with our approach, the proximity-level pretraining is excluded, and the two methods are denoted GMM-VGAE* and DGAE*. Additionally, we exclude the auxiliary term in DGCLUSTER's loss function that incorporates pairwise node information, and refer to this method as DGCLUSTER*. Finally, at the graph level, our method (with $l = \text{graph}$) is compared with DGI \cite{velickovic2019deep}, MVGRL \cite{hassani2020contrastive}, and LS-GCL \cite{yang2024local}. We evaluate our multi-level self-weighting framework (LSW-ML-GSSL) against five leading methods in multi-task graph learning, namely AutoSSL \cite{jin2021automated}, ParetoGNN \cite{ju2022multi}, DyFSS \cite{zhu2024every}, WAS \cite{fan2024decoupling}, and GraphTCM \cite{fang2024exploring}, together with four notable approaches for multi-scale graph learning: ANEMONE \cite{jin2021anemone}, MERIT \cite{ijcai2021p204}, MSSGCL \cite{ijcai2023p246}, and LMGTA \cite{li2025anomaly}. For each baseline, we use the official code shared by the authors and tune the hyperparameters in cases where explicit guidelines are not provided. To ensure a fair comparison, we run each method $10$ times and report the average performance along with its standard deviation.

\begin{table}[t]
\centering
\caption{Datasets statistics.}
\resizebox{\linewidth}{!}{
\begin{tabular}{ |c|c|c|c|c|c|c| } \hline 
\textbf{} & \textbf{Cora} & \textbf{CiteSeer} & \textbf{Pubmed} & \textbf{DBLP} & \textbf{Photo} & \textbf{Computers} \\ \hline 
\#Nodes            & 2,708 & 3,327 & 19,717 & 17,716 & 7,650 & 13,752 \\ \hline
\#Edges            & 10,556 & 9,104 & 88,648 & 105,734 & 238,162 & 491,722 \\ \hline
\#Features         & 1,433 & 3,703 & 500 & 1,639 &745 & 767\\ \hline
\#Classes          & 7 & 6 & 3  & 4 & 8 & 10\\ \hline 
\end{tabular}}
\label{tab:dataset_stats}
\end{table}

\textbf{Datasets, Downstream Tasks, and Evaluation Metrics.} We employed six diverse datasets to assess the performance of our approach: Cora \cite{mccallum2000automating}, CiteSeer \cite{giles1998citeseer}, Pubmed \cite{sen2008collective}, DBLP \cite{tang2008arnetminer}, Photo \cite{shchur2018pitfalls}, and Computers \cite{shchur2018pitfalls}. Table \ref{tab:dataset_stats} provides a summary of key dataset characteristics, with additional details outlined below.
\begin{itemize}
    \item \textbf{Cora}: The dataset is a citation network comprising 2,708 scientific publications (nodes), each classified into one of seven categories. It contains a total of 10,556 links. Each publication in the dataset is characterized by a binary vector (0/1), where each element signifies whether a specific word from a dictionary of 1,433 words (features) is absent or present.
    \item \textbf{CiteSeer}: The dataset includes 3,327 scientific publications divided into six categories, connected by a total of 9,104 links. Each publication is defined by a binary vector, marking whether each of 3,703 unique dictionary words is present or absent.
    \item \textbf{Pubmed}: This dataset includes 19,717 scientific articles from the PubMed database, all focused on diabetes, categorized into three distinct classes. The citation network contains 88,648 connections. Each article is represented by a TF/IDF (term frequency-inverse document frequency) weighted word vector derived from a dictionary containing 500 unique words. 
    \item \textbf{DBLP}: This is a citation network compiled from sources like DBLP, ACM, MAG (Microsoft Academic Graph), and others. The network comprises 17,716 articles connected by 105,734 links. Each article is represented by 1,639 features and categorized into one of 4 classes. 
    \item \textbf{Photo}: This dataset's nodes represent goods, while edges represent frequent co-purchases. The reviews are utilized to generate bag-of-words node features. The photo network comprises 7,650 nodes, 238,162 edges, and 745 features. It is partitioned into 8 distinct classes. 
    \item \textbf{Computers}: Similar to Photo, Computers is a co\-purchase graph derived from Amazon. It consists of 13,752 nodes, 491,722 edges, 767 features, and 10 classes. 
\end{itemize}

\begin{table}[t]
  \caption{Fixed hyperparameters for all datasets.}
  \begin{center}
  \begin{small}
  \begin{tabular}{|p{2.15cm}|c|p{2cm}|}
    \hline 
    Level & Parameter & Value \\ \hline
    \multirow{5}{*}{Augmentations} & Drop edge rate 1 & 0.2 \\
    \cline{2-3}
    & Drop edge rate 2 & 0.4 \\
    \cline{2-3}
    & Drop feature rate 1 & 0.3\\
    \cline{2-3}
    & Drop feature rate 2 & 0.4 \\
    \cline{2-3}
    & Dropping scheme & Uniform \\
    \hline
    \multirow{2}{*}{Architecture} & GNN encoder & 256-256\\ 
                                  \cline{2-3}
                                  & Projection head & 256-128 \\
    
                                   \hline
    \multirow{2}{*}{Training} & Optimizer & Adam \\ 
                                       \cline{2-3}
                                       & Learning rate & $0.001$ \\ 
                                       
                                        \hline
                                      
  \end{tabular}
  \end{small}
  \end{center}
  \label{Table:hyperparameters_indep}
\end{table}

Our experiments focus on three 
downstream tasks: node classification, node clustering, and link prediction.  
We assess the models using metrics specific to the downstream tasks. For node classification, we measure performance using Accuracy (ACC). Node clustering is evaluated using Normalized Mutual Information (NMI) and Adjusted Rand Index (ARI), while link prediction performance is evaluated using the Area Under the Receiver Operating Characteristic Curve (AUC-ROC).

\begin{table}[t]
\caption{Data-dependent hyperparameters.}
  \begin{center}
  %\begin{small}
  \resizebox{\linewidth}{!}{
  \begin{tabular}{|c|c|c|c|c|c|c|}
    \hline
    Parameter & Cora &  CiteSeer & Pubmed & DBLP & Photo & Computers \\ \hline
     $m$ & 0.15 & 0.15 & 0.10 & 0.20 & 0.10 & 0.10 \\ \hline
     $\gamma$ & 1.5 & 1.0 & 2.0 & 1.5 & 2.0 & 2.0 \\ \hline
     \# of epochs & 400 & 400 & 600 & 1,000 & 2,500 & 4,500 \\ \hline
  \end{tabular}
  }
%  \end{small}
  \end{center}
  
  \label{Table:hyperparameters_dep}
\end{table}

\begin{table*}[!t]
\centering
\caption{Performance and task generalization of our SL-GSSL methods and the state-of-the-art single-task GSSL approaches.}
\label{table_1}
\scalebox{0.68}{\begin{tabular}{l|l|ccccccc}
\toprule
\toprule
Abstraction level & Method & Cora & CiteSeer & Pubmed & DBLP & Photo & Computers  & Rank \\ 
\midrule
\rowcolor{gray!25} \multicolumn{9}{c}{\textsc{Average Performance}} \\ \hline

\multirow{3}{*}{\textsc{Node}} & \textsc{GRACE} \cite{zhu2020deep} & \underline{65.57} & 57.45 & 55.50 & 65.29 & \underline{68.81} & 61.27 & 2.7 \\
                               & \textsc{BGRL} \cite{thakoor2021large} & 58.21 & 48.90 & 54.53 & 46.32 & 66.18 & 60.40 & 3.7 \\
                               & \textsc{GCMAE} \cite{wang2024generative} & 53.78 & \underline{61.92} & \underline{57.12} & \underline{66.84} & 65.25 & \underline{65.50}& \underline{2.7}
                                    \\
                               & \textsc{SL-GSSL} & \textbf{71.02} & \textbf{64.51}& \textbf{60.78} & \textbf{69.40} & \textbf{70.79} &\textbf{71.69}  & \textbf{1.0} \\ \hline 

\multirow{3}{*}{\textsc{Proximity}} & \textsc{Node2Vec} \cite{grover2016node2vec} & 59.99 & 40.35 & 46.83 &50.90 & 68.95 & 62.30 & 4.5 \\
                                    & \textsc{GAE} \cite{kipf2016variational} & 63.68 & 54.76 & \underline{59.17} & 59.49 & 68.11 & 59.89 & 3.5 \\ 
                                    & \textsc{VGAE} \cite{kipf2016variational} & 64.20 & 53.92 & 57.59 & 58.85 & 69.26 & 59.43 & 3.7 \\ 
                                    & \textsc{NCLA} \cite{shen2023neighbor} & \underline{69.19} & \textbf{63.08}& 56.65 & \underline{63.20}& \underline{70.87}& \underline{65.84}& \underline{2.2}
                                    \\
                                    & \textsc{SL-GSSL} & \textbf{70.60}& \underline{62.79} & \textbf{63.81}& \textbf{69.09}&\textbf{71.89} & \textbf{66.36} & \textbf{1.2} \\ \hline

\multirow{3}{*}{\textsc{Cluster}} & \textsc{GMM-VGAE*} \cite{hui2020collaborative} & 19.41 & 17.61 & 28.44 & 26.39 & 19.51 &24.93  & 4.0 \\
                                  & \textsc{DGAE*} \cite{mrabah2022rethinking} & 26.37 & 23.21 & 33.70 & 34.09 & 33.81 & 33.28 & 3.0 \\
                                  & \textsc{DGCLUSTER*} \cite{bhowmick2024dgcluster} & \underline{60.37} & \underline{44.17} & \underline{47.96} & \underline{50.78}& \underline{55.25} & \underline{45.74} & \underline{2.0}
                                    \\
                                  & \textsc{SL-GSSL} &\textbf{66.60} & \textbf{59.62} & \textbf{56.83}& \textbf{63.17}& \textbf{65.70} & \textbf{51.96} & \textbf{1.0}\\ \hline

\multirow{3}{*}{\textsc{Graph}} & \textsc{DGI} \cite{velickovic2019deep} & 64.77  &49.43  & 57.30 & \underline{60.50} & 53.42 &55.47 & 3.0 \\
                                & \textsc{MVGRL}  \cite{hassani2020contrastive} & \underline{66.60} & \underline{60.73} & \underline{60.35} & 56.44 & \underline{61.26} & \underline{57.12} & \underline{2.3} \\
                                & \textsc{LS-GCL} \cite{yang2024local} & 66.01 &48.81 & 54.84& 58.27 & 51.81 & 54.87 & 3.7
                                    \\
                                & \textsc{SL-GSSL} & \textbf{71.46}& \textbf{62.50}  & \textbf{60.46}& \textbf{66.53}& \textbf{63.70}& \textbf{59.12} &\textbf{1.0}\\ \hline

\rowcolor{gray!25} \multicolumn{9}{c}{\textsc{Node Classification} (Accuracy)} \\ \hline

\multirow{3}{*}{\textsc{Node}} & \textsc{GRACE} \cite{zhu2020deep} & $\underline{83.88 \pm \scriptstyle 0.63}$ & $68.13 \pm \scriptstyle 0.90$ & $\underline{85.14 \pm \scriptstyle 0.24}$ & $\underline{84.33 \pm \scriptstyle 0.02}$ & $\underline{92.29 \pm \scriptstyle 0.17}$ & $86.20 \pm \scriptstyle 0.46$ & \underline{2.3}\\
                               & \textsc{BGRL} \cite{thakoor2021large} & $82.67\pm \scriptstyle 0.16$ & $66.92 \pm \scriptstyle 0.58$ & $\mathbf{85.35 \pm \scriptstyle 0.20}$ & $78.43 \pm \scriptstyle 0.33$ & $92.22 \pm \scriptstyle 0.31$ & $\mathbf{87.69 \pm \scriptstyle 0.13}$ & 2.7\\ 
                               & \textsc{GCMAE} \cite{wang2024generative} & $79.92 \pm \scriptstyle 0.44$& $ \underline{69.29 \pm \scriptstyle 0.20}$ & $84.08 \pm \scriptstyle 0.12$ & $84.25 \pm \scriptstyle 0.07$ &  $90.11 \pm \scriptstyle 0.49$& $86.14 \pm \scriptstyle 0.29$ & 3.5
                                    \\
                               & \textsc{SL-GSSL} & $\mathbf{84.08 \pm \scriptstyle 0.45}$& $\mathbf{71.02 \pm \scriptstyle 0.65}$ & $84.14 \pm \scriptstyle 0.18$ & $\mathbf{84.42 \pm \scriptstyle 0.23}$ & $\mathbf{92.78 \pm \scriptstyle 0.31}$& $\underline{86.22 \pm \scriptstyle 0.81}$ & \textbf{1.5} \\ \hline 

\multirow{3}{*}{\textsc{Proximity}} & \textsc{Node2Vec}  \cite{grover2016node2vec} & $75.62 \pm \scriptstyle 1.77$ & $49.68 \pm \scriptstyle 1.39$ & $79.42 \pm \scriptstyle 0.36$ & $ 78.63 \pm \scriptstyle 0.29$ & $89.83 \pm \scriptstyle 0.48$ & $85.64 \pm \scriptstyle 0.35$ & 4.5 \\
                                    & \textsc{GAE} \cite{kipf2016variational} & $80.46 \pm \scriptstyle 0.94$ & $68.45 \pm \scriptstyle 0.78$ & $82.57 \pm \scriptstyle 0.17$ & $81.59 \pm \scriptstyle 0.19 $ & $ \underline{91.51 \pm \scriptstyle 0.39}$ & $80.25 \pm \scriptstyle 0.42$ & 3.2 \\ 
                                    & \textsc{VGAE} \cite{kipf2016variational} & $\underline{80.62 \pm \scriptstyle 0.32}$ & $66.59 \pm \scriptstyle 1.04$ & $82.31 \pm \scriptstyle 0.35$ & $81.89 \pm \scriptstyle 0.18$ & $89.76 \pm \scriptstyle 0.50$ & $79.28 \pm \scriptstyle 0.66$ & 3.8 \\ 
                                    %& \textsc{NCLA} \cite{shen2023neighbor} &80.18 \pm \scriptstyle 0.80 &$\underline{72.34 \pm \scriptstyle 0.80}$ &$\underline{82.65 \pm \scriptstyle 0.17}$ & $\underline{82.89 \pm \scriptstyle 0.11}$& $90.91 \pm \scriptstyle 0.33$&$\underline{85.67 \pm \scriptstyle 0.40}$ & \underline{2.5}
                                    & \textsc{NCLA} \cite{shen2023neighbor} & $80.18 \pm \scriptstyle 0.80$ & $\underline{72.34 \pm \scriptstyle 0.80}$ & $\underline{82.65 \pm \scriptstyle 0.17}$ & $\underline{82.89 \pm \scriptstyle 0.11}$ & $90.91 \pm \scriptstyle 0.33$ & $\underline{85.67 \pm \scriptstyle 0.40}$ & \underline{2.5} \\
                                    & \textsc{SL-GSSL} & $\mathbf{82.32 \pm \scriptstyle 0.58}$ & $\mathbf{71.34 \pm \scriptstyle 1.07}$ & $\mathbf{83.23 \pm \scriptstyle 0.07}$& $\mathbf{82.33 \pm \scriptstyle 0.14}$& $\mathbf{92.10 \pm \scriptstyle 0.29}$ & $\mathbf{85.77 \pm \scriptstyle 0.56}$ & \textbf{1.0} \\ \hline

\multirow{3}{*}{\textsc{Cluster}} & \textsc{GMM-VGAE*} \cite{hui2020collaborative} & $22.09 \pm \scriptstyle 1.44$ & $19.34 \pm \scriptstyle 0.73$ & $39.68 \pm \scriptstyle 0.33$ & $43.67 \pm \scriptstyle 0.35$ & $24.89 \pm \scriptstyle 1.24$ & $37.39
\pm \scriptstyle 0.17$ & 4.0 \\
                                  & \textsc{DGAE*} \cite{mrabah2022rethinking} & $30.19 \pm \scriptstyle 0.27$ & $20.27 \pm \scriptstyle 1.17$ & $40.08 \pm \scriptstyle 1.54$ & $44.52 \pm \scriptstyle 0.34$ & $25.01 \pm \scriptstyle 1.10$ & $37.58 \pm \scriptstyle 0.16$ & 3.0 \\
                                  & \textsc{DGCLUSTER*} \cite{bhowmick2024dgcluster} &$\underline{77.95 \pm \scriptstyle 1.20}$ &$\underline{55.71 \pm \scriptstyle 1.54}$ & $\underline{61.13 \pm \scriptstyle 0.29}$&$\underline{75.49 \pm \scriptstyle 0.29}$ &$\underline{74.26 \pm \scriptstyle 0.60}$ & $\underline{75.23 \pm \scriptstyle 1.68}$&\underline{2.0}
                                    \\
                                  & \textsc{SL-GSSL} & $\mathbf{82.23 \pm \scriptstyle 1.82}$ & $\mathbf{68.94 \pm \scriptstyle 0.36}$ & $\mathbf{75.87 \pm \scriptstyle 0.43}$&$\mathbf{80.70 \pm \scriptstyle 0.12}$ & $\mathbf{91.51 \pm \scriptstyle 0.24}$ & $\mathbf{84.76 \pm \scriptstyle 0.60}$ &\textbf{1.0}\\ \hline

\multirow{3}{*}{\textsc{Graph}} & \textsc{DGI} \cite{velickovic2019deep} & $81.76 \pm \scriptstyle 1.27$ & $68.93 \pm \scriptstyle 0.28$ & $\underline{84.38 \pm \scriptstyle 0.24}$  & $82. 31 \pm \scriptstyle 0.21$ & $90.19 \pm \scriptstyle 0.54$ & $83.68 \pm \scriptstyle 0.54$ & 3.3 \\
                                & \textsc{MVGRL}  \cite{hassani2020contrastive} & $\underline{84.02 \pm \scriptstyle 0.68}$ & $\underline{71.65 \pm \scriptstyle 0.55}$ & $\mathbf{85.01 \pm \scriptstyle 0.27}$ & $\mathbf{84.28 \pm \scriptstyle 0.24}$ & $\mathbf{91.34 \pm \scriptstyle 0.57}$ & $\underline{85.53 \pm \scriptstyle 0.40}$ & \textbf{1.5} \\
                                %& \textsc{LS-GCL} \cite{yang2024local} &$83.22 \pm \scriptstyle 0.75$ &$65.60 \pm \scriptstyle 1.11$ &$81.82 \pm \scriptstyle 0.30$ & $83.26 \pm \scriptstyle 0.20$&$84.86 \pm \scriptstyle 0.82$ &83.93 \pm \scriptstyle 0.99 &3.3 \\
                                & \textsc{LS-GCL} \cite{yang2024local} & $83.22 \pm \scriptstyle 0.75$ & $65.60 \pm \scriptstyle 1.11$ & $81.82 \pm \scriptstyle 0.30$ & $83.26 \pm \scriptstyle 0.20$ & $84.86 \pm \scriptstyle 0.82$ & $83.93 \pm \scriptstyle 0.99$ & 3.3 \\
                                & \textsc{SL-GSSL} & $\mathbf{84.35 \pm \scriptstyle 0.48}$& $\mathbf{71.71 \pm \scriptstyle 0.77}$ & $81.28 \pm \scriptstyle 0.31$ & $\underline{83.29 \pm \scriptstyle 0.31}$&$\underline{90.91 \pm \scriptstyle 0.33}$ & $\mathbf{85.74 \pm \scriptstyle 0.54}$ & \underline{1.8}\\ \hline

\rowcolor{gray!25} \multicolumn{9}{c}{\textsc{Node Clustering} (NMI/ARI)} \\ \hline

\multirow{3}{*}{\textsc{Node}} & \textsc{GRACE} \cite{zhu2020deep} & $\underline{51.96/42.61}$ & $34.97/33.63$ & $25.64/23.78$ & $46.80/42.32$ & $\underline{51.98/39.18}$ & $45.82/24.50$ & \underline{2.7} \\
                               & \textsc{BGRL} \cite{thakoor2021large} & $40.61/25.71$ & $20.26/17.36$ & $20.36/19.11$ & $10.12/ 09.51$ & $51.09/36.44$ & $41.58/27.91$ & 3.7 \\
                               & \textsc{GCMAE} \cite{wang2024generative} & $31.11/17.83$ &$\underline{42.64/43.22}$ & $\underline{27.54/24.32}$ & $\underline{48.25/45.73}$&$49.21/34.56$ &$\underline{51.12/32.86}$ & 2.7
                                    \\
                               & \textsc{SL-GSSL} &$\mathbf{57.50/52.02}$ & $\mathbf{45.64/47.05}$ & $\mathbf{34.73/31.04}$ &$\mathbf{51.27/48.56}$ & $\mathbf{54.87/40.39}$ & $\mathbf{54.23/52.53}$ & \textbf{1.0}\\ \hline 

\multirow{3}{*}{\textsc{Proximity}} & \textsc{Node2Vec} \cite{grover2016node2vec} & $44.51/32.92$ & $24.17/07.87$ & $19.67/07.19$ & $21.48/11.58$ & $56.03/38.30$ & $44.90/26.94$ & 4.3\\
                                    & \textsc{GAE} \cite{kipf2016variational} & $46.10/37.94$ & $32.41/29.60$ & $\underline{29.42/29.10}$ & $32.31/29.50$ & $54.61/34.28$ & $43.45/24.58$ & 3.3 \\ 
                                    & \textsc{VGAE} \cite{kipf2016variational} & $46.19/37.13$ & $31.27/28.97$ & $26.73/24.99$ & $31.65/26.04$ & $52.67/39.62$ & $41.28/23.06$ & 4.2\\ 
                                    & \textsc{NCLA} \cite{shen2023neighbor} &$\underline{54.62/50.38}$ & $\underline{43.94/45.15}$ & $24.89/23.67$ & $\underline{36.34/38.69}$ & $\underline{57.04/42.48}$ & $ \underline{50.68/35.60}$ & \underline{2.3} \\
                                    & \textsc{SL-GSSL} & $\mathbf{56.75/51.66}$& $\mathbf{43.58/44.50}$  & $\mathbf{35.23/42.25}$ & $\mathbf{50.34/49.87}$ & $\mathbf{59.15/43.08}$ & $\mathbf{51.46/49.09}$ & \textbf{1.0}\\ \hline

\multirow{3}{*}{\textsc{Cluster}} & \textsc{GMM-VGAE*} \cite{hui2020collaborative} & $02.68/02.10$ & $01.64/01.01$ & $12.74/11.84$ & $09.06/02.83$ & $02.13/01.07$ & $06.67/05.72$ & 3.7\\
                                  & \textsc{DGAE*} \cite{mrabah2022rethinking} & $03.53/02.40$ & $05.38/04.80$ & $02.35/02.97$ & $06.35/07.44$ & $18.45/10.53$ & $08.23/03.81$ & 3.3 \\
                                  & \textsc{DGCLUSTER*} \cite{bhowmick2024dgcluster} &$\underline{50.10/29.40}$ & $\underline{24.35/10.94}$& $\underline{26.08/15.67}$&$\underline{32.21/14.24}$ &$\underline{34.84/21.29}$ & $\underline{17.21/06.01}$& \underline{2.0}
                                    \\
                                  & \textsc{SL-GSSL} & $\mathbf{51.23/43.77}$ & $\mathbf{40.03/39.37}$  &$\mathbf{31.05/28.70}$ & $\mathbf{39.97/43.43}$& $\mathbf{47.15/32.51}$ & $\mathbf{22.76/14.00}$  &\textbf{1.0}\\ \hline

\multirow{3}{*}{\textsc{Graph}} & \textsc{DGI} \cite{velickovic2019deep} & $53.09/42.84$ & $23.17/18.00$ & $24.92/25.63$ & $\underline{35.25/35.19}$ & $12.10/24.30$ & $32.32/19.77$ & 3.2 \\
                                & \textsc{MVGRL} \cite{hassani2020contrastive} & $\underline{52.43/47.35}$ & $\underline{42.03/42.69}$ & $\underline{31.69/32.39}$ & $29.51/27.53$ & $\underline{40.62/24.78}$ & $\underline{32.97/22.02}$ & \underline{2.3} \\
                                & \textsc{LS-GCL} \cite{yang2024local} &$53.66/41.71$ & $25.27/16.07$& $25.06/23.38$&$32.77/28.61$ &$16.37/23.84$ & $31.52/17.23$& 3.5
                                    \\
                                & \textsc{SL-GSSL} & $\mathbf{57.10/51.11}$ & $\mathbf{43.00/43.17}$ & $\mathbf{34.18/33.62}$ &$\mathbf{46.47/44.79}$ &$\mathbf{45.42/29.22}$ & $\mathbf{35.12/25.31}$  & \textbf{1.0}\\ \hline

\rowcolor{gray!25} \multicolumn{9}{c}{\textsc{Link Prediction} (ROC-AUC)} \\ \hline

\multirow{3}{*}{\textsc{Node}} & \textsc{GRACE}  \cite{zhu2020deep} & $83.86 \pm \scriptstyle 0.11 $ & $\underline{93.09 \pm \scriptstyle 0.73}$ & $87.44 \pm \scriptstyle 1.03$ & $87.72 \pm \scriptstyle 0.04$ & $\underline{91.82 \pm \scriptstyle 0.12}$ & $88.58 \pm \scriptstyle 0.08$ & 3.0 \\
                               & \textsc{BGRL} \cite{thakoor2021large} & $\underline{87.82 \pm \scriptstyle 1.77}$ & $91.06 \pm \scriptstyle 0.78$ & $\underline{93.32 \pm \scriptstyle 0.35}$ & $87.31 \pm \scriptstyle 0.35$ & $85.24 \pm \scriptstyle 0.48$ & $84.42 \pm \scriptstyle 0.33$ & 3.3 \\
                               & \textsc{GCMAE} \cite{wang2024generative} & $86.29 \pm \scriptstyle 0.20$&$92.53 \pm \scriptstyle 0.52$ &$92.55 \pm \scriptstyle 0.37$ & $\underline{89.15 \pm \scriptstyle 0.21}$ &$87.12 \pm \scriptstyle 0.35$ &$\underline{91.89 \pm \scriptstyle 0.07}$ & \underline{2.7}
                                    \\
                               & \textsc{SL-GSSL} & $\mathbf{90.50 \pm \scriptstyle 1.42}$ & $\mathbf{94.34 \pm \scriptstyle 0.84}$ &$\mathbf{93.21 \pm \scriptstyle 0.29}$ &$\mathbf{93.18 \pm \scriptstyle 0.37}$ & $\mathbf{94.53 \pm \scriptstyle 0.34}$ & $\mathbf{93.80 \pm 0.18}$ & \textbf{1.0}\\ \hline 

\multirow{3}{*}{\textsc{Proximity}} & \textsc{Node2Vec} \cite{grover2016node2vec}& $86.94 \pm \scriptstyle 1.17$ & $79.71 \pm \scriptstyle 1.57$ & $81.06 \pm \scriptstyle 0.51$ & $ 91.93 \pm \scriptstyle 0.17$ & $91.66 \pm \scriptstyle 0.08$ & $91.74 \pm \scriptstyle 0.09$ & 4.7 \\
                                    & \textsc{GAE} \cite{kipf2016variational} & $90.24 \pm \scriptstyle 0.74 $ & $88.59 \pm \scriptstyle 0.81$ & $\underline{95.60 \pm \scriptstyle 0.17}$ & $94.57 \pm \scriptstyle 0.43$ & $92.06 \pm \scriptstyle 0.22$ & $91.29 \pm \scriptstyle 0.30$ & 3.5 \\ 
                                    & \textsc{VGAE} \cite{kipf2016variational} & $89.88 \pm \scriptstyle 1.21$ & $88.85 \pm \scriptstyle 0.37$ & $ \mathbf{96.33 \pm \scriptstyle 0.39}$ & $\mathbf{95.84 \pm \scriptstyle 0.32}$ & $\mathbf{95.02 \pm \scriptstyle 0.57}$ & $\mathbf{94.09 \pm \scriptstyle 0.10}$ & \textbf{1.8} \\ 
                                    %& \textsc{NCLA} \cite{shen2023neighbor} &\underline{91.59 \pm \scriptstyle 0.30} & \underline{90.89 \pm \scriptstyle 0.25}& $95.42 \pm \scriptstyle 0.29$& $\underline{94.89 \pm \scriptstyle 0.12}$&$93.05 \pm \scriptstyle 0.52$ & $91.42 \pm \scriptstyle 0.21$& 2.7
                                    %\\
                                    & \textsc{NCLA} \cite{shen2023neighbor} & $\underline{91.59 \pm \scriptstyle 0.30}$ & $\underline{90.89 \pm \scriptstyle 0.25}$ & $95.42 \pm \scriptstyle 0.29$ & $\underline{94.89 \pm \scriptstyle 0.12}$ & $93.05 \pm \scriptstyle 0.52$ & $91.42 \pm \scriptstyle 0.21$ & 2.7 \\
                                    & \textsc{SL-GSSL} & $\mathbf{92.07 \pm \scriptstyle 0.15}$& $\mathbf{91.77 \pm \scriptstyle 0.24}$ &$94.54 \pm \scriptstyle 0.19$ & $93.83 \pm \scriptstyle 0.21$ & $\underline{93.25 \pm \scriptstyle 0.47}$& $\underline{93.98 \pm \scriptstyle 0.25}$ & \underline{2.3}\\ \hline

\multirow{3}{*}{\textsc{Cluster}} & \textsc{GMM-VGAE*} \cite{hui2020collaborative} & $50.78 \pm \scriptstyle 2.23$ & $48.47 \pm \scriptstyle 1.04$ & $49.52 \pm \scriptstyle 0.07$ & $50.02 \pm \scriptstyle 0.44$ & $49.97 \pm \scriptstyle 0.28$ & $49.94 \pm \scriptstyle 0.41$ & 4.0 \\
                                  & \textsc{DGAE*} \cite{mrabah2022rethinking} & $69.37 \pm \scriptstyle 0.95$ & $62.42 \pm \scriptstyle 0.04$ & $\underline{89.25 \pm \scriptstyle 0.01}$ & $78.06 \pm \scriptstyle 0.01$ & $81.27 \pm \scriptstyle 0.01$ & $83.50 \pm \scriptstyle 0.01$ & 2.8 \\
                                  & \textsc{DGCLUSTER*} \cite{bhowmick2024dgcluster} &$\underline{84.04 \pm \scriptstyle 0.89}$ &$\underline{85.70 \pm \scriptstyle 0.52}$ &$88.97 \pm \scriptstyle 0.17$ & $\underline{81.21 \pm \scriptstyle 0.34}$&$\underline{90.64 \pm \scriptstyle 0.12}$ &$\underline{84.53 \pm \scriptstyle 0.47}$ & \underline{2.2}
                                    \\
                                  & \textsc{SL-GSSL} &$\mathbf{89.17 \pm \scriptstyle 0.37}$ & $\mathbf{90.15 \pm \scriptstyle 0.52}$ & $\mathbf{91.73 \pm \scriptstyle 0.21}$ & $\mathbf{88.59 \pm \scriptstyle 0.83}$& $\mathbf{91.64 \pm \scriptstyle 0.29}$& $\mathbf{86.34 \pm \scriptstyle 0.38}$ &\textbf{1.0} \\ \hline

\multirow{3}{*}{\textsc{Graph}} & \textsc{DGI} \cite{velickovic2019deep} & $81.41 \pm \scriptstyle 0.14$ & $87.64 \pm \scriptstyle 0.06$ & $\mathbf{94.28 \pm \scriptstyle 0.04}$ & $\underline{89.28 \pm \scriptstyle 0.36}$ & $87.09 \pm \scriptstyle 0.26$ & $86.12 \pm \scriptstyle 0.45$ & \underline{2.8} \\
                                & \textsc{MVGRL} \cite{hassani2020contrastive} & $82.62 \pm \scriptstyle 1.15$ & $86.55 \pm \scriptstyle 1.05$ & $92.32 \pm \scriptstyle 1.36$ & $84.46 \pm \scriptstyle 1.90$ & $\underline{88.32 \pm \scriptstyle 0.02}$ & $\underline{87.99 \pm \scriptstyle 1.12}$ & 3.0 \\
                                & \textsc{LS-GCL} \cite{yang2024local}&$\underline{85.45 \pm \scriptstyle 0.13}$ & $\underline{88.32 \pm \scriptstyle 0.47}$& $89.10 \pm \scriptstyle 0.75$&$88.46 \pm \scriptstyle 0.47$ & $82.19 \pm \scriptstyle 0.17$&$86.82 \pm \scriptstyle 0.84$ & 3.0
                                    \\
                                & \textsc{SL-GSSL} &$\mathbf{93.31 \pm \scriptstyle 0.21}$ & $\mathbf{92.15 \pm \scriptstyle 0.84}$ &$\underline{92.77 \pm \scriptstyle 0.54}$ & $\mathbf{91.57 \pm \scriptstyle 0.18}$&$\mathbf{89.25 \pm \scriptstyle 0.45}$ & $\mathbf{90.34 \pm \scriptstyle 0.78}$ & \textbf{1.2} \\ \hline
\bottomrule
\end{tabular}}
\end{table*}

\begin{table*}
\centering
\caption{Ablation study of the single-level unified objective in Eq.~(\ref{eq.loss_unified}). 
\textsc{Best-SL-GSSL} denotes the best performing single-level SL-GSSL variant across the abstraction levels. For each ablation setting, we also report the best result attained across the abstraction levels.  
\textsc{Best-SL-GSSL-No-Margin} removes the margin by setting \(m=0\). 
\textsc{Best-SL-GSSL-Hinge-Loss} replaces the exponential transformation with the hinge form \(\max\!\left(0,\, s^- - s^+ + m\right)\). 
\textsc{Best-SL-GSSL-InfoNCE} uses the standard Information Noise Contrastive Estimation loss defined over \(s^+\) and \(s^-\).  
}
\label{sl_ablation}
\scalebox{0.68}{\begin{tabular}{l|cccccc}
\toprule
\toprule
 Method & Cora & CiteSeer & Pubmed & DBLP & Photo & Computers  \\ 
\midrule

\rowcolor{gray!25} \multicolumn{7}{c}{\textsc{Node Classification} (Accuracy)} \\ \hline
             \textsc{Best-SL-GSSL-InfoNCE} & $83.88 \pm \scriptstyle 0.63$ & $68.13 \pm \scriptstyle 0.90$ & $\underline{85.14 \pm \scriptstyle 0.24}$ & $\underline{84.33 \pm \scriptstyle 0.02}$ & $92.29 \pm \scriptstyle 0.17$ & $\underline{86.20 \pm \scriptstyle 0.46}$ \\
            
            \textsc{Best-SL-GSSL-No-Margin} & $\underline{84.00 \pm \scriptstyle 0.24}$ & $\underline{69.62 \pm \scriptstyle 0.59}$ & $83.93 \pm \scriptstyle 0.13$ & $82.04 \pm \scriptstyle 0.25$ & $92.38 \pm \scriptstyle 0.41$ & $86.00 \pm \scriptstyle 0.17$ \\
            \textsc{Best-SL-GSSL-Hinge-Loss} & $82.59 \pm \scriptstyle 0.52$ & $69.22 \pm \scriptstyle 0.73$ & $83.91 \pm \scriptstyle 0.10$ & $81.55 \pm \scriptstyle 0.20$ & $\underline{92.60 \pm \scriptstyle 0.23}$ & $86.00 \pm \scriptstyle 0.25$ \\
            \textsc{Best-SL-GSSL} &$\mathbf{84.35 \pm \scriptstyle 0.48}$ &$\mathbf{71.71 \pm \scriptstyle 0.77}$ & $\mathbf{85.35 \pm \scriptstyle 0.20}$&$\mathbf{84.42 \pm \scriptstyle 0.23}$ &$\mathbf{92.78 \pm \scriptstyle 0.31}$ & $\mathbf{87.69 \pm \scriptstyle 0.13}$\\\hline

\rowcolor{gray!25} \multicolumn{7}{c}{\textsc{Node Clustering} (NMI/ARI)} \\ \hline      
            \textsc{Best-SL-GSSL-InfoNCE}& $51.96/42.61$ & $34.97/33.63$ & $25.64/23.78$ & $46.80/42.32$ & $51.98/39.18$ & $ \underline{45.82/24.50}$ \\
            \textsc{Best-SL-GSSL-No-Margin} & $55.28/50.65$ & $\underline{42.88/44.24}$ & $34.00/34.93$ & $\underline{47.34/52.06}$ & $\underline{57.90/37.06}$ & $44.06/31.11$ \\
            \textsc{Best-SL-GSSL-Hinge-Loss} & $\underline{56.54/51.26}$ & $40.49/40.94$ & $\underline{34.84/35.00}$ & $41.69/47.28$ & $54.75/33.09$ & $41.51/28.47$ \\
            \textsc{Best-SL-GSSL} &$\mathbf{57.50/52.02}$ & $\mathbf{45.64/47.05}$ & $\mathbf{35.23/42.25}$ & $\mathbf{51.27/48.56}$ & $\mathbf{59.15/43.08}$ & $\mathbf{54.23/52.53}$\\ \hline

\rowcolor{gray!25} \multicolumn{7}{c}{\textsc{Link Prediction} (ROC-AUC)} \\ \hline
            \textsc{Best-SL-GSSL-InfoNCE} & $83.86 \pm \scriptstyle 0.11 $ & $93.09 \pm \scriptstyle 0.73$ & $87.44 \pm \scriptstyle 1.03$ & $87.72 \pm \scriptstyle 0.04$ & $91.82 \pm \scriptstyle 0.12 $ & $88.58 \pm \scriptstyle 0.08$ \\
            \textsc{Best-SL-GSSL-No-Margin} & $\underline{91.95 \pm \scriptstyle 0.23}$ & $\underline{93.77 \pm \scriptstyle 0.59}$ & $\underline{95.95 \pm \scriptstyle 0.28}$ & $\underline{95.10 \pm \scriptstyle 0.34}$ & $\underline{95.07 \pm \scriptstyle 0.19}$ & $\underline{93.36 \pm \scriptstyle 0.17}$ \\
            \textsc{Best-SL-GSSL-Hinge-Loss} & $90.06 \pm \scriptstyle 0.17$ & $90.15 \pm \scriptstyle 0.57$ & $94.32 \pm \scriptstyle 0.29$ & $91.55 \pm \scriptstyle 0.23$ & $94.98 \pm \scriptstyle 0.20$ & $92.78 \pm \scriptstyle 0.36$ \\
            \textsc{Best-SL-GSSL} & $\mathbf{93.31 \pm \scriptstyle 0.21}$ & $\mathbf{94.34 \pm \scriptstyle 0.84}$ & $\mathbf{96.33 \pm \scriptstyle 0.39}$ & $\mathbf{95.84 \pm \scriptstyle 0.32}$ & $\mathbf{95.13 \pm \scriptstyle 0.34}$ & $\mathbf{94.09 \pm \scriptstyle 0.10}$  \\ \hline

\bottomrule
\end{tabular}}
\end{table*}

%%%%Multi-Level results
\begin{table*}
\centering
\caption{Performance and task generalization of our LSW-ML-GSSL method and the state-of-the-art multi-task and multi-scale GSSL approaches.}
\label{table_2}
\scalebox{0.68}{\begin{tabular}{l|l|ccccccc}
\toprule
\toprule
  & Method & Cora & CiteSeer & Pubmed & DBLP & Photo & Computers  & Rank \\ 
\midrule

\rowcolor{gray!25} \multicolumn{9}{c}{\textsc{Node Classification} (Accuracy)} \\ \hline

\multirow{5}{*}{\textsc{Multi-Task Learning }} & \textsc{AutoSSL} \cite{jin2021automated}  &$81.75 \pm \scriptstyle 0.25$ & $71.00 \pm \scriptstyle 0.27$& $78.95 \pm \scriptstyle 0.10$& $82.89 \pm \scriptstyle 0.14$&$92.14 \pm \scriptstyle 0.06$ &$83.17 \pm \scriptstyle 0.03$ &$5.8$\\
                                &\textsc{ParetoGNN} \cite{ju2022multi}  & $80.14 \pm \scriptstyle 0.68$&$62.93 \pm \scriptstyle 1.59$ &$84.58 \pm \scriptstyle 0.21$ &$\underline{84.32 \pm \scriptstyle 0.12}$ & $\underline{92.38 \pm \scriptstyle 0.20}$& $\mathbf{87.16 \pm \scriptstyle 0.11}$& $\underline{4.2}$ \\ 
                                &\textsc{DyFSS} \cite{zhu2024every}  & $\underline{ 83.21 \pm \scriptstyle 0.32}$ & $\underline{72.82 \pm \scriptstyle 0.23}$ & $\underline{84.77 \pm \scriptstyle0.35} $& $84.13 \pm \scriptstyle0.05 $ & $90.88 \pm \scriptstyle0.08 $ & $79.28 \pm \scriptstyle0.08 $ & $4.2$  \\
                                &\textsc{WAS} \cite{fan2024decoupling}  & $81.76 \pm \scriptstyle 1.17$ & $71.37 \pm \scriptstyle 1.37$ & $79.60 \pm \scriptstyle 0.43$& $83.15 \pm \scriptstyle0.95 $ & $91.72 \pm \scriptstyle0.20 $ & $85.20 \pm \scriptstyle 0.53$ & $4.8$ \\
                                &\textsc{GraphTCM} \cite{fang2024exploring}  & $81.68 \pm \scriptstyle 0.25$ & $72.46 \pm \scriptstyle0.10 $ & $77.10 \pm \scriptstyle0.70 $& $83.52 \pm \scriptstyle 0.24$ & $92.06 \pm \scriptstyle 0.19$ & $84.95 \pm \scriptstyle0.34 $ & $5.2$ \\ \hline
\multirow{5}{*}{\textsc{Multi-Scale Learning }}
                                 & \textsc{ANEMONE} \cite{jin2021anemone}
                                 & $79.78 \pm \scriptstyle 0.42$
                                 & $64.79 \pm \scriptstyle 0.59$
                                 & $83.83 \pm \scriptstyle 0.65$
                                 & $81.82 \pm \scriptstyle 0.17$
                                 & $90.69 \pm \scriptstyle 0.09$
                                 & $86.15 \pm \scriptstyle 0.23$
                                 & $7.2$ \\
                                
                                 & \textsc{MERIT} \cite{ijcai2021p204}
                                 & $81. 58 \pm \scriptstyle 0.86$
                                 & $69.07 \pm \scriptstyle 0.14$
                                 & $83.24 \pm \scriptstyle 0.29$
                                 & $82.97 \pm \scriptstyle 0.33$
                                 & $90.75 \pm \scriptstyle 0.18$
                                 & $83.86 \pm \scriptstyle 0.16$
                                 & $6.2$ \\
                                
                                 & \textsc{MSSGCL} \cite{ijcai2023p246}
                                 & $80.05 \pm \scriptstyle 0.61$
                                 & $67.78 \pm \scriptstyle 0.25$
                                 & $81.82 \pm \scriptstyle 0.39$
                                 & $80.86 \pm \scriptstyle 0.41$
                                 & $89.32 \pm \scriptstyle 0.26$
                                 & $80.45 \pm \scriptstyle 0.26$
                                 & $8.3$ \\
                                
                                 & \textsc{LMGTA} \cite{li2025anomaly}
                                 & $78.31 \pm \scriptstyle 0.52$
                                 & $69.85 \pm \scriptstyle 0.37$
                                 & $81.77 \pm \scriptstyle 0.31$
                                 & $82.08 \pm \scriptstyle 0.22$
                                 & $89.06 \pm \scriptstyle 0.14$
                                 & $80.53 \pm \scriptstyle 0.21$
                                 & $8.3$ \\
                                & \textsc{LSW-ML-GSSL} &$\mathbf{84.87 \pm \scriptstyle 0.75}$ & $\mathbf{73.43 \pm \scriptstyle 0.45}$&$\mathbf{85.46 \pm \scriptstyle 0.13}$ &$\mathbf{84.47 \pm \scriptstyle 0.21}$ & $\mathbf{92.80 \pm \scriptstyle 0.45}$& $\underline{86.75 \pm \scriptstyle 0.35}$&$\mathbf{1.2}$ \\ \hline

\rowcolor{gray!25} \multicolumn{9}{c}{\textsc{Node Clustering} (NMI/ARI)} \\ \hline

\multirow{5}{*}{\textsc{Multi-Task Learning}} &  \textsc{AutoSSL} \cite{jin2021automated}  &$48.28/37.58$ & $43.75/44.25$& $27.93/25.48$ & $48.23/45.12$&$20.88/12.34$ &$26.62/10.75$ & $6.8$ \\
                               
                                &\textsc{ParetoGNN} \cite{ju2022multi}  &$40.15/31.08$ &$14.85/13.25$ &$29.58/35.63$ & $50.32/46.23$& $55.33/44.76$& $\underline{51.75/33.99}$& $5.7$ \\ 
                                &\textsc{DyFSS} \cite{zhu2024every}  &  $ \underline{54.17/49.19}$  & $43.60/44.60 $ & $ 33.22/40.56 $& $ 50.67/48.02 $ & $ \underline{57.40/42.14} $ & $ 45.46/32.00 $ & $\underline{3.2}$ \\
                                &\textsc{WAS} \cite{fan2024decoupling} & $52.42/45.99 $ & $ 36.76/37.38 $ & $ \underline{34.31/39.09} $& $50.07/45.85  $ & $ 54.23/40.87 $ & $ 44.25/31.42 $ & $4.8$ \\
                                &\textsc{GraphTCM} \cite{fang2024exploring}  & $ 52.15/46.41 $ & $ \underline{43.92/44.45}$ & $ 28.20/26.21$& $ \underline{51.12/48.32} $ & $55.35/36.34 $ & $ 46.34/28.82$ & $3.3$ \\ \hline
\multirow{5}{*}{\textsc{Multi-Scale Learning}}
                                 & \textsc{ANEMONE} \cite{jin2021anemone}
                                 & $46.40/38.33$ & $33.91/31.37$ & $29.55/30.02$
                                 & $24.94/18.07$ & $46.28/26.15$ & $41.22/21.25$ & $7.7$ \\
                                
                                 & \textsc{MERIT} \cite{ijcai2021p204}
                                 & $51.02/41.25$ & $42.91/43.84$ & $27.84/25.45$
                                 & $27.25/15.87$ & $54.95/43.24$ & $49.62/29.20$ & $5.5$ \\
                                
                                 & \textsc{MSSGCL} \cite{ijcai2023p246}
                                 & $41.35/22.84$ & $25.02/22.46$ & $20.24/19.59$
                                 & $24.32/17.13$ & $39.07/22.34$ & $33.66/18.42$ & $9.3$ \\
                                
                                 & \textsc{LMGTA} \cite{li2025anomaly}
                                 & $46.41/37.75$ & $38.51/35.25$ & $24.77/24.56$
                                 & $26.98/22.61$ & $48.64/32.02$ & $41.55/23.74$ & $7.2$ \\
                                &\textsc{LSW-ML-GSSL} &$\mathbf{61.66/58.00}$ & $\mathbf{45.22/46.22}$& $\mathbf{35.34/42.56}$& $\mathbf{51.39/52.97}$ &$\mathbf{57.68/46.55}$ & $\mathbf{54.47/52.73}$&$\mathbf{1.0}$\\ \hline

\rowcolor{gray!25} \multicolumn{9}{c}{\textsc{Link Prediction} (ROC-AUC)} \\ \hline

\multirow{5}{*}{\textsc{Multi-Task Learning}} & \textsc{AutoSSL} \cite{jin2021automated} &$85.25 \pm \scriptstyle 0.87$ & $91.25 \pm \scriptstyle 0.67$& $92.21 \pm \scriptstyle 0.18$ &$93.04 \pm \scriptstyle 0.46$ &$94.46 \pm \scriptstyle0.09$ & $91.21 \pm \scriptstyle 0.15
$& $6.8$\\
                                &\textsc{ParetoGNN} \cite{ju2022multi}  & $54.89 \pm \scriptstyle 3.32$ & $58.78 \pm \scriptstyle 1.78$&$94.16 \pm \scriptstyle 0.14$ &$93.56 \pm \scriptstyle 0.89$ & $\underline{95.04 \pm \scriptstyle 0.05}$&$\underline{93.51 \pm \scriptstyle 0.03}$ & $5.8$ \\ 
                                &\textsc{DyFSS} \cite{zhu2024every}  & $92.55 \pm \scriptstyle 0.52 $ & $94.10 \pm \scriptstyle0.10 $ & $\underline{95.53 \pm \scriptstyle 0.18} $& $95.63 \pm \scriptstyle0.23 $ & $ 94.24\pm \scriptstyle0.07 $ & $92.75 \pm \scriptstyle 0.08 \scriptstyle $ & $\underline{3.3}$ \\
                                &\textsc{WAS} \cite{fan2024decoupling} & $ 92.43\pm \scriptstyle0.07 $ &92.43 $ \pm \scriptstyle0.16 $ & $89.93 \pm \scriptstyle0.13 $& $ 93.78\pm \scriptstyle0.16 $ & $93.86 \pm \scriptstyle0.06 $ & $92.45 \pm \scriptstyle0.19 $ & $5.7$ \\
                                &\textsc{GraphTCM} \cite{fang2024exploring} & $ 90.91\pm \scriptstyle0.37 $ & $\underline{94.36 \pm \scriptstyle 0.26} $ & $ 93.55 \pm \scriptstyle 0.48$& $ 95.32\pm \scriptstyle0.16 $ & $93.51 \pm \scriptstyle 0.21$ & $ 92.21\pm \scriptstyle0.20 $ & $5.0$ \\
                                 \hline
\multirow{5}{*}{\textsc{Multi-Scale Learning}}
                                 & \textsc{ANEMONE} \cite{jin2021anemone}
                                 & $92.76 \pm \scriptstyle 0.08$
                                 & $93.12 \pm \scriptstyle 0.41$
                                 & $94.58 \pm \scriptstyle 0.21$
                                 & $93.52 \pm \scriptstyle 0.15$
                                 & $89.25 \pm \scriptstyle 0.11$
                                 & $93.19 \pm \scriptstyle 0.06$
                                 & $5.0$\\
                                
                                 & \textsc{MERIT} \cite{ijcai2021p204}
                                 & $\underline{92.97 \pm \scriptstyle 0.04}$
                                 & $93.45 \pm \scriptstyle 0.13$
                                 & $94.60 \pm \scriptstyle 0.27$
                                 & $\underline{95.85 \pm \scriptstyle 0.19}$
                                 & $91.34 \pm \scriptstyle 0.20$
                                 & $90.07 \pm \scriptstyle 0.14$
                                 & $4.3$ \\
                                
                                 & \textsc{MSSGCL} \cite{ijcai2023p246}
                                 & $84.04 \pm \scriptstyle 0.15$
                                 & $89.84 \pm \scriptstyle 0.50$
                                 & $85.17 \pm \scriptstyle 0.19$
                                 & $84.55 \pm \scriptstyle 0.55$
                                 & $82.18 \pm \scriptstyle 0.38$
                                 & $80.02 \pm \scriptstyle 0.24$
                                 & $9.7$\\

                                 & \textsc{LMGTA} \cite{li2025anomaly}
                                 & $86.56 \pm \scriptstyle 0.09$
                                 & $91.84 \pm \scriptstyle 0.26$
                                 & $87.79 \pm \scriptstyle 0.06$
                                 & $89.30 \pm \scriptstyle 0.53$
                                 & $88.85 \pm \scriptstyle 0.23$
                                 & $88.52 \pm \scriptstyle 0.44$
                                 & $8.3$ \\

                                &\textsc{LSW-ML-GSSL} & $\mathbf{93.64 \pm \scriptstyle 0.66}$ &$\mathbf{94.76 \pm \scriptstyle 0.63}$ &$\mathbf{96.84 \pm \scriptstyle 0.26}$ &$\mathbf{96.11 \pm \scriptstyle 0.22}$ &$\mathbf{95.22 \pm \scriptstyle 0.23}$ &$\mathbf{94.34 \pm \scriptstyle 0.15}$ &$\mathbf{1.0}$ \\ \hline

\bottomrule
\end{tabular}}
\end{table*}

\begin{table*}
\centering

\caption{The effect of the self-weighting mechanism compared to its linear counterparts.}
\label{linear_vs_self_weighted}
\scalebox{0.68}{\begin{tabular}{l|cccccc}
\toprule
\toprule
 Method & Cora & CiteSeer & Pubmed & DBLP & Photo & Computers  \\ 
\midrule

\rowcolor{gray!25} \multicolumn{7}{c}{\textsc{Node Classification} (Accuracy)} \\ \hline

            \textsc{Best-SL-GSSL} &$\underline{84.35 \pm \scriptstyle 0.48}$ &$\underline{71.71 \pm \scriptstyle 0.77}$ & $\underline{85.35 \pm \scriptstyle 0.20}$&$\underline{84.42 \pm \scriptstyle 0.23}$ &$\underline{92.78 \pm \scriptstyle 0.31}$ & $\mathbf{87.69 \pm \scriptstyle 0.13}$\\
            \textsc{L-ML-GSSL} &$83.75 \pm \scriptstyle 0.78$ &$71.47 \pm \scriptstyle 0.78$ &$84.78 \pm \scriptstyle 0.56$ &$84.16 \pm  \scriptstyle 0.23$ &$92.51 \pm \scriptstyle 0.48$ & $85.35 \pm \scriptstyle 0.35$ \\
            \textsc{LW-ML-GSSL} &$83.89 \pm \scriptstyle 0.35$ & $71.23 \pm \scriptstyle 0.75$ & $84.14 \pm \scriptstyle 0.36$ & $84.25 \pm \scriptstyle 0.20$ & $91.68 \pm \scriptstyle 0.15$ & $85.12 \pm \scriptstyle 0.08$\\
            \textsc{LSW-ML-GSSL} &$\mathbf{84.87 \pm \scriptstyle 0.75}$ & $\mathbf{73.43 \pm \scriptstyle 0.45}$&$\mathbf{85.46 \pm \scriptstyle 0.13}$ &$\mathbf{84.47 \pm \scriptstyle 0.21}$ & $\mathbf{92.80 \pm \scriptstyle 0.45}$& $\underline{86.75 \pm \scriptstyle 0.35}$ \\ \hline

\rowcolor{gray!25} \multicolumn{7}{c}{\textsc{Node Clustering} (NMI/ARI)} \\ \hline

            \textsc{Best-SL-GSSL} &$57.50/52.02$ & $\mathbf{45.64/47.05}$&$\underline{35.23/42.25}$ &$\underline{51.27/48.56}$ & $\underline{59.15/43.08}$&$\underline{54.23/52.53}$\\
            \textsc{L-ML-GSSL} & $\underline{58.05/52.05}$ & $44.72/45.92$& $33.25/41.32$& $50.29/47.86$& $54.12/40.53$& $54.13/52.33$\\
            \textsc{LW-ML-GSSL} &$57.89/52.46$ & $42.33/43.52$ & $34.27/41.51$ & $50.78/48.15$ & $52.88/38.95$ & $53.74/52.07$\\
            \textsc{LSW-ML-GSSL} &$\mathbf{61.66/58.00}$ & $\underline{45.22/46.22}$& $\mathbf{35.34/42.56}$& $\mathbf{51.39/52.97}$ &$\mathbf{57.68/46.55}$ & $\mathbf{54.47/52.73}$\\ \hline

\rowcolor{gray!25} \multicolumn{7}{c}{\textsc{Link Prediction} (ROC-AUC)} \\ \hline

            \textsc{Best-SL-GSSL} & $\underline{93.31 \pm \scriptstyle 0.21}$ &$\underline{94.34 \pm \scriptstyle 0.84}$ &$\underline{96.33 \pm \scriptstyle 0.39}$ &$\underline{95.84 \pm \scriptstyle 0.32}$ & $\underline{95.13 \pm \scriptstyle 0.34}$&$94.09 \pm \scriptstyle 0.10$  \\
            \textsc{L-ML-GSSL} &$89.05 \pm \scriptstyle 0.91$ & $88.45 \pm \scriptstyle 0.71$ & $93.25 \pm \scriptstyle0.12$&$92.77 \pm \scriptstyle 0.12$ & $90.83 \pm \scriptstyle 0.45$& $\underline{94.24 \pm \scriptstyle0.10}$\\
            \textsc{LW-ML-GSSL} & $89.82 \pm \scriptstyle 0.96$ & $89.00 \pm \scriptstyle 0.40$ & $93.85 \pm \scriptstyle 0.09$ & $93.90 \pm \scriptstyle 0.25$ & $92.26 \pm \scriptstyle 0.47$ & $94.05 \pm \scriptstyle 0.08$\\
            \textsc{LSW-ML-GSSL} & $\mathbf{93.64 \pm \scriptstyle 0.66}$ &$\mathbf{94.76 \pm \scriptstyle 0.63}$ &$\mathbf{96.84 \pm \scriptstyle 0.26}$ &$\mathbf{96.11 \pm \scriptstyle 0.22}$ &$\mathbf{95.22 \pm \scriptstyle 0.23}$ &$\mathbf{94.34 \pm \scriptstyle 0.15}$\\ \hline

\bottomrule
\end{tabular}}
\end{table*}

\begin{table*}
\centering
\caption{Comparison of multi-loss weighting methods and LSW-ML-GSSL.}
\label{lsw_vs_other_weighting_methods}
\scalebox{0.68}{\begin{tabular}{l|cccccc}
\toprule
\toprule
 Method & Cora & CiteSeer & Pubmed & DBLP & Photo & Computers  \\ 
\midrule

\rowcolor{gray!25} \multicolumn{7}{c}{\textsc{Node Classification} (Accuracy)} \\ \hline

            \textsc{Uncertainty} \cite{kendall2018multi} & $83.94 \pm \scriptstyle 0.27$ & $70.49 \pm \scriptstyle 0.61$ & $82.60 \pm \scriptstyle 0.32$ & $84.07 \pm \scriptstyle 0.14$ & $91.81 \pm \scriptstyle 0.33$ & $83.47 \pm \scriptstyle 0.29$  \\
            \textsc{GradNorm} \cite{chen2018gradnorm} & $83.01 \pm \scriptstyle 0.82$& $70.52 \pm \scriptstyle 0.42$ & $\underline{83.10 \pm \scriptstyle 0.22}$ & $\underline{84.31 \pm \scriptstyle 0.13}$ & $91.67 \pm \scriptstyle 0.29$& $83.52 \pm \scriptstyle 0.38$  \\
            \textsc{PCGrad} \cite{yu2020gradient} & $83.83 \pm \scriptstyle 0.34$ & $70.55 \pm \scriptstyle 0.75$ & $82.12 \pm \scriptstyle 0.27$ & $84.04 \pm \scriptstyle 0.17$ & $91.87 \pm \scriptstyle 0.10$ & $83.56 \pm \scriptstyle 0.32$  \\
            \textsc{CAGrad} \cite{liu2021conflict} & $84.42 \pm \scriptstyle 0.41$ & $70.31 \pm \scriptstyle 1.40$ & $82.12 \pm \scriptstyle 0.11$ & $84.17 \pm \scriptstyle 0.19$ & $91.71 \pm \scriptstyle 0.27$ & $84.04 \pm \scriptstyle 0.74$ \\
            \textsc{AdaTask} \cite{yang2023adatask} & $84.35 \pm \scriptstyle 0.23$ & $70.56 \pm \scriptstyle 0.30$ & $81.99 \pm \scriptstyle 0.16$ & $83.99 \pm \scriptstyle 0.16$ & $91.66 \pm \scriptstyle 0.17$ & $\underline{84.17 \pm \scriptstyle 0.35}$  \\
            \textsc{IGBv1} \cite{dai2023improvable} & $83.95 \pm \scriptstyle 0.41$ & $\underline{70.87 \pm \scriptstyle 0.32}$ & $82.12 \pm \scriptstyle 0.26$ & $84.00 \pm \scriptstyle 0.17$ & $91.84 \pm \scriptstyle 0.29$ & $84.15 \pm \scriptstyle 0.43$  \\
            \textsc{AUTAUT} \cite{zhongautomatic} & $\underline{84.70 \pm \scriptstyle 0.40}$ & $70.11 \pm \scriptstyle 0.96$ & $81.92 \pm \scriptstyle 0.40$ & $83.95 \pm \scriptstyle 0.19$ & $\underline{91.87 \pm \scriptstyle 0.44}$ & $84.04 \pm \scriptstyle 0.74$  \\
            \textsc{LSW-ML-GSSL} &$\mathbf{84.87 \pm \scriptstyle 0.75}$ & $\mathbf{73.43 \pm \scriptstyle 0.45}$&$\mathbf{85.46 \pm \scriptstyle 0.13}$ &$\mathbf{84.47 \pm \scriptstyle 0.21}$ & $\mathbf{92.80 \pm \scriptstyle 0.45}$& $\mathbf{86.75 \pm \scriptstyle 0.35}$ \\ \hline

\rowcolor{gray!25} \multicolumn{7}{c}{\textsc{Node Clustering} (NMI/ARI)} \\ \hline

            \textsc{Uncertainty} \cite{kendall2018multi} & $57.88/52.83$ & $42.63/44.48$ & $33.95/31.83$ & $\underline{50.81/51.38}$ & $55.16/37.55$ & $39.58/27.87$  \\
            \textsc{GradNorm} \cite{chen2018gradnorm} & $\underline{58.61/53.61}$ & $42.95/43.94$ & $34.29/32.19$ & $50.15/52.75$ & $52.94/35.41$ & $40.67/28.45$  \\
            \textsc{PCGrad} \cite{yu2020gradient} & $58.02/53.63$ & $42.45/43.42$ & $33.91/31.70$ & $38.52/41.26$ & $52.75/34.16$ & $33.45/23.49$   \\
            \textsc{CAGrad} \cite{liu2021conflict} & $58.31/52.57$ & $\underline{43.38/44.63}$ & $34.23/33.36$ & $50.60/51.62$ & $56.75/40.12$ & $41.21/29.65$  \\
            \textsc{AdaTask} \cite{yang2023adatask} & $57.35/52.43$ & $42.27/43.44$ & $31.31/27.19$ & $50.02/49.91$ & $47.77/29.63$ & $23.49/15.21$  \\
            \textsc{IGBv1} \cite{dai2023improvable} & $58.19/53.75$ & $42.46/43.74$ & $\underline{34.55/33.94}$ & $50.78/52.05$ & $\underline{56.95/42.23}$ & $\underline{41.56/28.96}$  \\
            \textsc{AUTAUT} \cite{zhongautomatic} & $58.06/52.86$ & $43.04/44.63$ & $33.22/30.94$ & $50.52/51.57$ & $55.02/38.84$ & $40.72/28.88$  \\
            \textsc{LSW-ML-GSSL} &$\mathbf{61.66/58.00}$ & $\mathbf{45.22/46.22}$& $\mathbf{35.34/42.56}$& $\mathbf{51.39/52.97}$ &$\mathbf{57.68/46.55}$ & $\mathbf{54.47/52.73}$\\ \hline

\rowcolor{gray!25} \multicolumn{7}{c}{\textsc{Link Prediction} (ROC-AUC)} \\ \hline

            \textsc{Uncertainty} \cite{kendall2018multi} & $91.14 \pm \scriptstyle 0.13$ & $93.48 \pm \scriptstyle 0.53$ & $94.41 \pm \scriptstyle 0.14$ & $94.62 \pm \scriptstyle 0.07$ & $94.83 \pm \scriptstyle 0.03$ & $93.03 \pm \scriptstyle 0.11$  \\
            \textsc{GradNorm} \cite{chen2018gradnorm} & $\underline{91.37 \pm \scriptstyle 0.32}$ & $\underline{94.28 \pm \scriptstyle 0.66}$ & $\underline{94.65 \pm \scriptstyle 0.33}$ & $94.99 \pm \scriptstyle 0.16$ & $94.53 \pm \scriptstyle 0.18$ & $92.84 \pm \scriptstyle 0.26$  \\
            \textsc{PCGrad} \cite{yu2020gradient} & $90.07 \pm \scriptstyle 0.62$ & $92.94 \pm \scriptstyle 0.29$ & $92.61 \pm \scriptstyle 0.45$ & $94.39 \pm \scriptstyle 0.17$ & $93.89 \pm \scriptstyle 0.14$ & $93.08 \pm \scriptstyle 0.61$  \\
            \textsc{CAGrad} \cite{liu2021conflict} & $91.30 \pm \scriptstyle 0.98$ & $93.61 \pm \scriptstyle 0.86$ & $94.50 \pm \scriptstyle 0.49$ & $94.83 \pm \scriptstyle 0.37$ & $94.88 \pm \scriptstyle 0.32$ & $93.02 \pm \scriptstyle 0.29$  \\
            \textsc{AdaTask} \cite{yang2023adatask} & $90.15 \pm \scriptstyle 0.69$ & $93.53 \pm \scriptstyle 0.41$ & $93.55 \pm \scriptstyle 0.39$ & $94.14 \pm \scriptstyle 0.47$ & $89.33 \pm \scriptstyle 0.25$ & $89.39 \pm \scriptstyle 0.57$  \\
            \textsc{IGBv1} \cite{dai2023improvable} & $91.12 \pm \scriptstyle 0.48 $ & $93.65 \pm \scriptstyle 0.68$ & $94.59 \pm \scriptstyle 0.18$ & $95.17 \pm \scriptstyle 0.53$ & $\underline{94.92 \pm \scriptstyle 0.11}$ & $\underline{93.25 \pm \scriptstyle 0.30}$  \\
            \textsc{AUTAUT} \cite{zhongautomatic} & $91.33 \pm \scriptstyle 0.65$ & $94.13 \pm \scriptstyle 0.88$ & $94.51 \pm \scriptstyle 0.15$ & $\underline{95.42 \pm \scriptstyle 0.28}$ & $94.37 \pm \scriptstyle 0.09$ & $92.27 \pm \scriptstyle 0.16$  \\
            \textsc{LSW-ML-GSSL} & $\mathbf{93.64 \pm \scriptstyle 0.66}$ &$\mathbf{94.76 \pm \scriptstyle 0.63}$ &$\mathbf{96.84 \pm \scriptstyle 0.26}$ &$\mathbf{96.11 \pm \scriptstyle 0.22}$ &$\mathbf{95.22 \pm \scriptstyle 0.23}$ &$\mathbf{94.34 \pm \scriptstyle 0.15}$\\ \hline

\bottomrule
\end{tabular}}
\end{table*}

\begin{table*}
\caption{The impact of each abstraction level on the performance of LSW-ML-GSSL.}
  \begin{center}
  \begin{small}
  \scalebox{0.68}{
  \begin{tabular}{|c c c c | c c c c | c c c c | c c c c|}
    \hline
    \multicolumn{4}{|c|}{Method} & \multicolumn{4}{c|}{Cora} & \multicolumn{4}{c|}{CiteSeer} &  \multicolumn{4}{c|}{DBLP} \\
    \cline{1-16}
    Node & Proximity & Cluster & Graph & ACC & NMI & ARI & AUC & ACC & NMI & ARI & AUC & ACC & NMI & ARI & AUC \\ \hline
    \ding{51} & \ding{56} & \ding{56} & \ding{56} & $83.77$ & $56.74$ &$52.61$  & $92.70$ &$69.98$  & $41.35$ & $41.09$ & $92.75$ &$83.92$ &$50.88$  & $51.75$ & $94.10$ \\
    \ding{56} & \ding{51} & \ding{56} & \ding{56} & $83.65$ & $57.47$ & $52.45$ & $92.16$ & $70.65$ & $40.52$ & $40.37$ & $93.04$ & $83.61$& $50.90$ & $52.29$& $94.76$ \\
    \ding{56} & \ding{56} & \ding{51} & \ding{56} & $79.37$ & $46.19$ & $42.44$ & $86.51$ & $64.91$ & $41.00$ &$39.21$  & $89.61$ &$81.23$ & $40.12$ & $43.81$ & $91.08$\\
    \ding{56} & \ding{56} & \ding{56} & \ding{51} & $80.49$ & $51.00$ & $46.06$ & $93.12$ & $70.66$ & $40.40$ & $41.48$ & $92.18$ & $82.70$ & $50.70$ & $51.62$ & $92.24$ \\
    \ding{51} & \ding{51} & \ding{56} & \ding{56} & $84.24$ & $59.31$ & $54.88$ & $92.79$ & $71.45$ & $42.68$ & $43.32$ & $\underline{94.21}$ & $83.45$ & $50.45$ & $51.65$ & $94.77$ \\
    \ding{51} & \ding{56} & \ding{51} & \ding{56} & $84.12$ & $55.05$ & $51.17$ & $92.34$ & $70.58$ & $41.85$ & $41.92$ & $92.92$ & $83.96$ & $51.02$ & $52.21$ & $94.96$\\
    \ding{51} & \ding{56} & \ding{56} & \ding{51} & $83.56$ & $59.35$ & $54.87$ &  $91.20$ & $71.35$ & $43.73$ & $44.54$ & $92.84$ & $83.85$ & $50.87$ & $52.14$ & $95.04$ \\
    \ding{56} & \ding{51} & \ding{51} & \ding{56} & $83.63$ & $54.82$ & $48.88$ & $89.79$ & $70.98$ & $42.49$ & $44.64$ & $91.68$ & $81.97$ & $50.37$ & $51.67$  & $93.25$ \\
    \ding{56} & \ding{51} & \ding{56} & \ding{51} & $83.89$ & $56.82$ & $52.58$ & $89.08$ & $69.50$ & $43.47$ & $45.16$ & $91.84$ & $82.64$ & $50.74$ & $52.03$ & $93.77$ \\
    \ding{56} & \ding{56} & \ding{51} & \ding{51} & $82.97$ & $56.46$ & $52.82$ & $88.55$ & $71.55$  & $43.21$ & $43.70$ & $92.47$ & $82.02$ & $50.25$ & $51.98$ & $94.16$\\
    \ding{56} & \ding{51} & \ding{51} & \ding{51} & $81.09$ & $58.35$ & $56.66$ & $88.62$ &$70.65$  & $42.77$ & $43.10$ & $91.62$ &$84.18$  & $\underline{51.30}$ & $52.37$& $93.43$\\  
    \ding{51} & \ding{56} & \ding{51} & \ding{51}  & $\underline{84.73}$ & $\underline{59.93}$ &$\underline{57.27}$  & $91.62$ & $71.68$ & $43.18$ & $43.91$ & $91.58$ & $\underline{84.43}$ & $51.23$ &$\underline{52.49}$ &$93.28$ \\ 
    \ding{51} & \ding{51} & \ding{56} &  \ding{51} &$84.65$  & $58.52$ & $53.86$ &$\underline{93.24}$  & $\underline{71.81}$ &$\underline{44.47}$  & $\underline{45.44}$ & $92.65$ &$84.28$  & $50.27$ &$51.87$ & $\underline{95.07}$\\ 
    \ding{51} & \ding{51}  & \ding{51} & \ding{56} &$83.74$ & $57.49$ & $53.13$ &$91.91$  & $68.09$ &$37.44$  &$37.94$  & $93.61$ &$83.37$  & $50.43$ &$51.91$ &$94.98$ \\ 
    \ding{51} & \ding{51}  & \ding{51} & \ding{51} &$\mathbf{84.87}$  & $\mathbf{61.66}$ & $\mathbf{58.00}$  & $\mathbf{93.64}$ & $\mathbf{73.43}$ & $\mathbf{45.22}$ & $\mathbf{46.22}$ & $\mathbf{94.76}$ & $\mathbf{84.47}$ & $\mathbf{51.39}$ & $\mathbf{52.97}$& $\mathbf{96.11}$ \\
    \hline
  \end{tabular}}
  \end{small}
  \end{center}
  
  \label{Table:table_3}
\end{table*}

\textbf{Hyperparameters.} To ensure a fair comparison, the same GNN encoder architecture is employed across all baselines for each dataset. The GNN encoder is a graph convolutional network (GCN) that projects the input graph into a latent space of dimension 256. We use PReLU activation functions for the GNN encoder. Subsequently, a two-layer fully-connected projection head refines the embeddings. The ELU activation function is used for the projection head. We measure positive and negative similarity scores using a shifted cosine similarity ($(1+cos)/2$). The graph augmentation hyperparameters control the transformations applied to the graph during training. These augmentations include edge dropout for the two positive augmentations, which randomly remove edges from the graph with two different rates. 
Node feature dropout for the two positive augmentations introduces stochasticity by masking a fraction of the node features with two different probability rates.

We categorize the hyperparameters of our model into two types. The first type consists of constant hyperparameters that are independent of the processed dataset, as detailed in Table \ref{Table:hyperparameters_indep}. This category includes hyperparameters related to data augmentation, the GNN architecture, and the training process.
For example, the edge drop rates and feature drop rates for positive augmentations are fixed, as are the GNN hidden and latent dimensions, the learning rate, and the optimizer. The second category consists of three hyperparameters influenced by the input dataset: $m$, $\gamma$, and the number of epochs. We select fixed values for $m$ and $\gamma$ from the respective ranges $[0.10, \, 0.15, \, 0.20, \, 0.25, \, 0.30]$ and $[1.0, \, 1.5, \, 2.0, \, 2.5, \, 3.0]$. The number of epochs is determined based on the maximum validation accuracy. The hyperparameters that are influenced by the properties of the data are provided in Table \ref{Table:hyperparameters_dep}.

\textbf{Single-Task Results.} Table~\ref{table_1} summarizes the comparative performance of various single-task GSSL methods across multiple datasets and tasks. The results are categorized based on the abstraction levels of the methods—\textsc{Node}, \textsc{Proximity}, \textsc{Cluster}, and \textsc{Graph}. The evaluated tasks include node classification (Accuracy), node clustering (NMI/ARI), and link prediction (ROC-AUC). We report the average performance across all tasks to provide a comprehensive assessment of overall effectiveness. Our methods yield competitive results across diverse tasks and datasets and consistently outperform previous approaches in terms of average performance. In particular, \textsc{SL-GSSL} excels in node clustering and link prediction. 

It is clear that no single abstraction level consistently outperforms all others in every downstream task. Each abstraction level has its strengths and weaknesses.  
Its suitability varies depending on the task’s demands and inherent properties. In the task of node classification, our node-level approach shows superior performance on 4 out of the 6 evaluated datasets. This could be attributed to its ability to capture fine-grained patterns that are critical for node classification. By focusing on individual nodes, the model benefits from a detailed understanding of node-specific features. 
For link prediction, our proximity-level approach shows superior performance on most datasets. The superior performance of the proximity-level approach in link prediction tasks stems from its ability to effectively capture and utilize relational patterns and interactions between node pairs. Unlike node-level approaches that focus on individual nodes and their features, proximity-level models prioritize the structural and semantic relationships between nodes, which are essential for accurate link prediction. In addition, the results show that our node-level approach outperforms other methods in terms of clustering performance. However, the cluster-level methods have the potential to yield superior results when combined with a pretext task, which was not included in our experiments. Consequently, methods such as GMM-VGAE* and DGAE*, which strongly rely on proximity-level pretraining, exhibit lower performance in this context. Overall, the results reveal that none of the single-level methods has consistent generalizability across all tasks and datasets. This limitation highlights the need for multi-task GSSL. 

\textbf{Ablation on the Single-Level Unified Loss.} To assess how the single-level unified objective benefits from its design choices, we compare \textsc{Best-SL-GSSL} with three ablation variants that remove the margin, replace the exponential term in Eq.~(\ref{eq.loss_unified}) with a hinge form \cite{hoffer2015deep}, or substitute the loss with the standard InfoNCE objective \cite{zhu2020deep}.
\textsc{Best-SL-GSSL-No-Margin} sets $m{=}0$.
\textsc{Best-SL-GSSL-Hinge-Loss} replaces the exponential transformation with the hinge form $\max(0, s^{-}-s^{+}+m)$.
\textsc{Best-SL-GSSL-InfoNCE} replaces the unified loss with the standard InfoNCE objective over the same positive and negative similarity scores. As can be seen in Table \ref{sl_ablation}, \text{Best-SL-GSSL} is consistently the most effective and robust across datasets and tasks. Removing the margin consistently harms performance, especially for clustering, highlighting its role as a separation constraint that stabilizes optimization and improves the structure of the learned embedding space. Replacing the exponential term with a hinge penalty is generally more detrimental, especially for link prediction and clustering, suggesting that the exponential loss yields better optimization. The InfoNCE-based replacement yields lower performance than the unified objective in most settings, indicating that the proposed formulation is more closely matched to the single-level similarity optimization problem addressed in this work. Overall, both the margin and smooth exponential penalty are key to strong downstream generalization.      

\textbf{Multi-Task and Multi-Scale Results.} Table~\ref{table_2} illustrates the performance of our linear self-weighting multi-level GSSL (LSW-ML-GSSL) method on node classification, node clustering, and edge prediction, compared to several state-of-the-art multi-task and multi-scale GSSL approaches. Overall, LSW-ML-GSSL consistently achieves the best and most robust performance across datasets and tasks, highlighting its strong generalization beyond a single evaluation setting. In contrast to multi-task methods that can suffer from objective interference, and multi-scale methods that may not adequately reconcile granularity levels, our approach explicitly integrates node-, proximity-, cluster-, and graph-level signals while automatically balancing their contributions through linear self-weighting. This adaptive coordination prevents any single level from dominating the optimization and promotes representations that remain simultaneously discriminative. This further emphasizes the significance of the new decision boundary in enhancing optimization flexibility and achieving a more definitive convergence status.

\textbf{The Impact of the Proposed Weighting Mechanism.} As shown in Table \ref{linear_vs_self_weighted}, the proposed LSW-ML-GSSL consistently delivers superior performance compared to both its linear weighting counterparts (L-ML-GSSL and LW-ML-GSSL) and the leading single-level baseline (Best-SL-GSSL). Notably, neither L-ML-GSSL nor LW-ML-GSSL is able to surpass the best-performing single-level method. This indicates that simply applying a linear combination of similarities and dissimilarities across the four abstraction levels is insufficient to properly balance their contributions. As a result, the downstream performance of these linear strategies remains limited, often falling below that of Best-SL-GSSL. In contrast, the self-weighting mechanism adjusts the weights in a fine-grained, adaptive way based on how far each similarity and dissimilarity score is from their target values. This enables a more balanced and effective use of multi-level graph semantics, resulting in consistently superior performance.

\begin{figure}[t]
    \centering
    \includegraphics[width=\linewidth]{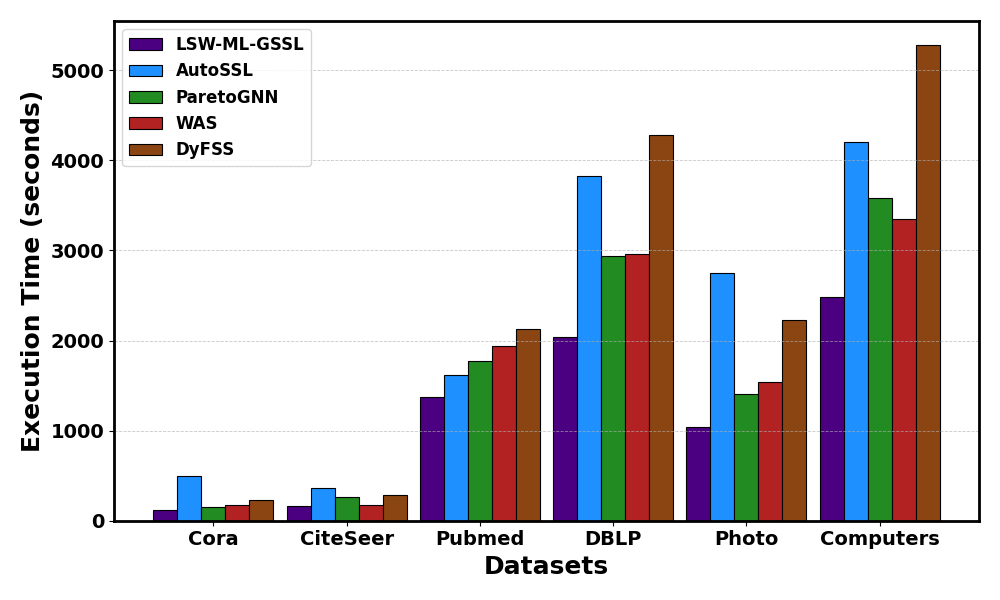} 
    \caption{Execution time in seconds of LSW-ML-GSSL and state-of-the-art multi-task GSSL methods.} 
    \label{fig:execution_time} 

   % \vspace{-4mm}
\end{figure}

\textbf{Comparison with Multi-Loss Weighting Baselines.} Table \ref{lsw_vs_other_weighting_methods} compares our LSW-ML-GSSL approach with seven multi-loss balancing baselines, namely Uncertainty weighting \cite{kendall2018multi}, GradNorm \cite{chen2018gradnorm}, PCGrad \cite{yu2020gradient}, CAGrad \cite{liu2021conflict}, AdaTask \cite{yang2023adatask}, IGBv1 \cite{dai2023improvable}, and AUTAUT \cite{zhongautomatic}. For these methods, the aggregated training loss is obtained by combining the corresponding single level losses through the respective weighting strategy. These approaches regulate optimization primarily at the loss or gradient level by adjusting task-level weights or alleviating gradient conflicts, and thus focus on learning a global compromise among objectives. In opposition, our linear self-weighting mechanism operates at a finer-grained level, weighting similarity and dissimilarity scores directly. Even with these advanced balancing schemes, LSW-ML-GSSL consistently ranks first across tasks and datasets. This suggests that loss reweighting or gradient conflict handling is less effective for multi-level GSSL than our linear self-weighting strategy, which better captures cross-level synergy and improves generalization.

\textbf{Ablation of Abstraction Levels in LSW-ML-GSSL.}
This experiment examines the impact of each abstraction level on the performance of LSW-ML-GSSL across all downstream tasks.  As shown in Table \ref{Table:table_3}, both node-level and graph-level contributions are fundamental to our approach. This implies that achieving effective learning requires the presence of extreme levels of granularity (i.e., node and graph). In addition, multi-level variants consistently surpass single-level ones, showing that each abstraction level contributes distinct and complementary information. It can further be observed that each abstraction level benefits different downstream tasks. Cluster and graph information yield the largest gains for clustering, proximity information is most influential for edge prediction, and the combination of node and graph information proves most effective for node classification. Finally, the full model achieves superior performance compared to all other variants. This highlights the complementary roles of all granularity levels and further illustrates how the self-weighting mechanism is essential for enabling effective synergy among these abstraction levels.

\begin{figure*}[t]
    \centering
    \begin{subfigure}{0.2\linewidth}
        \centering
        \includegraphics[width=\textwidth]{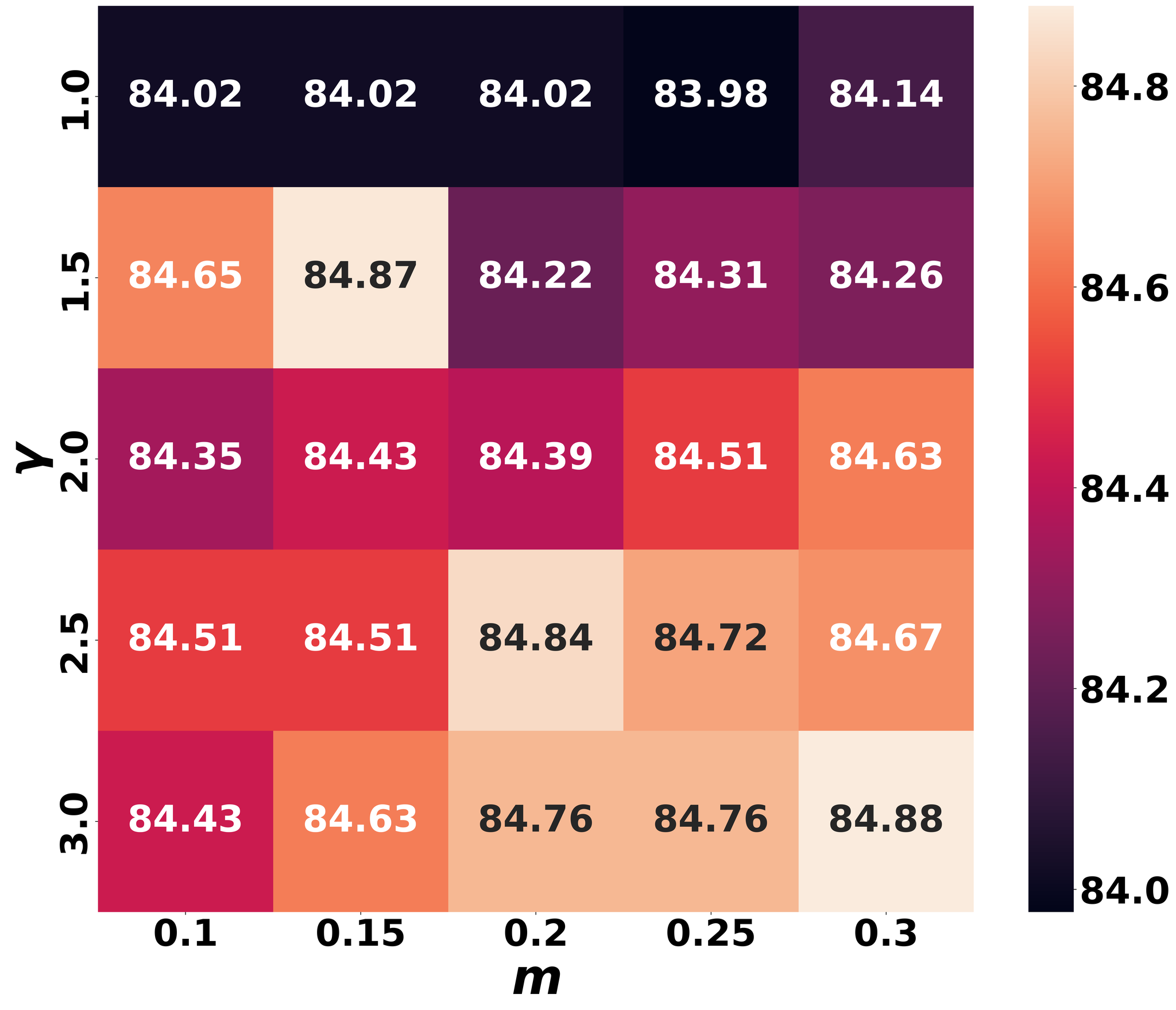}
        \caption{ACC Cora}
    \end{subfigure}
    \hfill
    \begin{subfigure}{0.2\linewidth}
        \centering
        \includegraphics[width=\textwidth]{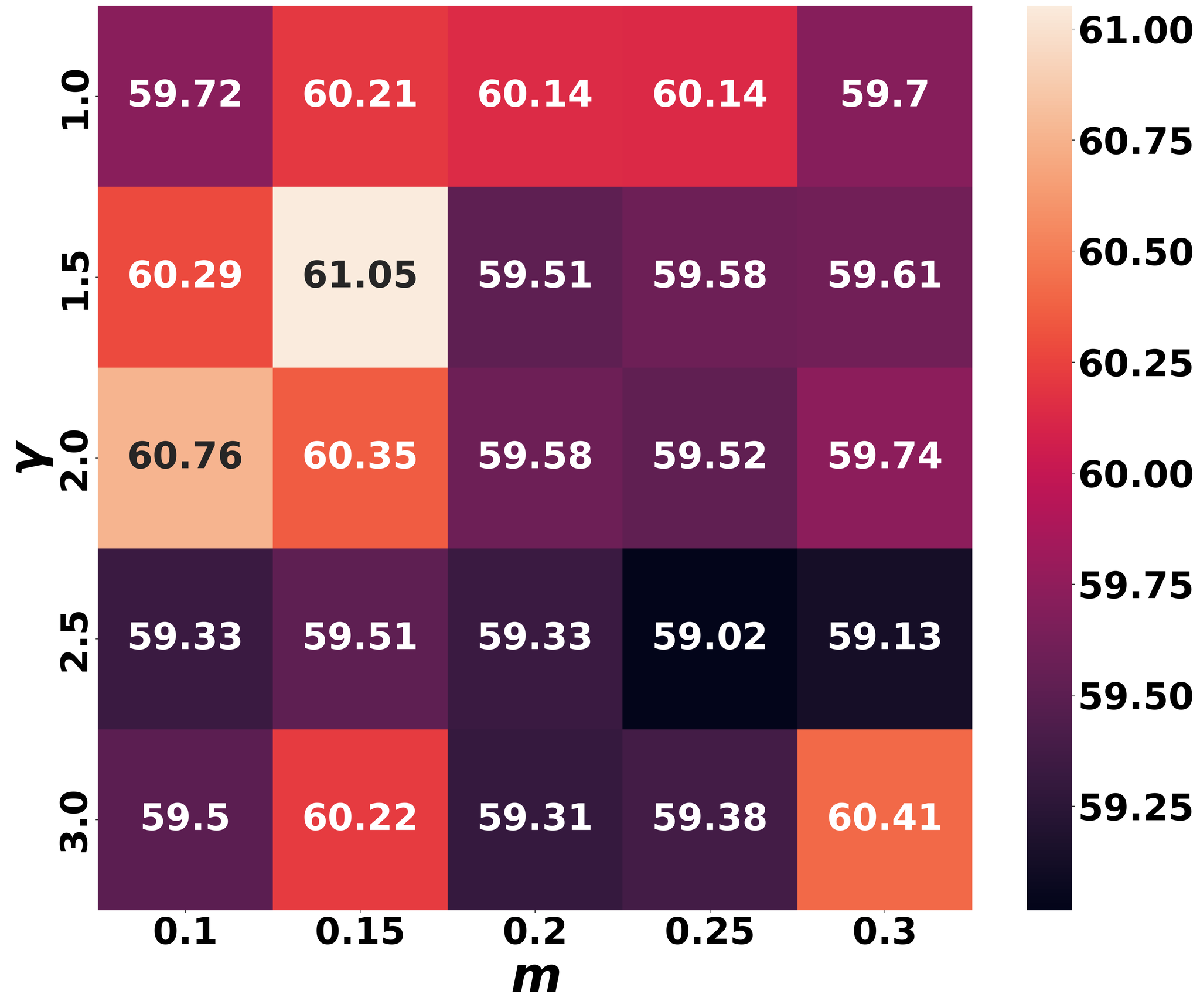}
        \caption{NMI Cora}
    \end{subfigure}
    \hfill
    \begin{subfigure}{0.2\linewidth}
        \centering
      \includegraphics[width=\textwidth]{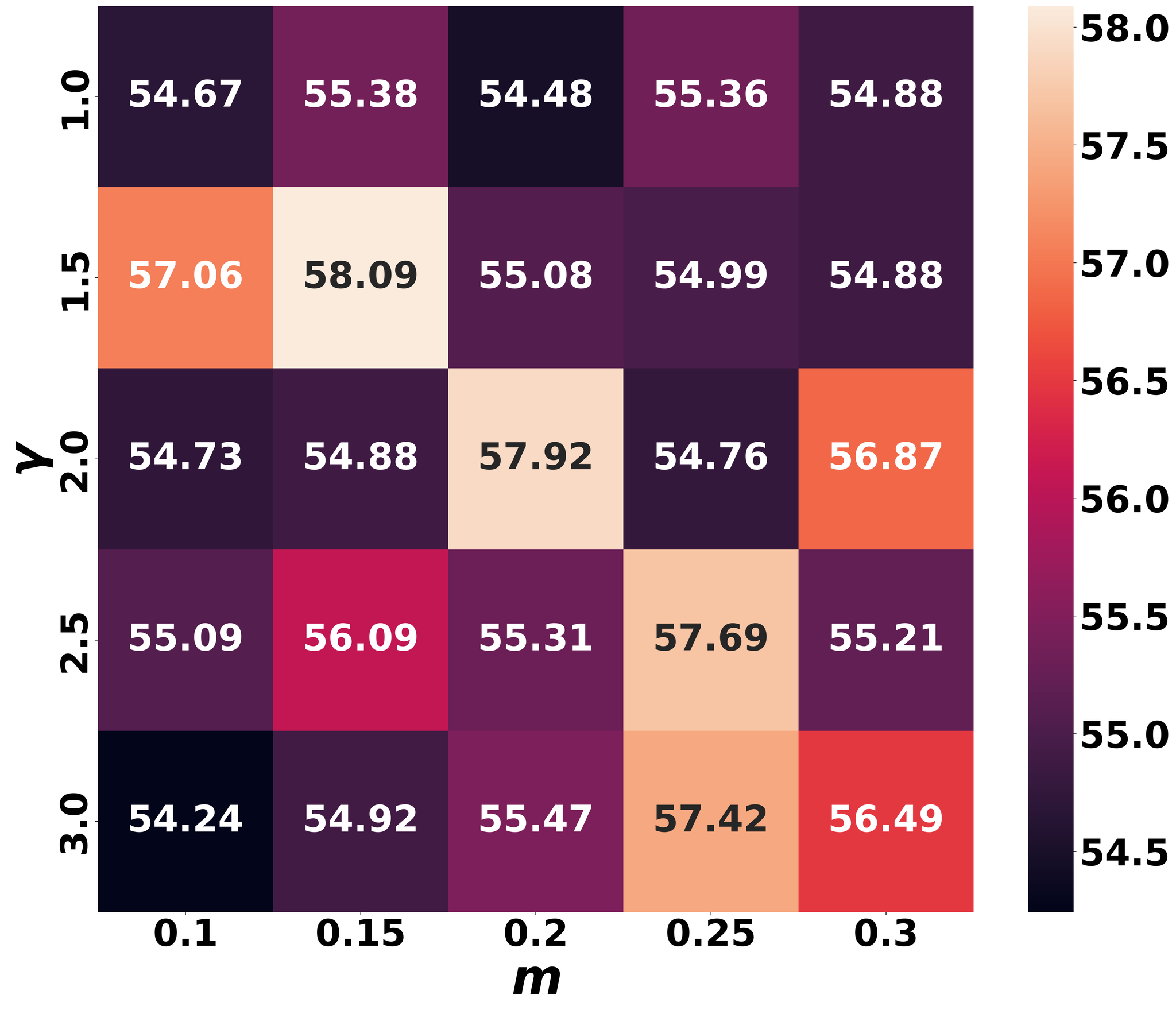}
        \caption{ARI Cora}
    \end{subfigure}
     \hfill
 \begin{subfigure}{0.2\linewidth}
        \centering
      \includegraphics[width=\textwidth]{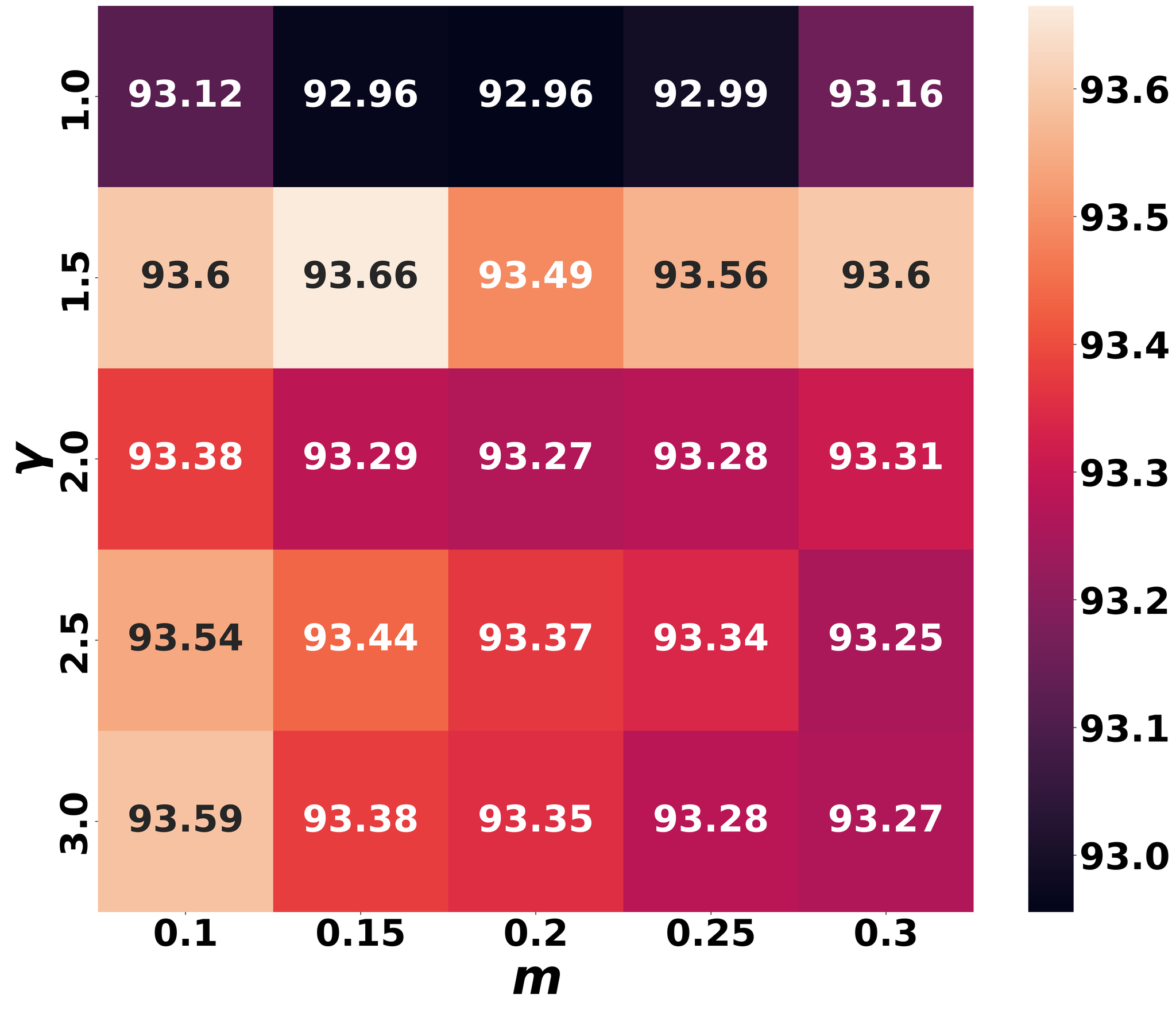}
        \caption{ROC-AUC Cora}
    \end{subfigure}   

    \vspace{0.5cm} 
    
    \begin{subfigure}{0.2\linewidth}
        \centering
        \includegraphics[width=\textwidth]{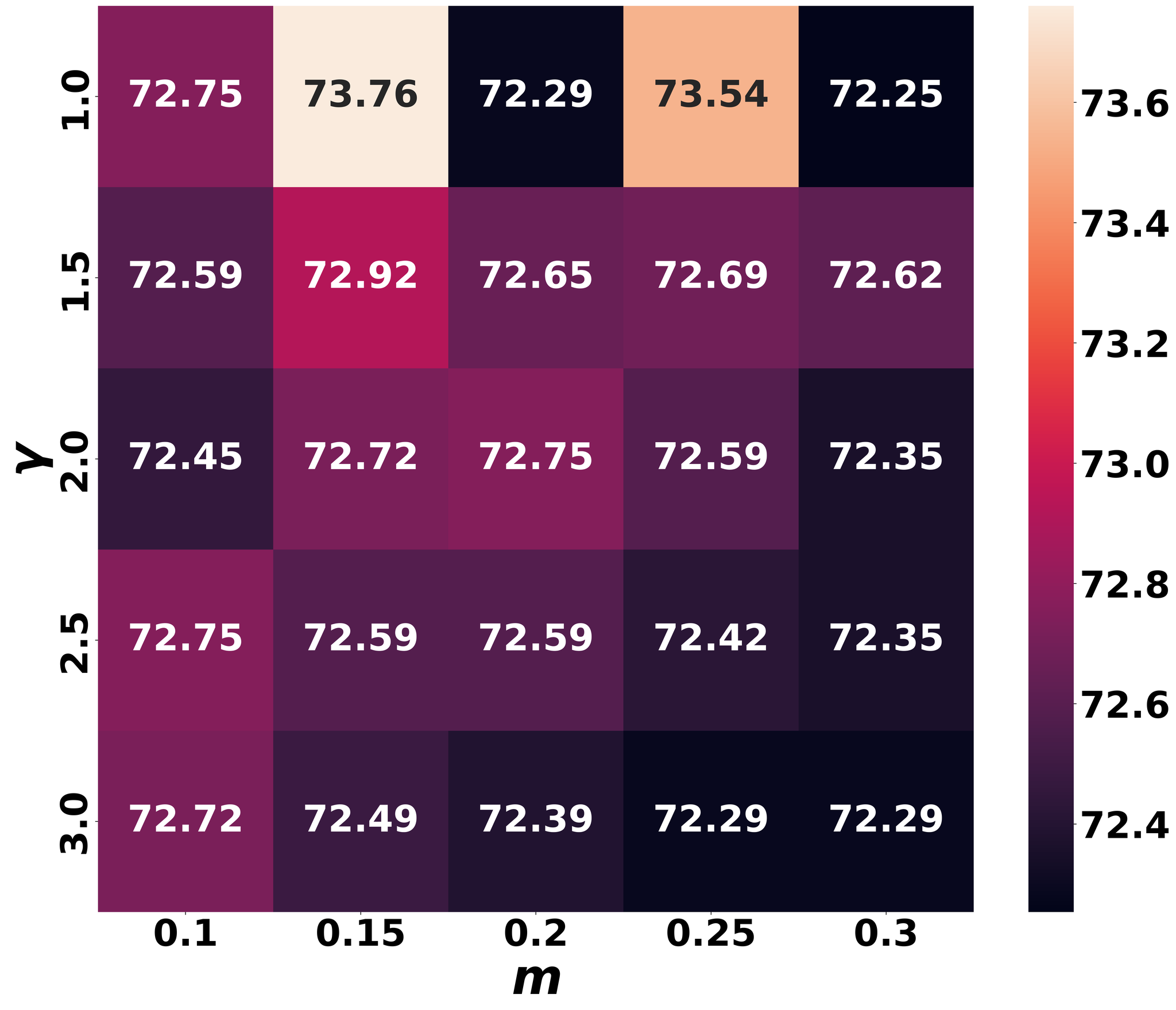}
        \caption{ACC CiteSeer}
    \end{subfigure}
    \hfill
    \begin{subfigure}{0.2\linewidth}
        \centering
        \includegraphics[width=\textwidth]{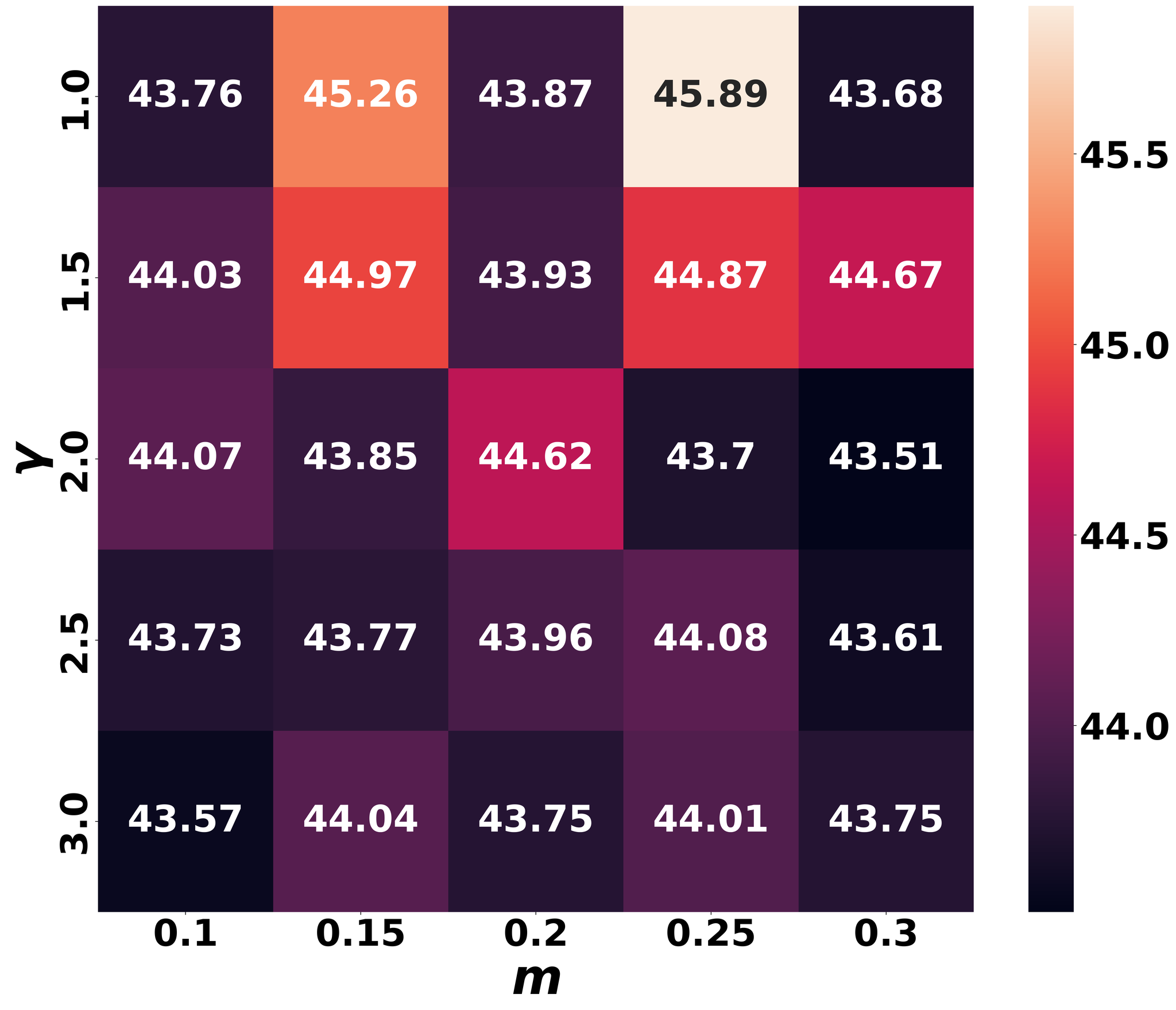}
        \caption{NMI CiteSeer}
    \end{subfigure}
    \hfill
    \begin{subfigure}{0.2\linewidth}
        \centering
       \includegraphics[width=\textwidth]{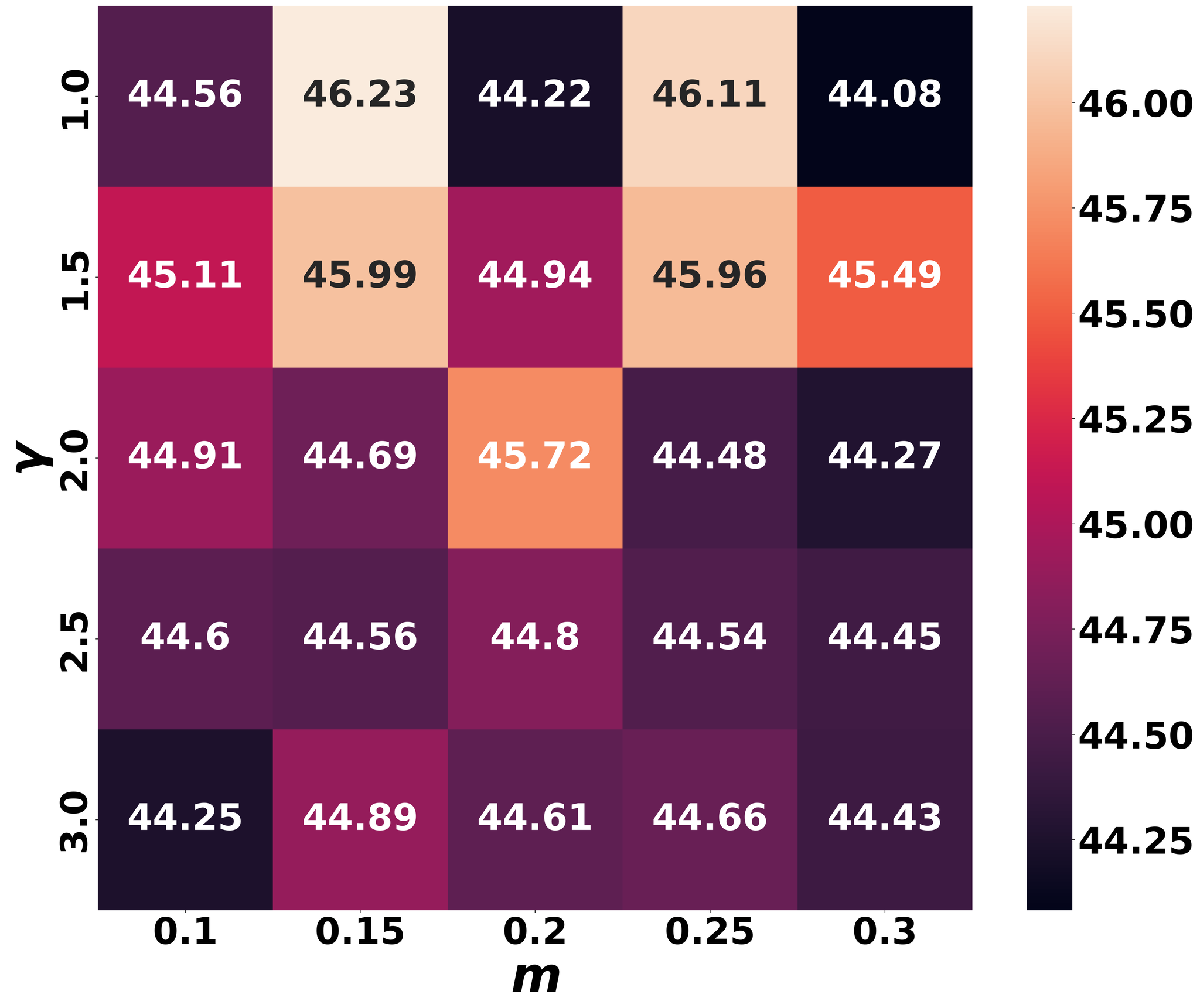}
       \caption{ARI CiteSeer}
    \end{subfigure}
     \hfill
    \begin{subfigure}{0.2\linewidth}
        \centering
      \includegraphics[width=\textwidth]{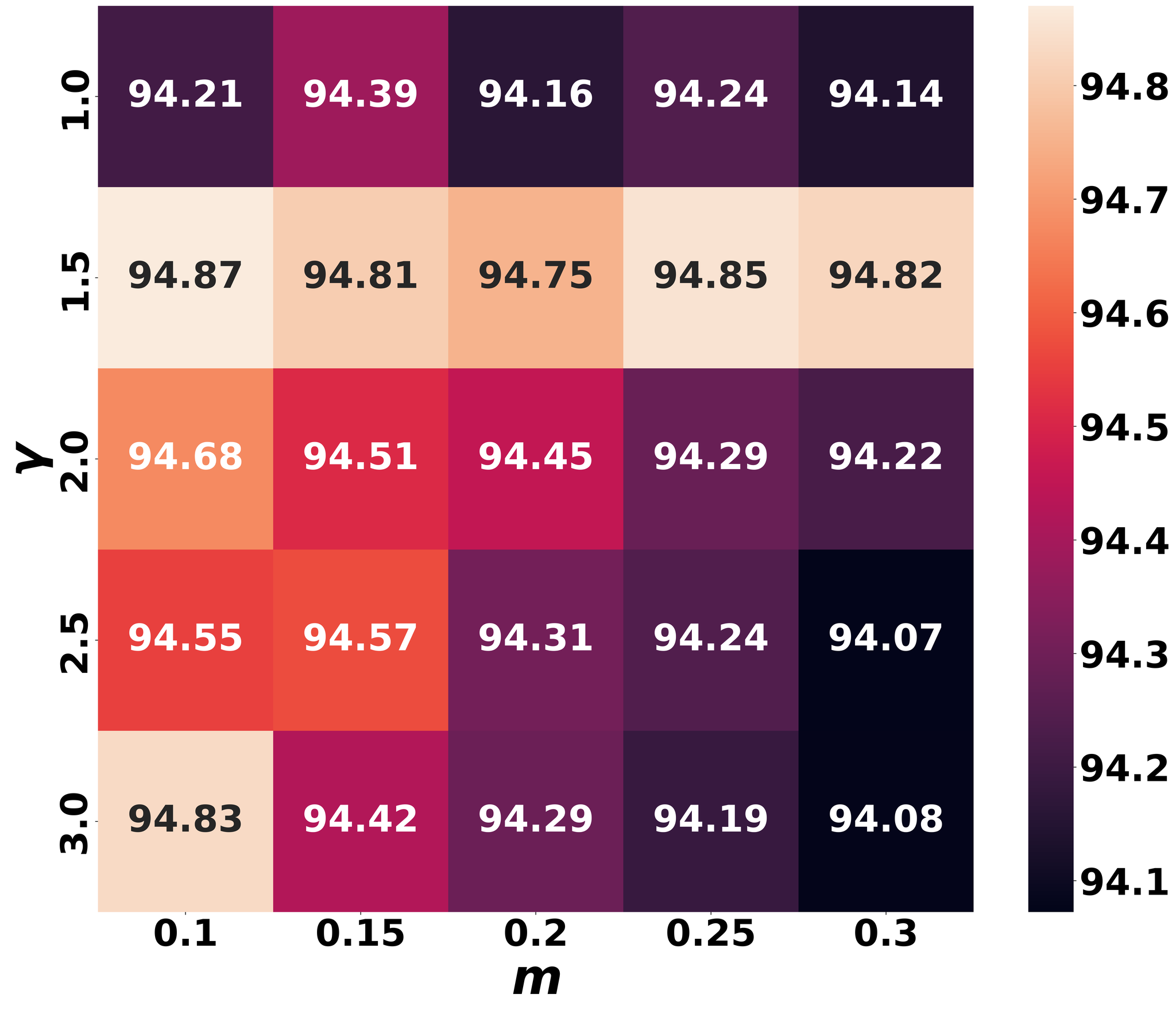}
        \caption{ROC-AUC CiteSeer}
    \end{subfigure}    

    \caption{Sensitivity of LSW-ML-GSSL to $m$ and $\gamma$. 
    }
    \label{fig:sensitivity}
\end{figure*}

\begin{figure*}[t]
    \centering
    \begin{subfigure}{0.2\linewidth}
        \centering
        \includegraphics[width=\textwidth]{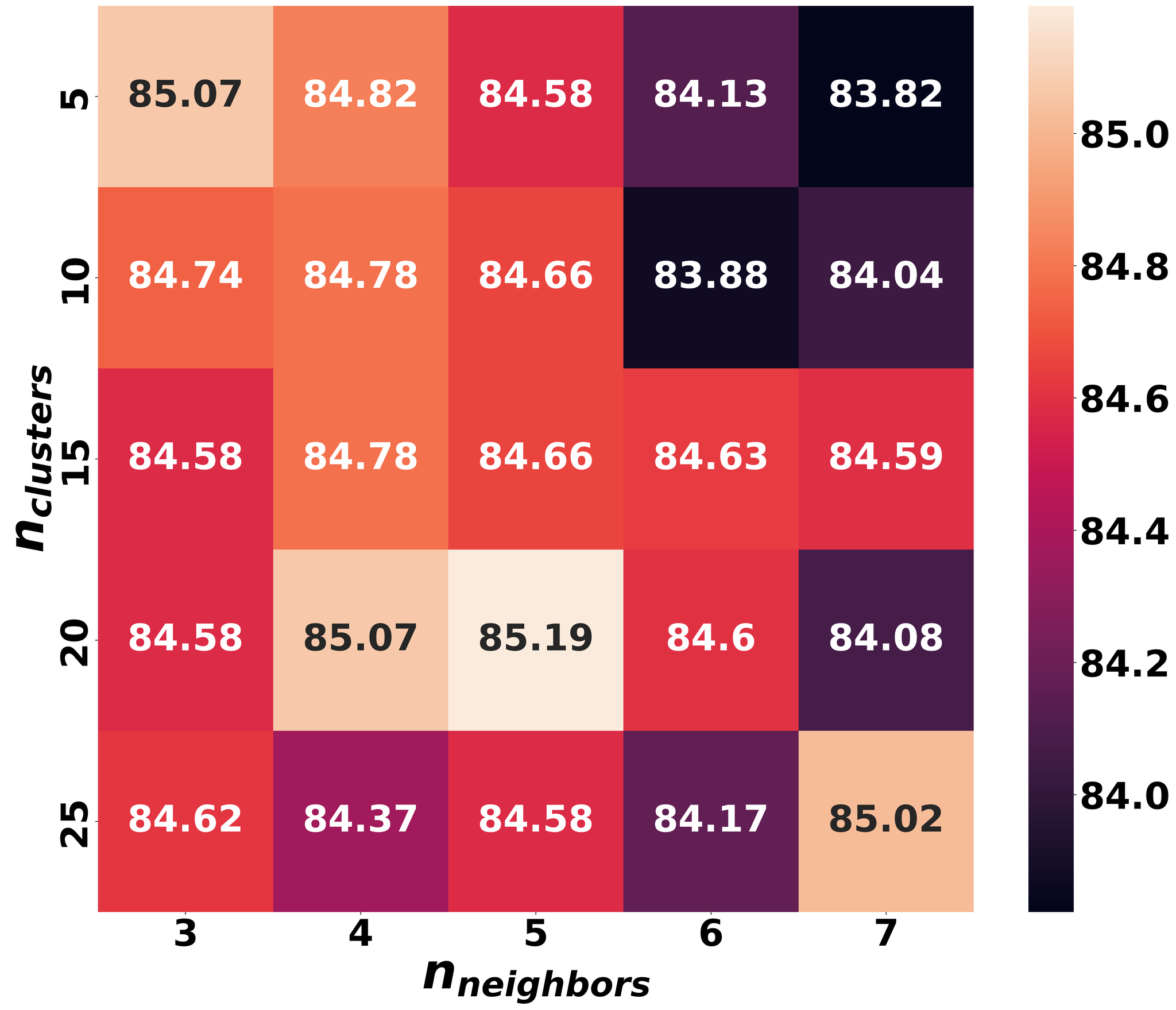}
        \caption{ACC Cora}
    \end{subfigure}
    \hfill
    \begin{subfigure}{0.2\linewidth}
        \centering
        \includegraphics[width=\textwidth]{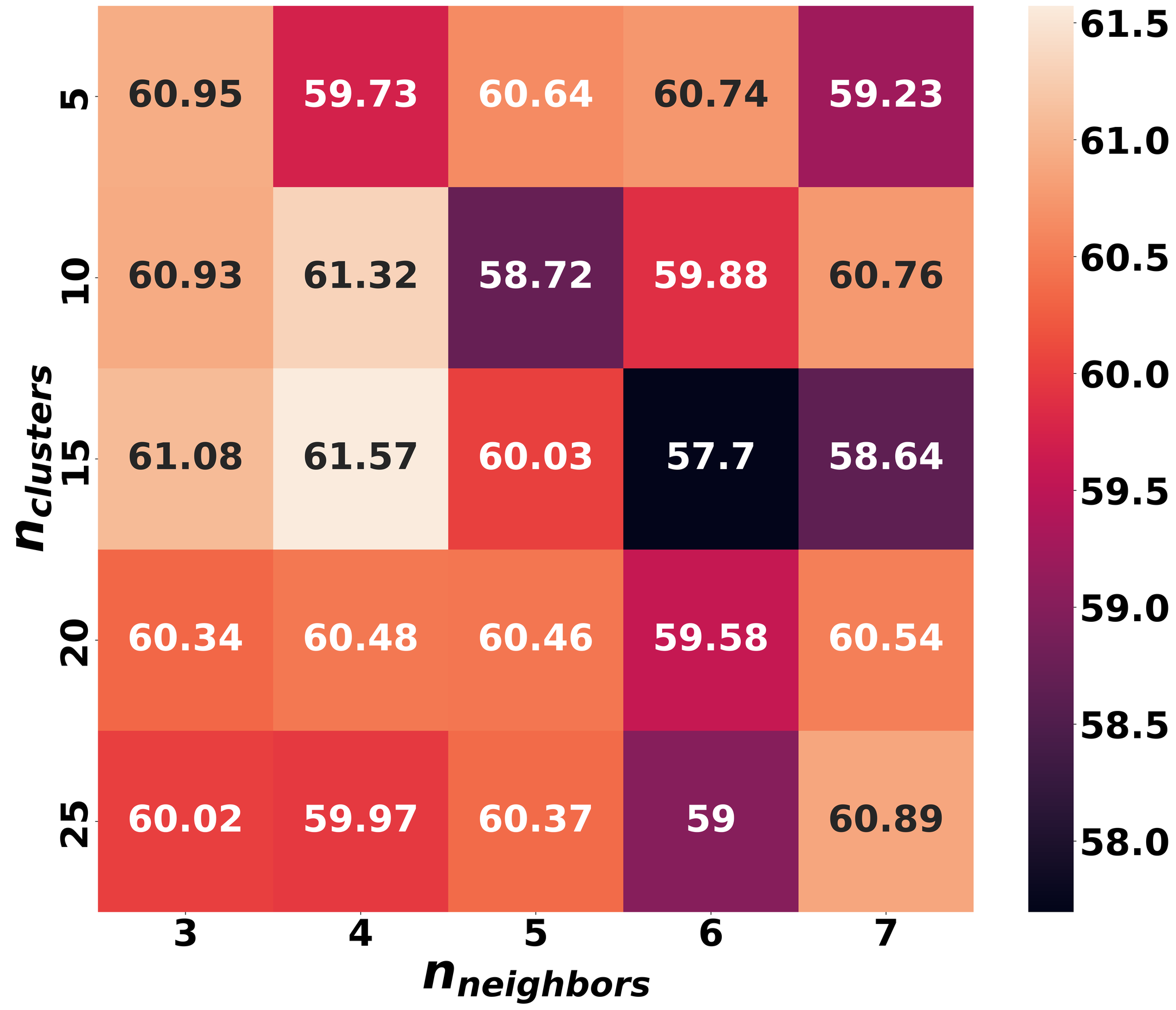}
        \caption{NMI Cora}
    \end{subfigure}
    \hfill
    \begin{subfigure}{0.2\linewidth}
        \centering
      \includegraphics[width=\textwidth]{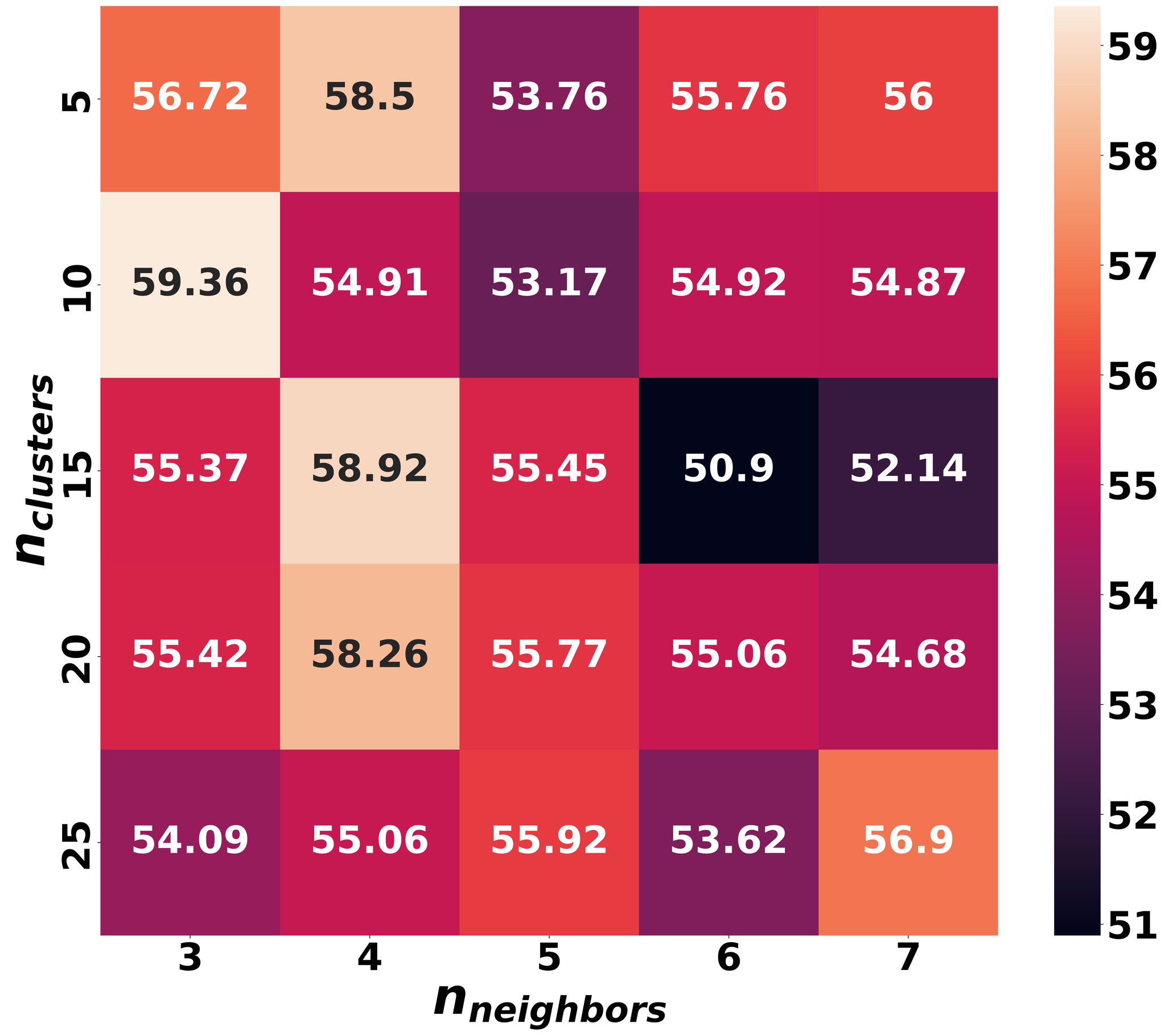}
        \caption{ARI Cora}
    \end{subfigure}
     \hfill
 \begin{subfigure}{0.2\linewidth}
        \centering
      \includegraphics[width=\textwidth]{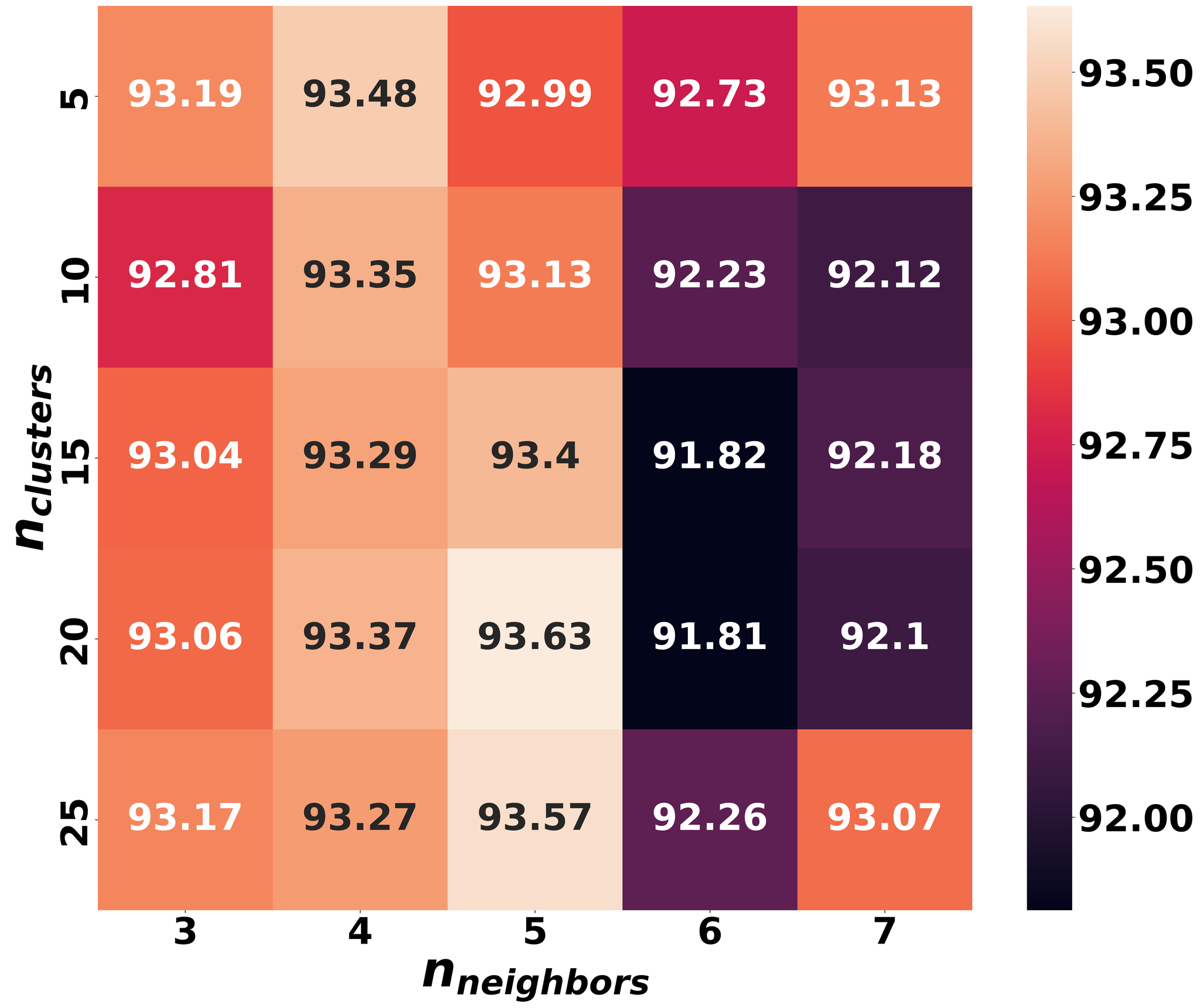}
        \caption{ROC-AUC Cora}
    \end{subfigure}   

    \vspace{0.5cm} 
    
    \begin{subfigure}{0.2\linewidth}
        \centering
        \includegraphics[width=\textwidth]{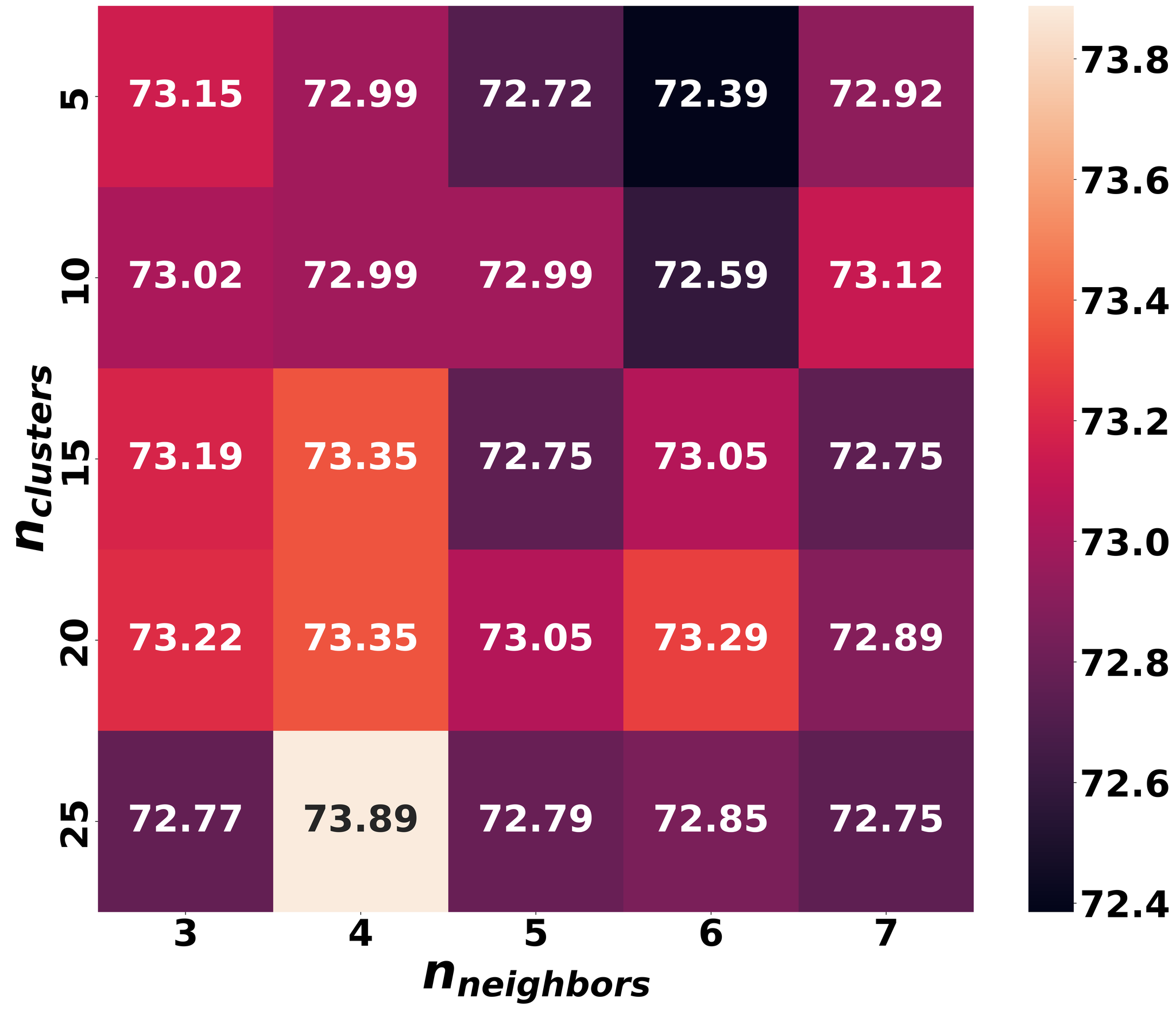}
        \caption{ACC CiteSeer}
    \end{subfigure}
    \hfill
    \begin{subfigure}{0.2\linewidth}
        \centering
        \includegraphics[width=\textwidth]{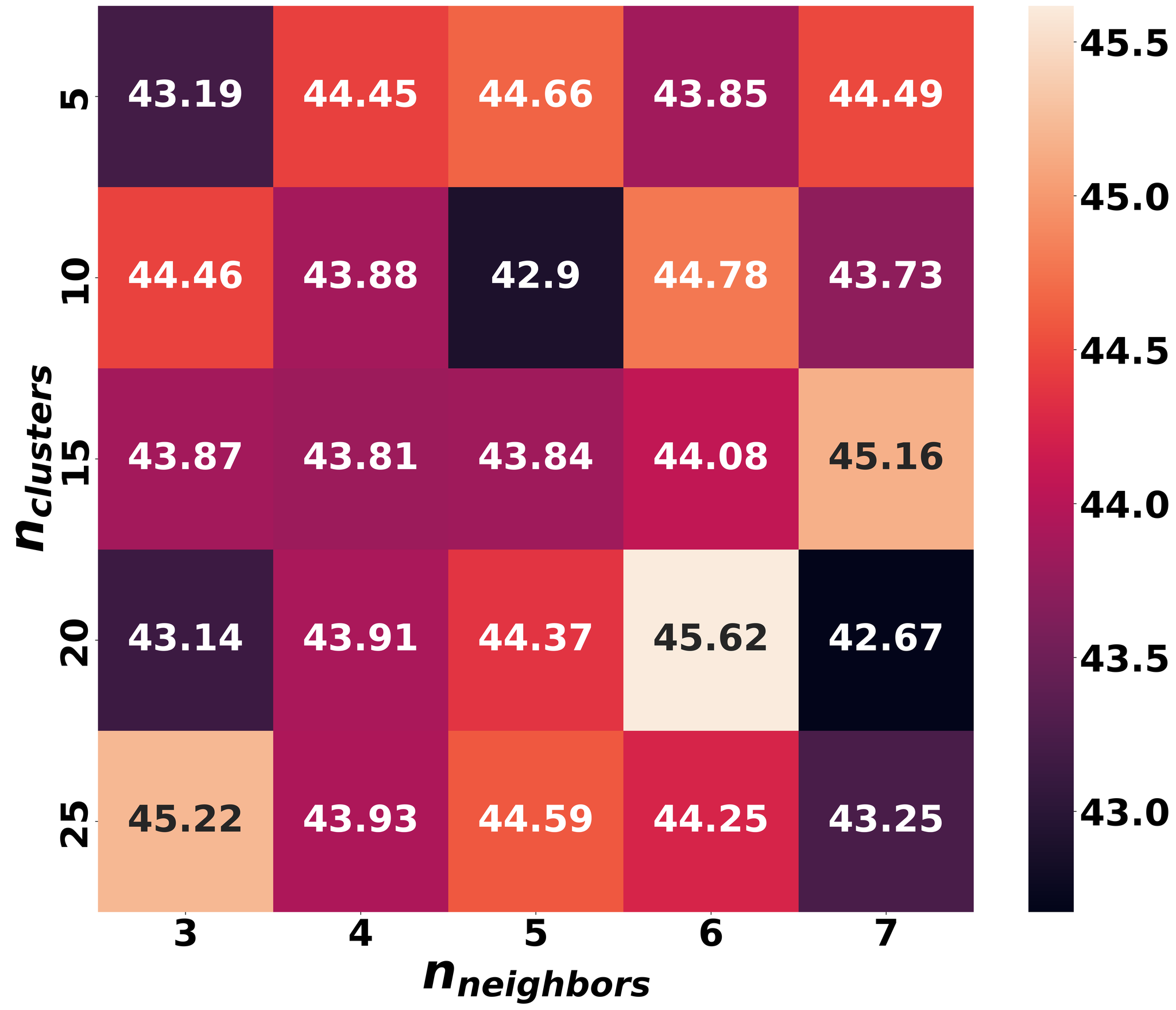}
        \caption{NMI CiteSeer}
    \end{subfigure}
    \hfill
    \begin{subfigure}{0.2\linewidth}
        \centering
       \includegraphics[width=\textwidth]{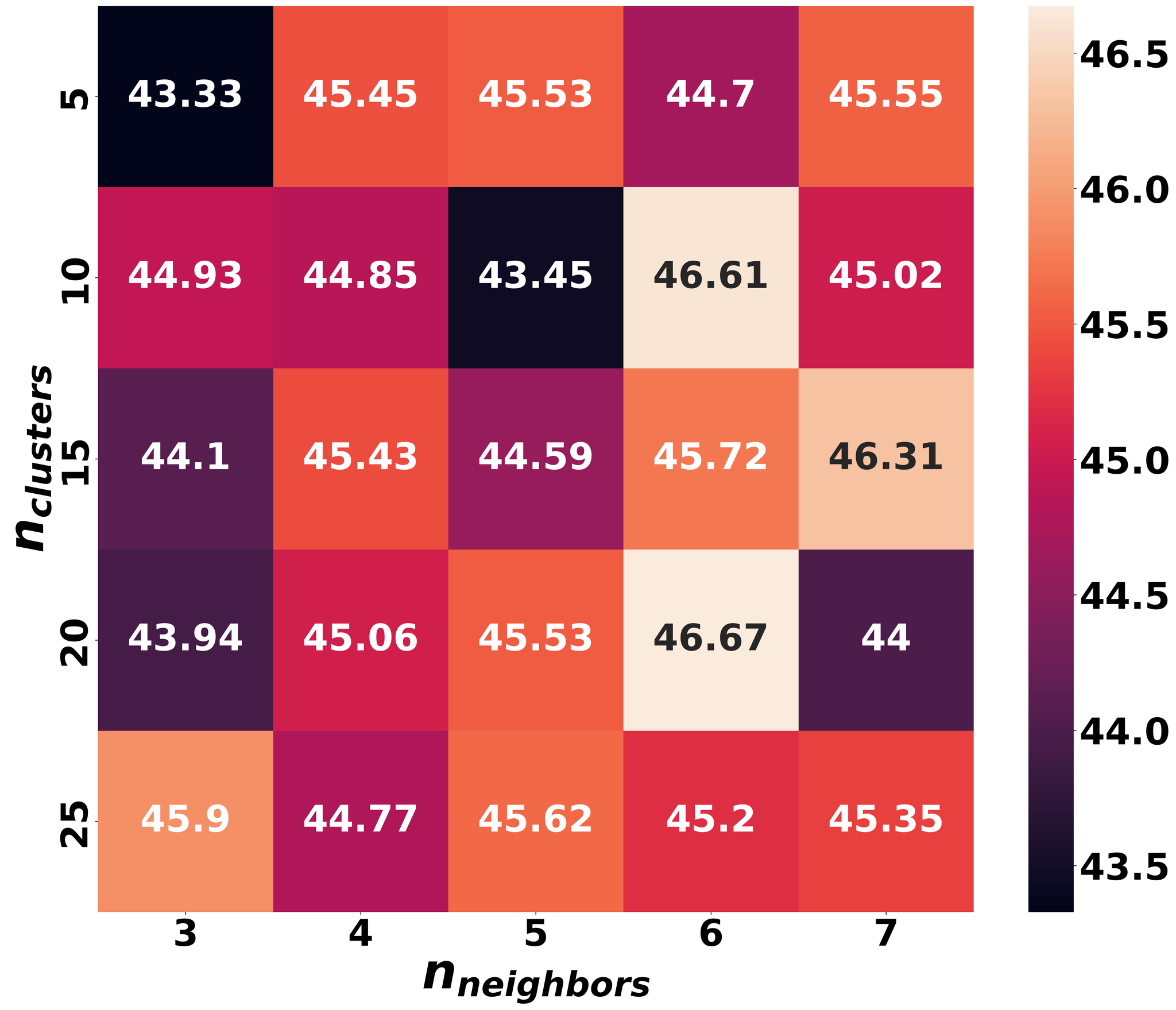}
       \caption{ARI CiteSeer}
    \end{subfigure}
     \hfill
    \begin{subfigure}{0.2\linewidth}
        \centering
      \includegraphics[width=\textwidth]{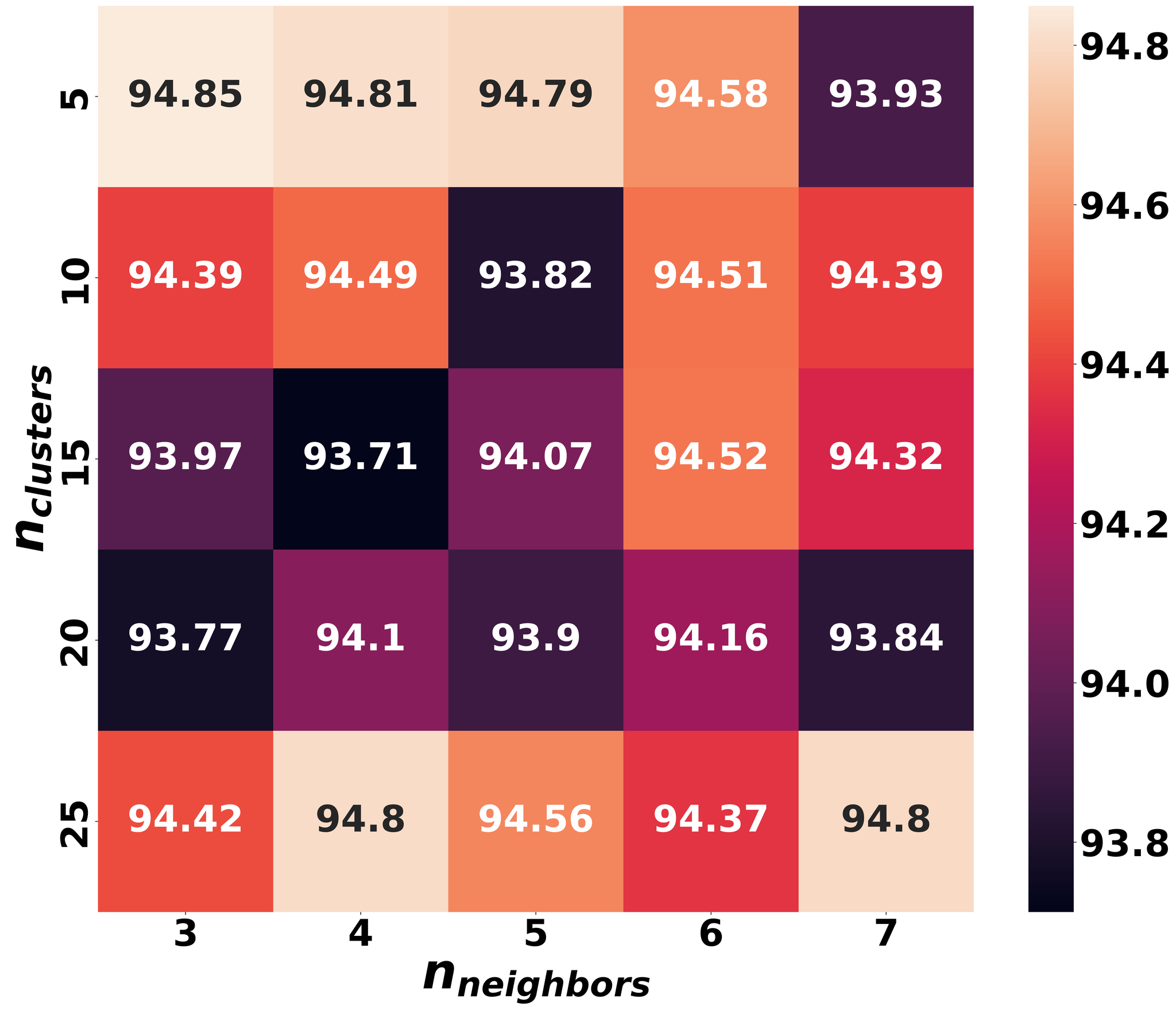}
        \caption{ROC-AUC CiteSeer}
    \end{subfigure}    

    \caption{Sensitivity of LSW-ML-GSSL to the neighborhood size and the number of clusters.}
    \label{fig:sensitivity_revision}
\end{figure*}

\textbf{Execution Time.} In this experiment, we evaluate the execution time of LSW-ML-GSSL compared to previous multi-task GSSL methods. The results in Fig. \ref{fig:execution_time} illustrate the efficiency of our approach compared with state-of-the-art methods. This efficiency stems from the simplicity of the self-weighting solution that eliminates the inner-optimization requirement. Furthermore, our multi-task GSSL method employs a single GNN encoder, whereas previous multi-task GSSL methods typically rely on multiple architectural components and/or computationally intensive operations. These include trainable gating networks (e.g., DyFSS), task-specific expert networks (e.g., GraphTCM, DyFSS, and WAS), and inner-optimization algorithms (e.g., AutoSSL and ParetoGNN). Notably, GraphTCM was excluded from the comparison due to its significantly higher computational cost. Specifically, its execution times on the Cora, CiteSeer, Pubmed, DBLP, Photo, and Computers datasets are 716, 1056, 39696, 59109, 29364, and 63116 seconds, respectively.

\textbf{Sensitivity to Hyperparameters.} We evaluate the sensitivity of LSW-ML-GSSL to the data-dependent hyperparameters \( m \) and \( \gamma \). Specifically, we vary \( m \) within the range \([0.10, 0.15, 0.20, 0.25, 0.30]\) and \( \gamma \) within \([1.0, 1.5, 2.0, 2.5, 3.0]\), while keeping all other hyperparameters constant across datasets. Fig. \ref{fig:sensitivity} illustrates the impact of these variations on model performance, assessed on three downstream tasks (i.e., node classification, node clustering, and link prediction) using ACC, NMI, ARI, and ROC-AUC for the Cora and CiteSeer datasets. The results indicate that LSW-ML-GSSL remains highly robust across different hyperparameter settings. As we can see, our approach exhibits small fluctuations in performance across the three downstream tasks. For both datasets, ROC-AUC scores are particularly stable across all values of \( m \) and \( \gamma \), confirming the method's reliability in link prediction tasks. ACC values show a slight increase for higher values of \( \gamma \), particularly for CiteSeer, where the best performance is observed around \( \gamma = 2.5 \). For NMI and ARI, which evaluate clustering quality, the performance of LSW-ML-GSSL exhibits minor variations, especially for smaller values of \( m \) and higher values of \( \gamma \). These results suggest the existence of a trade-off between the two hyperparameters. Interestingly, Cora appears to be more sensitive to changes in \( m \) and \( \gamma \) than CiteSeer, with noticeable performance peaks for certain configurations. Overall, our approach yields consistent performance for a broad range of values. These findings confirm that LSW-ML-GSSL is highly stable and robust to small hyperparameter changes.

In addition, we conduct a sensitivity analysis of LSW-ML-GSSL with respect to the neighborhood size and the number of clusters, and report the results in Fig.~\ref{fig:sensitivity_revision}. Overall, the model exhibits a broad performance plateau on both Cora and CiteSeer. ACC and ROC-AUC remain consistently high across a wide range of values, while NMI/ARI vary only mildly and degrade mainly when neighborhoods become large. These trends suggest that these parameters do not require careful dataset-specific tuning. Consequently, we fix them to moderate values to ensure stable performance while maintaining a consistent accuracy–efficiency trade-off. In contrast, $m$ and $\gamma$ directly affect the objective, making them inherently more data dependent and worth tuning.  

\section{Conclusion}
This work introduces the multi-level GSSL paradigm that captures information from four levels of granularity: node, proximity, cluster, and graph. In single-task learning, we propose a unified framework that can operate seamlessly at different abstraction levels. In multi-task learning, we extend the unified framework by integrating the positive and negative similarity scores across granularity levels. Our multi-task formulation aligns with the multi-level GSSL paradigm. First, the positive and negative scores are linearly combined. Second, we propose a self-weighting mechanism that adaptively prioritizes the scores that deviate significantly from their target values. The self-weighting mechanism enhances the optimization flexibility and defines a more precise convergence target. Extensive experiments show that the proposed single-task and multi-task approaches outperform state-of-the-art methods. In particular, our self-weighted method promotes task generalization and relinquishes the need for inner optimization and task-specific expert networks.

While our framework demonstrates strong performance on homogeneous graphs, many real-world networks exhibit heterogeneous structures with multiple node and edge types. A promising direction is extending the multi-level GSSL paradigm to heterogeneous graphs by replacing the homogeneous encoder with a type- or relation-aware GNN, together with per node category projection modules that embed different entities in a common space, so the unified similarity objective remains well defined. In this setting, the proximity level can be redefined using semantically meaningful meta-paths, and the cluster level can leverage type-aware community structures. Notably, our self-weighting mechanism is well-suited for this extension, as it can naturally balance contributions across both abstraction levels and relation types. A comprehensive heterogeneous treatment and evaluation are left as future work.

\appendices
\section{Convergence Analysis} 
\label{appendix_a}
We analyze the optimization dynamics directly in the similarity-score space $\{s_l^+,s_l^-\}_{l=1}^4$ under a toy setting with a single anchor and one positive/negative sample per level. This formulation enables a rigorous characterization of how the proposed self-weighting mechanism shapes the loss geometry and the resulting training dynamics.

\setcounter{theorem}{0}
\setcounter{proposition}{0}
% Ensure unique PDF anchors after resetting theorem-like counters.
\makeatletter
\renewcommand{\theHtheorem}{appendix.\arabic{theorem}}
\renewcommand{\theHproposition}{appendix.\arabic{proposition}}
\makeatother

\begin{proposition}[Quadratic form of the exponent argument]

By adopting a shifted cosine similarity ($(1+cos)/2$) and applying the reparameterization $\delta^+=1-m$, $\delta^-=m$, $o_l^+=1+m$, and $o_l^-=-m$ for all levels $l$, we can rewrite $\Delta''_l$ as follows:
\[
\Delta''_l
=
(s_l^-)^2 + (1-s_l^+)^2 - 2m^2.
\]
Consequently, the exponent argument of the toy version of Eq.~(\ref{eq.loss_lsw_ml_gssl}) becomes:
\[
\sum_{l=1}^{4} \Delta''_l
=
D - 8m^2,
\qquad
D = \sum_{l=1}^{4}\big[(s_l^-)^2 + (1-s_l^+)^2\big].
\]
\end{proposition}

%%%%%%%%%%%%%%%%%%%%%%%%%%%%%%%%%%%%%%%%%%%%Proof Prop. 1%%%%

\begin{IEEEproof}
From Eqs.~(\ref{positive_weight}--\ref{negative_weight}), the self-weighting coefficients are
$\alpha_l^+=[o_l^+-s_l^+]_+$ and $\alpha_l^-=[s_l^- - o_l^-]_+$.
Since we use a shifted cosine similarity ($(1+cos)/2$), the similarity scores satisfy $s_l^+,s_l^- \in [0,1]$ for all $l$. Since $s_l^- \in [0,1]$ and $o_l^-=-m$, we have $s_l^- - o_l^- = s_l^-+m >0$, thus $\alpha_l^- = s_l^-+m$.
Similarly, since $s_l^+ \in [0,1]$ and $o_l^+=1+m$, we have $\alpha_l^+ = 1+m-s_l^+$.
Plugging into Eq.~(\ref{eq.residual_similarity_score_second}) and using $\delta^-=m$ and $\delta^+=1-m$: 
\begin{align}
\Delta''_l
&= (s_l^-+m) \,(s_l^- - m) - (1+m-s_l^+) \,\big(s_l^+-(1-m)\big)\nonumber\\
&= (s_l^-)^2 - m^2 - (1+m-s_l^+) \, \big(s_l^+ - 1 + m\big).\label{eq:delta_expand}
\end{align}
Let $a = s_l^+ - 1$. Then $a \in [-1,0]$ and
$(1+m-s_l^+) = m-a$ and $(s_l^+ - 1 + m) = a+m$, hence

\begin{equation}
\begin{aligned}
(1+m-s_l^+) \, \big(s_l^+ - 1 + m\big) 
&= (m-a) \, (m+a)\\
&=m^2-a^2\\
&=m^2-(s_l^+-1)^2.
\end{aligned}
\end{equation}

\noindent Substituting back into \eqref{eq:delta_expand} gives
\begin{equation}
\begin{aligned}
\Delta''_l
&= (s_l^-)^2 - m^2 - \Big(m^2-(s_l^+-1)^2\Big) \\
&= (s_l^-)^2 + (s_l^+-1)^2 - 2m^2 \\
&= (s_l^-)^2 + (1-s_l^+)^2 - 2m^2 .
\end{aligned}
\end{equation}
Summing over $l=1,\dots,4$ yields\\
$\sum_{l=1}^{4}\Delta''_l = \sum_{l=1}^{4}\big[(s_l^-)^2+(1-s_l^+)^2\big] - 8m^2
= D-8m^2$.
\end{IEEEproof}

%%%%%%%%%%%%%%%%%%%%%%Theorem 1%%%%%%%%%%%%%%%%%%%%%%%%%%%%

\begin{theorem}[Error contraction in the toy similarity space]
\label{thm:radial_contraction}
In the same toy setting, we define the error vector: 
\[
\mathbf e =
[s_1^-,\dots,s_4^-,\,1-s_1^+,\dots,1-s_4^+]^\top.
\]
Let
\[
Z = D-8m^2,
\qquad
\mathcal{L} = \log\big(1+\exp(\gamma Z)\big).
\]
Then
\[
\nabla_{\mathbf e}\mathcal{L} = 2\,\gamma\,\sigma(\gamma \, Z)\,\mathbf e,
\quad\text{where}\quad
\sigma(t)=\frac{e^t}{1+e^t}\in(0,1).
\]
Under gradient descent with step size $0<\eta<\frac{1}{2\gamma}$, the error vector contracts multiplicatively:
\[
\mathbf e^{(t+1)} = \Big(1-2 \, \eta \, \gamma \, \sigma(\gamma \, Z^{(t)})\Big)\mathbf e^{(t)},
\]
and therefore:
\[
D^{(t+1)}=
\Big(1-2 \,\eta\,\gamma\,\sigma(\gamma \, Z^{(t)})\Big)^2\,D^{(t)}.
\]
Hence the dynamics perform radial descent toward the ideal point
$s_l^- \to 0$ and $s_l^+ \to 1$ for all $l$.
\end{theorem}
%%%%%%%%%%%%%%%%%%%%%%%%%%%%%%%%Proof Theorem 1%%%%%%%%%%%%%%%%%%%%
\begin{IEEEproof}
We first compute the partial derivatives by the chain rule.
Since $\mathcal{L}=\log(1+\exp(\gamma \, Z))$, we have:
\begin{equation}
\frac{\partial \mathcal{L}}{\partial Z}
=
\frac{\gamma \exp(\gamma \, Z)}{1+\exp(\gamma \, Z)}
=
\gamma\,\sigma(\gamma \, Z).
\end{equation}
By definition,
$D=\sum_{l=1}^4\big[(s_l^-)^2+(1-s_l^+)^2\big]$ and $Z=D-8m^2$.
Hence, for each $l$, we have:
\begin{equation}
\frac{\partial Z}{\partial s_l^-} = \frac{\partial D}{\partial s_l^-} = 2 \, s_l^-,
\end{equation}
\begin{equation}
\frac{\partial Z}{\partial s_l^+} = \frac{\partial D}{\partial s_l^+}
= 2 \, (s_l^+-1) = -2 \, (1-s_l^+).
\end{equation}
Combining these gives:
\begin{equation}
\begin{aligned}
\frac{\partial \mathcal{L}}{\partial s_l^-} &= \frac{\partial \mathcal{L}}{\partial Z}\frac{\partial Z}{\partial s_l^-} \\
&= \gamma \, \sigma(\gamma \, Z)\cdot 2 \,s_l^- \\
&= 2 \, \gamma \, \sigma(\gamma \, Z) \, s_l^-,
\end{aligned}
\end{equation}
and
\begin{equation}
\begin{aligned}
\frac{\partial \mathcal{L}}{\partial s_l^+} &= \frac{\partial \mathcal{L}}{\partial Z}\frac{\partial Z}{\partial s_l^+} \\
&= \gamma \, \sigma(\gamma \, Z)\cdot 2 \,(s_l^+-1) \\
&= -2 \, \gamma \, \sigma(\gamma \, Z)(1-s_l^+).
\end{aligned}
\end{equation}

\noindent We now derive the gradient descent updates.
For the negative similarities, we have:
\begin{align}
s_l^{-(t+1)}
&= s_l^{-(t)} - \eta \, \frac{\partial \mathcal{L}}{\partial s_l^-}\Big|_{t}\nonumber\\
&= s_l^{-(t)} - \eta \cdot 2\,\gamma\,\sigma(\gamma \, Z^{(t)})\, s_l^{-(t)}\nonumber\\
&= \Big(1-2\,\eta\,\gamma\,\sigma(\gamma \, Z^{(t)})\Big)s_l^{-(t)}.\label{eq:update_neg}
\end{align}
For the positive similarities:
\begin{align}
s_l^{+(t+1)}
&= s_l^{+(t)} - \eta \, \frac{\partial \mathcal{L}}{\partial s_l^+}\Big|_{t}\nonumber\\
&= s_l^{+(t)} - \eta \, \cdot \Big(-2 \,\gamma \, \sigma(\gamma \, Z^{(t)})(1-s_l^{+(t)})\Big)\nonumber\\
&= s_l^{+(t)} + 2 \, \eta \, \gamma \, \sigma(\gamma \, Z^{(t)})\,(1-s_l^{+(t)}).\label{eq:update_pos}
\end{align}
Define $u_l^{(t)} = 1-s_l^{+(t)}$. From \eqref{eq:update_pos}, we get:
\begin{align}
u_l^{(t+1)}
&= 1-s_l^{+(t+1)}\nonumber\\
&= 1-\Big(s_l^{+(t)} + 2\eta \,\gamma \,\sigma(\gamma \, Z^{(t)})(1-s_l^{+(t)})\Big)\nonumber\\
&= (1-s_l^{+(t)}) - 2\eta \,\gamma \, \sigma(\gamma \,Z^{(t)})(1-s_l^{+(t)})\nonumber\\
&= \Big(1-2 \, \eta \, \gamma \, \sigma(\gamma \, Z^{(t)})\Big)u_l^{(t)}.\label{eq:update_u}
\end{align}
Equations \eqref{eq:update_neg} and \eqref{eq:update_u} show that all eight components
of $\mathbf{e}^{(t)}=[s_1^{-(t)},\dots,s_4^{-(t)},u_1^{(t)},\dots,u_4^{(t)}]^\top$
are multiplied by the same scalar factor
$c_t = 1-2 \, \eta \, \gamma \, \sigma(\gamma \, Z^{(t)})$, i.e.,
$\mathbf{e}^{(t+1)}=c_t \, \mathbf{e}^{(t)}$.

Finally, since $0<\sigma(\cdot)<1$ and $0<\eta<\frac{1}{2\gamma}$,
we have $0<c_t<1$ for all $t$, so $\|\mathbf{e}^{(t+1)}\|_2^2=c_t^2 \|\mathbf{e}^{(t)}\|_2^2$.
Because $D^{(t)}=\|\mathbf{e}^{(t)}\|_2^2$, this yields
$D^{(t+1)}=c_t^2 D^{(t)}$ and proves monotone contraction toward
$\mathbf{e}^\star=\mathbf{0}$, i.e., $s_l^-=0$ and $s_l^+=1$ for all $l$.
\end{IEEEproof}

%%%%%%%%%%%%%%%%%%%%%%%%%%%%%%%%%%%% Prop 2 %%%%%%%%%%%%%%%%%%%%
After establishing that the multiplicative contraction property in Theorem~\ref{thm:radial_contraction} is specific to the self-weighting mechanism, we now show that this behavior does not, in general, hold for the linear multi-level combination in Eq.~(\ref{eq.loss_l_ml_gssl}). To illustrate this point, we consider the same toy setting with fixed coefficients $\beta_l$: 
\begin{equation}
\mathcal{L}_{\mathrm{lin}}
=
\log\Big(1+\exp\big(\gamma S\big)\Big),
\label{eq:L_lin_toy_1}
\end{equation}
\begin{equation}
S = \sum_{l=1}^{4}\beta_l\big(s_l^- - s_l^+ + m\big).
\label{eq:L_lin_toy_2}
\end{equation}
\begin{proposition}[No radial descent / no multiplicative contraction for the linear combination]
The gradient of $\mathcal{L}_{\mathrm{lin}}$ in $\mathbf{e}$-coordinates can be expressed as follows:
\[
\nabla_{\mathbf{e}} \mathcal{L}_{\mathrm{lin}}
=
\gamma\,\sigma(\gamma S)\,
[\beta_1,\dots,\beta_4,\beta_1,\dots,\beta_4]^\top,
\]
which is a fixed direction independent of $\mathbf{e}$ (up to the scalar factor
$\sigma(\gamma S)$). Therefore, in general $\nabla_{\mathbf{e}} \mathcal{L}_{\mathrm{lin}}$
is not parallel to $\mathbf{e}$ and the update is not radial. In particular, there does not exist a scalar sequence $\{c_t\}_t$, such that $\mathbf{e}^{(t+1)}=c_t \, \mathbf{e}^{(t)}$ holds for all initial $\mathbf{e}^{(0)}$ under
gradient descent on $\mathcal{L}_{\mathrm{lin}}$.
\end{proposition}

%%%%%%%%%%%%%%%%%%%%%%%%%%%%%%%%%Proof Prop. 2%%%%%%%%%
\begin{IEEEproof}
The derivative of $\mathcal{L}_{\mathrm{lin}}=\log(1+\exp(\gamma \, S))$ w.r.t. $S$ is
$\partial \mathcal{L}_{\mathrm{lin}}/\partial S = \gamma \, \sigma(\gamma \, S)$.
Moreover,
$\partial S/\partial s_l^- = \beta_l$ and $\partial S/\partial s_l^+ = -\beta_l$.
Thus
$\partial \mathcal{L}_{\mathrm{lin}}/\partial s_l^- = \gamma \, \sigma(\gamma \, S) \, \beta_l$
and
$\partial \mathcal{L}_{\mathrm{lin}}/\partial s_l^+ = -\gamma \, \sigma(\gamma \, S) \, \beta_l$.

From the definition of $\mathbf{e}$, its components satisfy $e_l=s_l^-$ and $e_{4+l}=1-s_l^+$ for $l=1,\dots,4$.
By the chain rule,
$\partial \mathcal{L}_{\mathrm{lin}}/\partial e_{4+l}
=
-(\partial \mathcal{L}_{\mathrm{lin}}/\partial s_l^+)
=
\gamma \, \sigma(\gamma \, S) \, \beta_l$.
Hence
$\nabla_{\mathbf{e}} \mathcal{L}_{\mathrm{lin}}
=
\gamma \, \sigma(\gamma \, S) \, [\beta_1,\dots,\beta_4,\beta_1,\dots,\beta_4]^\top$,
which is a fixed direction independent of $\mathbf{e}$ (up to the scalar factor).

To show that multiplicative contraction cannot hold in general, assume by contradiction
that there exists a scalar $c_t$ such that
$\mathbf{e}^{(t+1)}=c_t \, \mathbf{e}^{(t)}$ for all $\mathbf{e}^{(t)}$.
Then the gradient descent step
$\mathbf{e}^{(t+1)}=\mathbf{e}^{(t)}-\eta \, \nabla_{\mathbf{e}} \mathcal{L}
_{\mathrm{lin}}(\mathbf{e}^{(t)})$
implies that $\nabla_{\mathbf{e}} \mathcal{L}_{\mathrm{lin}}(\mathbf{e}^{(t)})$ is always parallel to
$\mathbf{e}^{(t)}$. However, since $\nabla_{\mathbf{e}} \mathcal{L}_{\mathrm{lin}}$ has a fixed direction,
this can only hold when $\mathbf{e}^{(t)}$ is collinear with
$[\beta_1,\dots,\beta_4,\beta_1,\dots,\beta_4]^\top$, which is not true for general
initial conditions.

For an explicit counterexample, take $\mathbf{e}^{(0)}$ with nonzero entries only in the
first level, e.g.,
$\mathbf{e}^{(0)}=[1,0,0,0,\,0,0,0,0]^\top$.
Then $\nabla_{\mathbf{e}} \mathcal{L}_{\mathrm{lin}}(\mathbf{e}^{(0)})$ has nonzero entries in the first and fifth coordinates whenever $\beta_1\neq 0$, so $\mathbf{e}^{(1)}=\mathbf{e}^{(0)}-\eta \, \nabla_{\mathbf{e}} \mathcal{L}_{\mathrm{lin}}(\mathbf{e}^{(0)})$
generally has nonzero entries in multiple coordinates and cannot equal
$c_0 \mathbf{e}^{(0)}$ for any scalar $c_0$.
This proves that the multiplicative contraction property of
Theorem~\ref{thm:radial_contraction} does not hold for the linear combination.
\end{IEEEproof}

\balance
\bibliographystyle{IEEEtran}

\bibliography{references}

\vspace*{-1cm}
%===================== Mahmoud =========================%
\begin{IEEEbiography}[{\includegraphics[width=1in, height=1.25in,clip,keepaspectratio]{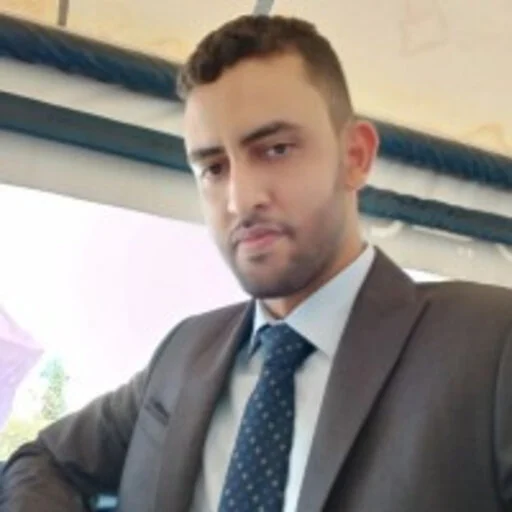}}]
{Mohamed Mahmoud Amar} is a Ph.D. student at the University of Quebec at Montreal (UQAM). His research interests include graph representation learning and multi-task learning.
\end{IEEEbiography}

\vspace*{-1cm}
%===================== Nairouz =========================%
\begin{IEEEbiography}[{\includegraphics[width=1in, height=1.25in,clip,keepaspectratio]{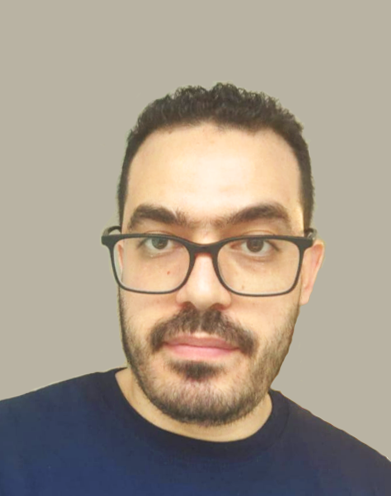}}]
{Nairouz Mrabah} received his Ph.D. degree from the University of Quebec at Montreal (UQAM). His research interests include clustering, graph representation learning, and vision-language models.
\end{IEEEbiography}

\vspace*{-1cm}
%===================== Mohamed =========================%
\begin{IEEEbiography}[{\includegraphics[width=1in,height=1.20in,clip,keepaspectratio]{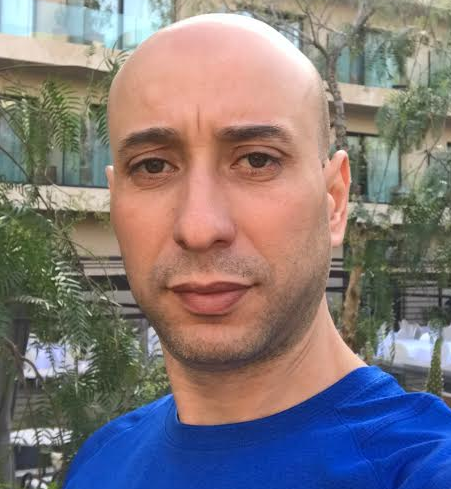}}]{Mohamed Bouguessa}
\scriptsize received the M.Sc. and the Ph.D. degrees, respectively, in 2005 and 2009 from the University of Sherbrooke, Quebec, Canada. He is currently a professor of computer science at the University of Quebec at Montreal (UQAM). His research focuses on graph mining and graph representation learning.
\end{IEEEbiography}
\vspace*{-1cm}

%===================== Abdoulaye =========================%
\begin{IEEEbiography}[{\includegraphics[width=1in,height=1.20in,clip,keepaspectratio]{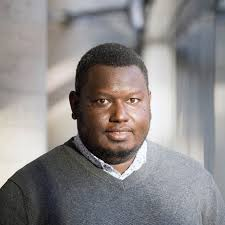}}]{Abdoulaye Baniré Diallo}
\scriptsize received the PhD degree from McGill University, Montreal, QC, Canada, in 2009. He is currently a Professor of Computer Science at the Université du Québec à Montréal (UQAM). His research focuses on the design of algorithmic methods for the analysis of heterogeneous biological data.
\end{IEEEbiography}

\end{document}